\documentclass{article}

\usepackage[preprint, nonatbib]{neurips_2021}

\usepackage[utf8]{inputenc} % allow utf-8 input
\usepackage[T1]{fontenc}    % use 8-bit T1 fonts
\usepackage{hyperref}       % hyperlinks
\usepackage{url}            % simple URL typesetting
\usepackage{booktabs}       % professional-quality tables
\usepackage{amsfonts}       % blackboard math symbols
\usepackage{nicefrac}       % compact symbols for 1/2, etc.
\usepackage{microtype}      % microtypography
\usepackage{xcolor}         % colors

\usepackage{graphicx, bbm, latexsym, amsfonts, epsfig, verbatim, amsmath, color, epsfig, verbatim, enumerate, multirow, caption, graphicx, subcaption, algorithm, algpseudocode}
\usepackage{url, multicol, color, latexsym, amsfonts, epsfig, verbatim, tikz, cases, bbm}
\usepackage{algcompatible}
\usepackage{diagbox,multicol,graphicx}
\usepackage{booktabs, caption, makecell}
\usepackage{threeparttable}
\usepackage{array}

\usepackage{bbm}
\usepackage{amsmath, amsthm, amssymb}
\usepackage{lipsum}
\usepackage{graphicx}
\usepackage{color}
\usepackage{multirow}
\usepackage{amsfonts}
\usepackage{sectsty}
\usepackage{subcaption, wrapfig}
\usepackage{footnote}

\usepackage[multidot]{grffile}

\DeclareMathOperator*{\argmin}{arg\,min}

\newtheorem{theorem}{Theorem}
\newtheorem{lemma}{Lemma}

\newtheorem{definition}{Definition}
\newtheorem{assumption}{Assumption}

\newtheorem{proposition}{Proposition}
\newtheorem{remark}{Remark}

\title{Differentially Private Federated Learning via Inexact ADMM}

\author{%
  Minseok Ryu\\
  Mathematics and Computer Science Division\\
  Argonne National Laboratory\\
  Lemont, IL 60439 \\
  \texttt{mryu@anl.gov} \\
  \And
  Kibaek Kim \\
  Mathematics and Computer Science Division \\
  Argonne National Laboratory\\
  Lemont, IL 60439 \\
  \texttt{kimk@anl.gov} \\
}

\begin{document}

\maketitle

\begin{abstract}
Differential privacy (DP) techniques can be applied to the federated learning model to protect data privacy against inference attacks to communication among the learning agents. The DP techniques, however, hinder achieving a greater learning performance while ensuring strong data privacy. In this paper we develop a DP inexact alternating direction method of multipliers algorithm that solves a sequence of subproblems with the objective perturbation by random noises generated from a Laplace distribution. We show that our algorithm provides $\bar{\epsilon}$-DP for every iteration, where $\bar{\epsilon}$ is a privacy parameter controlled by a user. 
% and $\mathcal{O}(1/T)$ rate of convergence in expectation, where $T$ is the number of iterations. 
Using MNIST and FEMNIST datasets for the image classification, we demonstrate that our algorithm reduces the testing error by at most $22\%$ compared with the existing DP algorithm, while achieving the same level of data privacy. The numerical experiment also shows that our algorithm converges faster than the existing algorithm.
\end{abstract}

\section{Introduction} \label{sec:introduction}
In this work we propose a privacy-preserving algorithm for solving a federated learning (FL) model \cite{konevcny2015federated}, namely, a machine learning (ML) model that aims to learn global model parameters \textit{without} collecting locally stored data into a central server.
The proposed algorithm is based on an inexact alternating direction method of multipliers (IADMM) that solves a sequence of subproblems whose \textit{objective functions} are perturbed by injecting some random noises for ensuring \textit{differential privacy} (DP) on the distributed data.
We show that the proposed algorithm provides more accurate solutions compared with the state-of-the-art DP algorithm \cite{huang2019dp} while both algorithms provide the same level of data privacy.
As a result, the proposed algorithm can mitigate a trade-off between data privacy and solution accuracy (i.e., learning performance in the context of ML), which is one of the main challenges in developing DP algorithms as described in \cite{dwork2014algorithmic}.

Developing highly accurate privacy-preserving algorithms can enhance the practical uses of FL in applications with sensitive data (e.g., electronic health records \cite{shickel2017deep} and mobile device data \cite{mcmahan2017communication}) because a greater learning performance can be achieved while preserving privacy on the sensitive data exposed to be leaked during a training process.
Because of the importance of FL, incorporating privacy-preserving techniques into optimization algorithms for solving the FL models has been studied extensively.

\textbf{Related Work.}
The empirical risk minimization (ERM) model used for learning parameters in supervised ML is often vulnerable to adversarial attacks \cite{madry2017towards}, a situation that motivates the application of privacy-preserving techniques (e.g., DP \cite{dwork2006calibrating} and homomorphic encryption \cite{kaissis2020secure}) to protect data.
Among these techniques, DP has been widely used in the ML community and is especially useful for protecting data against inference attacks \cite{shokri2017membership}.

Formally, DP is a privacy-preserving technique that randomizes the output of an algorithm such that any single data point cannot be inferred by an adversary that can reverse-engineer the randomized output.
Depending on where to inject noises to randomize the output, DP can be categorized by input \cite{fukuchi2017differentially, kang2020input}, output \cite{dwork2006calibrating, chaudhuri2011differentially}, and objective \cite{chaudhuri2011differentially, kifer2012private} perturbation methods.
Compared with input perturbation, which directly perturbs input data by adding random noises,  \textit{output perturbation} and \textit{objective perturbation} methods provide a randomized output of an optimization problem by injecting random noises into its true output and objective function, respectively.
In \cite{chaudhuri2011differentially}, the authors propose a differentially private ERM that utilizes the output and objective perturbation methods to ensure DP on data.
Also, Abadi et al.~\cite{abadi2016deep} apply the output perturbation to stochastic gradient descent (SGD) in order to ensure DP on data for every iteration of the algorithm.
The privacy-preserving technique in our work is the \textit{objective perturbation} method:  we randomize the output of the trust-region subproblem by perturbing its objective function with some random noises.
For details of differentially private ML, we refer readers to \cite{sarwate2013signal,kifer2012private,iyengar2019towards}.

Within the context of FL, various distributed optimization algorithms have been developed for solving the distributed ERM model.
For example, \texttt{FedAvg} in \cite{mcmahan2017communication} is an algorithm that combines SGD for each agent with a central server that performs model averaging.
Another example is \texttt{FedProx} in \cite{li2018federated} that is constructed by replacing the local SGD in \texttt{FedAvg} with an optimization problem with an additional \textit{proximal} function.
These algorithms do not guarantee data privacy during a training process, however, preventing their practical uses.
Readers interested in details of FL should see \cite{kairouz2019advances,li2019survey,li2020federated}; for details  about FL without the central server, see \cite{li2017robust,elgabli2020gadmm}.

In order to preserve privacy on data used for the FL model, various DP algorithms have been proposed in the literature, where the output and objective perturbations are incorporated for ensuring DP (see \cite{agarwal2018cpsgd, wei2020federated, naseri2020toward,zhang2016dynamic, huang2019dp}).
For example, the intermediate model parameters and/or gradients computed for every iteration of the \texttt{FedAvg}-type and \texttt{FedProx}-type algorithms are perturbed for guaranteeing DP as in \cite{naseri2020toward} and \cite{wei2020federated}, respectively, which can be seen as the output perturbation.
Also, in \cite{zhang2016dynamic}, the primal and dual variables computed for every iteration of the ADMM algorithm are perturbed, which can be seen as the output and objective perturbations, respectively. 
Zhang and Zhu \cite{zhang2016dynamic} compare the two perturbation methods, as \cite{chaudhuri2011differentially} did under the general ML setting, and show that the objective perturbation can provide more accurate solutions compared with the output perturbation.
The use of the objective perturbation is somewhat limited, however, because it requires the objective function to be twice differentiable and strongly convex whereas the twice differentiability restriction can be relaxed to the differentiability for the output perturbation.
In \cite{huang2019dp}, the authors incorporate the output perturbation into IADMM that utilizes the first-order approximation with a proximal function.
Introducing the first-order approximation in ADMM enforces smoothness of the objective function, hence satisfying the aforementioned differentiability assumption for ensuring DP.
Also, the authors show that the algorithm has $\mathcal{O}(1/\sqrt{T})$ rate of convergence in expectation, where $T$ is the number of iterations.
Moreover, their numerical experiments demonstrate that the algorithm outperforms DP-ADMM in \cite{zhang2016dynamic} and DP-SGD in \cite{abadi2016deep}.

\textbf{Contributions.}
In this paper, as compared with the DP-IADMM algorithm in \cite{huang2019dp}, we incorporate the \textit{objective perturbation} into IADMM that utilizes the first-order approximation.
Our main contributions are summarized as follows:
\begin{itemize}
  \itemsep0em
  \item Proof that the our new IADMM algorithm provides DP on data
  % \item Proof that the rate of convergence in expectation for our DP algorithm is improved to $\mathcal{O}(1/T)$
  \item Numerical demonstration that our DP algorithm provides more accurate solutions compared with the existing DP algorithm \cite{huang2019dp}
\end{itemize}

\textbf{Organization and Notation.}
The remainder of the paper is organized as follows.
In Section \ref{sec:model} we describe an FL model using a distributed ERM and present the existing inexact ADMM algorithm for solving the FL model.
In Section \ref{sec:DP-IADMM-Trust} we propose a new DP inexact ADMM algorithm for solving the FL model that ensures DP on data and converges to an optimal solution with the sublinear convergence rate.
In Section \ref{sec:experiments} we describe numerical experiments to demonstrate the outperformance of the proposed algorithm.

We denote by $\mathbb{N}$ a set of natural numbers. For $A \in \mathbb{N}$, we define $[A]:= \{1,\ldots,A\}$ and denote by $\mathbb{I}_{A}$ a $A \times A$ identity matrix.
We use $\langle \cdot, \cdot \rangle$ and $\|\cdot \|$ to denote the scalar product and the Euclidean norm, respectively.

% The existing subgradient-based algorithms (e.g., \cite{nedic2014stochastic}, \cite{cohen2017projected}) can find an optimal solution of \eqref{ERM_0} with $\mathcal{O}(1/\sqrt{T})$ rate of convergence under a general convex function setting, where $T$ is the number of iterations.
% Moreover, the existing ADMM algorithms (e.g., \cite{he20121}, \cite{he2015non}) can solve \eqref{ERM_0} with $\mathcal{O}(1/T)$ rate of convergence.

\section{Federated Learning Model} \label{sec:model}
%% 1. Model Description
\textbf{Distributed ERM.}
Consider a set $[P]$ of agents connected to a central server.
Each agent $p \in [P]$ has a training dataset $\mathcal{D}_p := \{x_{pi}, y_{pi} \}_{i=1}^{I_p}$, where $I_p$ is the number of data samples, $x_{pi} \in \mathbb{R}^J$ is a $J$-dimensional data feature, and $y_{pi} \in \mathbb{R}^K$ is a $K$-dimensional data label.
We consider a \textit{distributed} ERM problem given by
\begin{align}
\min_{w  \in \mathcal{W}} \ & \textstyle\sum_{p=1}^P \Big\{ \frac{1}{I} \sum_{i=1}^{I_p} \ell(w; x_{pi},y_{pi}) + \frac{\beta}{P} r(w) \Big\}, \label{ERM_0}
\end{align}
where
$w \in \mathbb{R}^{J \times K}$ is a global model parameter vector,
$\mathcal{W}$ is a compact convex set,
$\ell(\cdot)$ is a convex loss function, $r(\cdot)$ is a convex regularizer function, $\beta > 0$ is a regularizer parameter, and $I := \sum_{p=1}^P I_p$.
Since \eqref{ERM_0} is a convex optimization problem, it can be expressed by an \textit{equivalent} Lagrangian dual problem. More specifically, we first rewrite \eqref{ERM_0} as
\begin{subequations}
\label{ERM_1}
\begin{align}
  \min_{w, z_1, \ldots, z_P \in \mathcal{W}} \ & \textstyle\sum_{p=1}^P f_p(z_p;\mathcal{D}_p)  \\
  \mbox{s.t.} \ & w_{jk} = z_{pjk}, \ \forall p \in [P], \forall j \in [J], \forall k \in [K], \label{ERM_1-1}
\end{align}
where $z_p \in \mathbb{R}^{J \times K}$ is a local parameter vector defined for every agent $p \in [P]$ and
\begin{align}
& f_p(z_p;\mathcal{D}_p) :=  \textstyle\frac{1}{I} \sum_{i=1}^{I_p} \ell(z_p; x_{pi},y_{pi}) + \frac{\beta}{P} r(z_p). \label{def_fn_f}
\end{align}
\end{subequations}
By introducing dual variables $\lambda_p \in \mathbb{R}^{J \times K}$ associated with constraints \eqref{ERM_1-1}, the Lagrangian dual problem is given by
\begin{align}
  \max_{  \lambda_1, \ldots, \lambda_P} \ \min_{w, z_1, \ldots, z_P  \in \mathcal{W}} \ \  \textstyle\sum_{p=1}^P f_p(z_p; \mathcal{D}_p) + \langle \lambda_p, w-z_p \rangle. \label{ERM_LDual}
\end{align}
Since \eqref{ERM_1} is a convex optimization problem, solving \eqref{ERM_LDual} provides an optimal solution to \eqref{ERM_1}. 

\textbf{Inexact ADMM.}
ADMM is an iterative optimization algorithm that can find an optimal solution of \eqref{ERM_LDual} in the augmented Lagrangian form.
More specifically, for every iteration $t \in [T]$, it updates $(w^{t}, z^{t}, \lambda^{t}) \rightarrow (w^{t+1}, z^{t+1}, \lambda^{t+1})$ by solving a sequence of the following subproblems:
\begin{subequations}
\label{ADMM}
\begin{align}
& w^{t+1} \leftarrow \argmin_{w } \ \textstyle\sum_{p=1}^P \langle \lambda^t_p, w \rangle + \frac{\rho^t}{2} \|w-z^t_p\|^2, \label{ADMM-1} \\
& z^{t+1}_p \leftarrow \argmin_{z_p \in \mathcal{W}} \ \textstyle f_p(z_p;\mathcal{D}_p) - \langle \lambda^t_p, z_p \rangle + \frac{\rho^t}{2} \|w^{t+1}-z_p\|^2, \ \forall p \in [P], \label{ADMM-2} \\
& \lambda^{t+1}_p \leftarrow \textstyle \lambda^{t}_p + \rho^t (w^{t+1}-z^{t+1}_p), \ \forall p \in [P], \label{ADMM-3}
\end{align}
\end{subequations}
where $\rho^t > 0$ is a penalty parameter that controls the proximity of the global and local parameters.

One need not solve the subproblem \eqref{ADMM-2} exactly in each iteration to guarantee the overall convergence. In~\cite{huang2019dp}, \eqref{ADMM-2} is replaced with the following inexact subproblem:
\begin{subequations}
\label{ADMM-2-Prox}
\begin{align}
& z^{t+1}_p \leftarrow \textstyle \argmin_{z_p  \in \mathcal{W} } \ H^t(z_p; \mathcal{D}_p) + \frac{1}{2\eta^t}\|z_p - z^t_p\|^2,   \\
& H^t(z_p; \mathcal{D}_p) := \langle f'_p(z^t_p;\mathcal{D}_p), z_p \rangle  + \textstyle \frac{\rho^t}{2} \|w^{t+1}-z_p + \frac{1}{\rho^t} \lambda^t_p  \|^2. \label{function_G_1}
\end{align}
\end{subequations}
This subproblem is obtained by (i) replacing the convex function $f_p(z_p;\mathcal{D}_p)$ in \eqref{ADMM-2} with its lower bound $\widehat{f}_p(z_p;\mathcal{D}_p) := f_p(z^t_p;\mathcal{D}_p) + \langle f'_p(z^t_p;\mathcal{D}_p), \ z_p - z_p^t \rangle$, where $f'_p(z^t_p;\mathcal{D}_p)$ is a subgradient of $f_p$ at $z^t_p$, and (ii) adding a \textit{proximal} term $\frac{1}{2\eta^t}\|z_p - z^t_p\|^2$ with a proximity parameter $\eta^t > 0$ that controls the proximity of a new solution $z^{t+1}_p$ from $z^t_p$ computed from the previous iteration.

Alternatively, a \textit{trust-region} constraint can be introduced to form the following inexact subproblem:
\begin{subequations}
\label{ADMM-2-Trust}
\begin{align}
& z^{t+1}_p \leftarrow \textstyle \argmin_{z_p  \in  \{\mathcal{W} \cap \widehat{\mathcal{W}}^t_p\}} \ H^t(z_p; \mathcal{D}_p) ,   \\
& \widehat{\mathcal{W}}^t_p := \{ z_p \in \mathbb{R}^{J \times K} : \|z_p - z^t_p\| \leq \delta^t \} , \ \forall p \in [P], \label{trust-region}
\end{align}
\end{subequations}
where \eqref{trust-region} defines a trust region with a proximity parameter $\delta^t > 0$.
Note that both \textit{proximal} and \textit{trust-region} techniques are used for finding a new solution within a certain distance from the solution computed in the previous iteration and have been widely used for numerous optimization algorithms (e.g., the bundle method \cite{teo2010bundle}).
We will discuss how to set $(\rho^t, \eta^t, \delta^t)$ in Sections \ref{sec:convergence} and \ref{sec:experiments}.

In this paper we refer to $\{\text{\eqref{ADMM-1}}\rightarrow\eqref{ADMM-2-Prox}\rightarrow\text{\eqref{ADMM-3}}\}_{t=1}^T$ and $\{\text{\eqref{ADMM-1}}\rightarrow \eqref{ADMM-2-Trust}\rightarrow\text{\eqref{ADMM-3}}\}_{t=1}^T$ as IADMM-Prox and IADMM-Trust, respectively.
Note that each agent $p$ solves the inexact subproblem (\eqref{ADMM-2-Prox} or  \eqref{ADMM-2-Trust}) while the central server computes \eqref{ADMM-1} and \eqref{ADMM-3}.
We consider such a training process, where the data $\mathcal{D}_p$ defining the inexact subproblem can be inferred by an adversary who can access the information $(w^{t+1},\lambda^t_p, z^{t+1}_p)$ exchanged.
To protect $\mathcal{D}_p$, we introduce \textit{differential privacy}  into the algorithmic processes, which will be discussed in the next section.

\section{Differentially Private Inexact ADMM} \label{sec:DP-IADMM-Trust}
In this section we propose two DP-IADMM algorithms that iteratively solve the \textit{constrained} subproblem (\eqref{ADMM-2-Prox} or \eqref{ADMM-2-Trust}) whose \textit{objective function} is perturbed by some random noises for ensuring DP.
The privacy and convergence analyses of the proposed algorithms are presented in Sections \ref{sec:privacy} and \ref{sec:convergence}.
% In Appendix \ref{apx-DP-IADMM-Prox}, we also present a DP-IADMM-Prox algorithm that solves the subproblem \eqref{ADMM-2-Proximal} whose \textit{objective function} is perturbed by some random noises for ensuring DP, as well as its privacy and convergence analyses.

DP is a data privacy preservation technique that aims to protect data by randomizing \textit{outputs} of an algorithm that takes data as inputs.
A formal definition follows.
\begin{definition}{\textbf{(Definition 3 in \cite{chaudhuri2011differentially}})} \label{def:differential_privacy_1}
  A randomized algorithm $\mathcal{A}$ provides $\bar{\epsilon}$-DP if for any two datasets $\mathcal{D}$ and $\mathcal{D}'$ that differ in a single entry and for any set $\mathcal{S}$,
  \begin{align}
  e^{-\bar{\epsilon}} \ \mathbb{P}( \mathcal{A}(\mathcal{D}') \in \mathcal{S} )  \leq \mathbb{P}( \mathcal{A}(\mathcal{D}) \in \mathcal{S} ) \leq e^{\bar{\epsilon}} \ \mathbb{P}( \mathcal{A}(\mathcal{D}') \in \mathcal{S} ),  \label{def_differential_privacy_1}
  \end{align}
  where $\mathcal{A}(\mathcal{D})$ (resp. $\mathcal{A}(\mathcal{D}')$) is the randomized output of $\mathcal{A}$ on input $\mathcal{D}$ (resp. $\mathcal{D}'$).
  \end{definition}
  According to the inequalities \eqref{def_differential_privacy_1}, $\mathbb{P}( \mathcal{A}(\mathcal{D}) \in \mathcal{S}) - \mathbb{P}( \mathcal{A}(\mathcal{D}') \in \mathcal{S}) \rightarrow 0$ as $\bar{\epsilon} \rightarrow 0$.
  This implies that as $\bar{\epsilon}$ decreases, it becomes harder to distinguish the two datasets $\mathcal{D}$ and $\mathcal{D}'$ by analyzing the randomized outputs, thus providing stronger data privacy.

\textbf{Objective Perturbation.}
We construct a randomized algorithm $\mathcal{A}$ satisfying \eqref{def_differential_privacy_1} by introducing some calibrated random noises into the objective function of the subproblem (\eqref{ADMM-2-Prox} or \eqref{ADMM-2-Trust})  to protect data in an $\bar{\epsilon}$-DP manner.
The subproblems \eqref{ADMM-2-Prox} and \eqref{ADMM-2-Trust} with the random noises are given by
\begin{align}
  & z^{t+1}_p (\mathcal{D}_p) = \textstyle \argmin_{z_p \in \mathcal{W}} \ G^t(z_p; \mathcal{D}_p, \tilde{\xi}^t_p) + \frac{1}{2\eta^t} \| z_p - z_p^{t}\|^2, \ \text{ and }  \label{DPADMM-2-Prox} \\
  & z^{t+1}_p (\mathcal{D}_p) = \textstyle \argmin_{z_p \in \mathcal{W} \cap \widehat{\mathcal{W}}^t_p}   \ G^t(z_p; \mathcal{D}_p, \tilde{\xi}^t_p), \label{DPADMM-2-Trust}
\end{align}  
respectively, where
\begin{align}
  G^t(z_p; \mathcal{D}_p, \tilde{\xi}^t_p) := \langle f'_p(z^t_p;\mathcal{D}_p), z_p \rangle  + \textstyle \frac{\rho^t}{2} \|w^{t+1}-z_p + \frac{1}{\rho^t}(\lambda^t_p - \tilde{\xi}^{t}_p) \|^2, \label{function_G_2}
\end{align}
and $\tilde{\xi}^{t}_p \in \mathbb{R}^{J \times K}$ is a noise vector sampled from a Laplace distribution with zero mean, whose probability density function (pdf) is given by
\begin{subequations}  
\label{Laplace}  
\begin{align}
& L ( \tilde{\xi}^{t}_p ; \bar{\epsilon}, \bar{\Delta}_p^t ) := \textstyle \frac{\bar{\epsilon}}{2 \bar{\Delta}_p^t} \exp \big(- \frac{\bar{\epsilon} \| \tilde{\xi}^{t}_p \|_1 }{\bar{\Delta}_p^t} \big), \label{Laplace-pdf} \\
& \bar{\Delta}_p^t := \textstyle \max_{\mathcal{D}'_p \in \widehat{\mathcal{D}}_p} \| f'_p(z^t_p;\mathcal{D}_p) - f'_p(z^t_p;\mathcal{D}'_p)\|_1, \label{Delta} \\
& \widehat{\mathcal{D}}_p := \text{a collection of datasets differing a single entry from a given } \mathcal{D}_p. \label{DataCollect}
\end{align}
\end{subequations}
Note that the function $G^t$ in \eqref{function_G_2} is constructed by adding a linear function $\langle \tilde{\xi}^t_{p}, z_{p} \rangle$ to the function $H^t$ in \eqref{function_G_1} (see Appendix \ref{apx-derivation} for the derivation). 

Some remarks follow.
\begin{remark} \label{remark:ADMM-2-General}
  Observe that
  (i) the function $G^t$ in \eqref{function_G_2} is strongly convex with a constant $\rho^{\text{min}} > 0$, where $\rho^{\text{min}} \leq \rho^t$ for all $t$, and  
  (ii) $\tilde{\xi}^t_{pjk} = 0$ makes \eqref{DPADMM-2-Prox} and \eqref{DPADMM-2-Trust} equal to \eqref{ADMM-2-Prox} and \eqref{ADMM-2-Trust}, respectively.
\end{remark}

We present DP-IADMM-Prox and DP-IADMM-Trust algorithms in Algorithm \ref{algo:DP-IADMM-Prox} and Algorithm \ref{algo:DP-IADMM-Trust}, respectively.
In line 3, the central server solves \eqref{ADMM-1}, which has a closed-form solution.
In line 5, each agent $p$ solves \eqref{DPADMM-2-Prox} or \eqref{DPADMM-2-Trust} whose objective function is perturbed by the Laplacian noises described in \eqref{Laplace}.
In line 7, the central server collects the information $z^{t+1}_p$ from all agents to update dual variables $\lambda^{t+1}$ as described in \eqref{ADMM-3}.

\begin{multicols}{2}
  \begin{algorithm}[H]
    \caption{DP-IADMM-Prox.}
    \label{algo:DP-IADMM-Prox}
    \begin{algorithmic}[1]
    \STATE Initialize $\lambda^1, z^1 \in \mathbb{R}^{P \times J \times K}$.
    \FOR{$t \in [T]$}
    \STATE Compute $w^{t+1}$ by solving \eqref{ADMM-1}.
    \FOR{$p \in [P]$}{ \textbf{in parallel}}
    \STATE Find $z^{t+1}_p$ by solving \eqref{DPADMM-2-Prox}.
    \ENDFOR
    \STATE Compute $\lambda^{t+1}$ as in \eqref{ADMM-3}.
    \ENDFOR
    \end{algorithmic}
  \end{algorithm}
  \begin{algorithm}[H]
    \caption{DP-IADMM-Trust.}
    \label{algo:DP-IADMM-Trust}
    \begin{algorithmic}[1]
    \STATE Initialize $\lambda^1, z^1 \in \mathbb{R}^{P \times J \times K}$.
    \FOR{$t \in [T]$}
    \STATE Compute $w^{t+1}$ by solving \eqref{ADMM-1}.
    \FOR{$p \in [P]$}{ \textbf{in parallel}}
    \STATE Find $z^{t+1}_p$ by solving \eqref{DPADMM-2-Trust}.
    \ENDFOR
    \STATE Compute $\lambda^{t+1}$ as in \eqref{ADMM-3}.
    \ENDFOR
    \end{algorithmic}
  \end{algorithm}
\end{multicols}

\subsection{Privacy Analysis} \label{sec:privacy}
In this section we focus on showing that $\bar{\epsilon}$-DP in Definition \ref{def:differential_privacy_1} is guaranteed for every iteration of Algorithm \ref{algo:DP-IADMM-Prox} while the privacy analysis for Algorithm \ref{algo:DP-IADMM-Trust} is in Appendix \ref{apx:privacy-analysis-trust}.
To this end, using the following lemma, we will show that the \textit{constrained} problem \eqref{DPADMM-2-Prox} provides $\bar{\epsilon}$-DP.

\begin{lemma}{\textbf{(Theorem 1 in \cite{kifer2012private})}} \label{lemma:theorem1}
Let $\mathcal{A}$ be a randomized algorithm induced by the random variable $\tilde{\xi}$ that provides $\phi(\mathcal{D},\tilde{\xi})$.
Consider a sequence of randomized algorithms $\{\mathcal{A}_{\ell}\}$, each of which provides $\phi^{\ell}(\mathcal{D},\tilde{\xi})$.
If $\mathcal{A}_{\ell}$ is $\bar{\epsilon}$-DP for all $\ell$ and satisfies a pointwise convergence condition, namely, $\lim_{{\ell} \rightarrow \infty} \phi^{\ell}(\mathcal{D},\tilde{\xi}) = \phi(\mathcal{D},\tilde{\xi})$, then $\mathcal{A}$ is also $\bar{\epsilon}$-DP.
\end{lemma}

For the rest of this section, we fix $t \in \mathbb{N}$ and $p \in [P]$. For ease of exposition, we express the feasible region of \eqref{DPADMM-2-Prox} using $M$ inequalities, namely,
\begin{align*}
     \mathcal{W} \Leftrightarrow \{ z_p \in \mathbb{R}^{J \times K} : h_m (z_p) \leq 0, \ \forall m \in [M] \}, 
\end{align*}
where $h_m$ is convex and twice continuously differentiable. 
The subproblem \eqref{DPADMM-2-Prox} can be expressed by
\begin{align}
\min_{z_p} \ G^t(z_p; \mathcal{D}_p, \tilde{\xi}^t_p) + \textstyle \frac{1}{2\eta^t} \| z_p - z_p^{t}\|^2 + \mathcal{I}_{\mathcal{W}}(z_p),
\end{align}
where $ \mathcal{I}_{\mathcal{W}}(z_p)$ is an indicator function that takes zero if $z_p \in \mathcal{W}$ and $\infty$ otherwise.
We notice that the indicator function can be approximated by the following function:
\begin{align}
& g(z_p; \ell) :=  \textstyle \sum_{m=1}^M \ln ( 1+ e^{\ell h_m(z_p)}), \label{logfn-prox}  
\end{align}
where $\ell > 0$. Note that the function $g$ is similar to the Logarithmic barrier function (LBF), namely $-(1/\ell) \sum_{m=1}^M \ln (-h_m(z_p))$, in that the approximation becomes closer to the indicator function as $\ell \rightarrow \infty$. The main difference of $g$ from LBF is that the output of $g$ exists even when $h_m(z_p) > 0$. 
By replacing the indicator function with the function $g$ in \eqref{logfn-prox}, we construct the following \textit{unconstrained} problem whose objective function is strongly convex:
\begin{align}
  & z^{t+1}_p(\ell, \mathcal{D}_p) = \textstyle \argmin_{z_p \in \mathbb{R}^{J \times K} } \ G^t(z_p; \mathcal{D}_p, \tilde{\xi}^t_p) + \frac{1}{2\eta^t} \| z_p - z_p^{t}\|^2 + g(z_p; \ell). \label{ADMM-2-Prox-log}
\end{align}

We first show that \eqref{ADMM-2-Prox-log} satisfies the pointwise convergence condition and provides $\bar{\epsilon}$-DP as in Propositions \ref{prop:pointwise_convergence} and \ref{prop:dpinapproximation}, respectively.
% , we show that \eqref{DPADMM-2-Trust}, namely $\mathcal{A}$ in Lemma \ref{lemma:theorem1}, provides $\bar{\epsilon}$-DP (see Theorem \ref{thm:privacy}).
\begin{proposition}\label{prop:pointwise_convergence}
  For fixed $t$ and $p$, we have $\lim_{\ell \rightarrow \infty} z_p^{t+1}(\ell,\mathcal{D}_p) = z_p^{t+1}(\mathcal{D}_p)$, where $z_p^{t+1}(\mathcal{D}_p)$ and $z_p^{t+1}(\ell,\mathcal{D}_p)$ are from \eqref{DPADMM-2-Prox} and \eqref{ADMM-2-Prox-log}, respectively.
\end{proposition}
\begin{proof}
See Appendix \ref{apx-prop:pointwise_convergence}
\end{proof}

\begin{proposition} \label{prop:dpinapproximation}
  For fixed $t$, $p$, and $\ell$, \eqref{ADMM-2-Prox-log} provides $\bar{\epsilon}$-DP, namely, satisfying
  \begin{align}
    \label{DP_1}
        e^{-\bar{\epsilon}} \ \mathbb{P} \big( z^{t+1}_p(\ell;  \mathcal{D}'_p) \in \mathcal{S} \big) \leq \mathbb{P} \big( z^{t+1}_p (\ell; \mathcal{D}_p )\in \mathcal{S}  \big)  \leq e^{\bar{\epsilon}} \ \mathbb{P} \big( z^{t+1}_p(\ell;  \mathcal{D}'_p) \in \mathcal{S} \big)
  \end{align}
  for all $\mathcal{S} \subset \mathbb{R}^{J \times K}$ and all $\mathcal{D}'_p \in \widehat{\mathcal{D}}_p$, where $\widehat{\mathcal{D}}_p$ is from \eqref{DataCollect}.
\end{proposition}
\begin{proof}
See Appendix \ref{apx-prop:dpinapproximation}
\end{proof}

Based on Propositions \ref{prop:pointwise_convergence} and \ref{prop:dpinapproximation}, Lemma \ref{lemma:theorem1} can be used for proving the following theorem.
\begin{theorem} \label{thm:privacy}
  For fixed $t$ and $p$, \eqref{DPADMM-2-Prox} provides $\bar{\epsilon}$-DP, namely, satisfying
  \begin{align*}
    e^{-\bar{\epsilon}} \ \mathbb{P}( z^{t+1}_p(\mathcal{D}'_p) \in \mathcal{S} ) \leq \mathbb{P}( z^{t+1}_p(\mathcal{D}_p) \in \mathcal{S}) \leq e^{\bar{\epsilon}} \ \mathbb{P}( z^{t+1}_p( \mathcal{D}'_p) \in \mathcal{S} ), 
  \end{align*}
  for all $\mathcal{S} \subset \mathbb{R}^{J \times K}$ and all $\mathcal{D}'_p \in \widehat{\mathcal{D}}_p$, where $\widehat{\mathcal{D}}_p$ is from \eqref{DataCollect}.
\end{theorem}
% \begin{proof}
% \end{proof}
\begin{remark}
Theorem \ref{thm:privacy} and Theorem \ref{thm:privacy_trust} in Appendix \ref{apx:privacy-analysis-trust} show that $\bar{\epsilon}$-DP is guaranteed for every iteration of Algorithm \ref{algo:DP-IADMM-Prox} and Algorithm \ref{algo:DP-IADMM-Trust}, respectively. This result can be extended by introducing the existing composition theorem in \cite{dwork2014algorithmic} to ensure $\bar{\epsilon}$-DP for the entire process of the algorithm.
\end{remark}

\subsection{Convergence Analysis} \label{sec:convergence}
In this section we show that a sequence of solutions generated by Algorithm \ref{algo:DP-IADMM-Prox} converges to an optimal solution in \textit{expectation} with $\mathcal{O}(1/\sqrt{T})$ rate while the convergence rate of Algorithm \ref{algo:DP-IADMM-Trust} remains as a future reasearch.

Throughout this section, we make the following assumptions.
\begin{assumption}\label{assump:convergence}
  In \eqref{DPADMM-2-Prox}, (i) $\eta^t = 1/\sqrt{t}$,
  (ii) $\rho^t > 0$ is nondecreasing and bounded above (i.e., $\rho^t \leq \rho^{\text{max}}, \forall t$).
  (iii) The convex function $f_p$ from \eqref{def_fn_f} is $L$-Lipschitz over a set $\mathcal{W}$ with respect to the Euclidean norm.
  % (iv) There exists $\gamma$ such that $\gamma \geq  \| \lambda^*\|$, where $\lambda^*$ is a dual optimal solution.
\end{assumption}

Under Assumption \ref{assump:convergence} (iii), the following parameters can be defined (see Appendix \ref{apx:existence_of_UBs} for details):
\begin{subequations}
\label{def_upper}
\begin{align} 
& \textstyle U_1 := \max_{u \in \mathcal{W}}  \max_{p \in [P]} \| f'_p(u; \mathcal{D}_p) \|, \label{def_upper_1} \\ 
& \textstyle  U_2 := \max_{u, v \in \mathcal{W} } \|u-v\|, \label{def_upper_2} \\
& \textstyle  U_3 := \max_{u \in \mathcal{W}} \max_{p \in [P]} \max_{\mathcal{D}'_p \in \widehat{\mathcal{D}}_p} \| f'_p(u;\mathcal{D}_p) - f'_p(u;\mathcal{D}'_p)\|_1. \label{def_upper_3}
\end{align}
\end{subequations}

For fixed $t$, we derive from the first-order optimality condition of \eqref{ADMM-1}, 
namely, 
$\sum_{p=1}^P \lambda^t_p + \rho^t(w^{t+1}-z^t_p)=0$, that
\begin{align}
& \textstyle \sum_{p=1}^P \langle  \tilde{\lambda}^t_p, w^{t+1}-w \rangle = 0, \ \forall w, \label{optimality_condition}
\end{align}
where $\tilde{\lambda}^t_p := \lambda^t_p + \rho^t (w^{t+1} - z^t_p)$.

\begin{proposition} \label{prop:case_2_basic_inequality}
  Under Assumption \ref{assump:convergence}, for fixed $t$ and $p$, it follows from the subproblem \eqref{DPADMM-2-Prox} that
  \begin{align}
  & f_p(z^t_p) - f_p(z_p) - \langle \lambda^{t+1}_p, z^{t+1}_p - z_p \rangle  \nonumber \\
  & \leq  \textstyle \frac{\eta^t \| f'(z^t_p) + \tilde{\xi}^{t}_p \|^2}{2}   + \frac{1}{2\eta^t} \big( \|z_p-z^t_p\|^2  - \|z_p-z^{t+1}_p\|^2 \big) + \langle \tilde{\xi}^{t}_p, z_p - z^t_p \rangle, \ \forall z_p \in \mathcal{W}. \label{case_2_basic_inequality}
  \end{align}  
\end{proposition}
\begin{proof}
See Appendix \ref{apx-prop:case_2_basic_inequality}.
\end{proof}

\begin{theorem} \label{thm:convergence_rate_prox}
  Under Assumption \ref{assump:convergence}, we derive
  \begin{subequations}
    \begin{align}
      & \mathbb{E} \Big[ F(z^{(T)}) - F(z^*) + \gamma \| Aw^{(T)} - z^{(T)} \| \Big] \leq \frac{1}{T} \Big(  (PU_1^2 + 2PJKU_3^2/\bar{\epsilon}^2 + U_2/2)\sqrt{T} 
       \nonumber \\
      &   + \gamma U_2 + \frac{U_2 \rho^{\text{max}}}{2} + \frac{(\gamma + \| \lambda^1 \|)^2}{2\rho^1}  \Big), \label{case_2_inequality_added}
    \end{align}
    where $U_1$, $U_2$, $U_3$ are from \eqref{def_upper}, $z^*$ is an optimal solution, and 
  \begin{align}
    & w^{(T)} := \textstyle \frac{1}{T} \sum_{t=1}^T w^{t+1}, \ \ z^{(T)} := \textstyle \frac{1}{T} \sum_{t=1}^T z^{t}, \ \
    z^t := [(z_1^t)^{\top}, \ldots, (z_P^t)^{\top}]^{\top}, \nonumber  \\
    & F(z) := \textstyle \sum_{p=1}^P f_p(z_p), \ \
    \tilde{\xi}^t := [(\tilde{\xi}_1^t)^{\top}, \ldots, (\tilde{\xi}_P^t)^{\top}]^{\top}, \ \
    A^{\top} := \begin{bmatrix}
        \mathbb{I}_J \ \cdots \  \mathbb{I}_J
      \end{bmatrix}_{J \times PJ}.  \nonumber
  \end{align}  
  \end{subequations}
  The rate of convergence in expectation produced by Algorithm \ref{algo:DP-IADMM-Prox} is $\mathcal{O}(1/(\sqrt{T}\bar{\epsilon}^2))$.
  \end{theorem}
  \begin{proof}
  See Appendix \ref{apx-thm:convergence_rate_prox}.
  \end{proof}

\section{Numerical Experiments} \label{sec:experiments}
In this section we compare the proposed DP-IADMM-Prox (Algorithm \ref{algo:DP-IADMM-Prox}) and DP-IADMM-Trust (Algorithm \ref{algo:DP-IADMM-Trust}) with the state of the art in \cite{huang2019dp}, as a baseline algorithm. The algorithm in \cite{huang2019dp} has demonstrated more accurate solutions than the other existing DP algorithms, such as DP-SGD \cite{abadi2016deep}, DP-ADMM with the output perturbation method (Algorithm 2 in \cite{huang2019dp}), and DP-ADMM with the objective perturbation method \cite{zhang2016dynamic} (see Figure 6 in \cite{huang2019dp}).
Note that as a DP technique, the output perturbation method is used in the baseline algorithm in \cite{huang2019dp} while the objective perturbation method is used in our algorithms.
We implemented the algorithms in Python, and the experiments were run on Swing, a 6-node GPU computing cluster at  Argonne National Laboratory. Each node of Swing has 8 NVIDIA A100 40 GB GPUs, as well as 128 CPU cores.
The implementation is available at \url{https://github.com/APPFL/DPFL-IADMM-Classification.git}.

\textbf{Algorithms.}
We denote (i) our DP-IADMM-Prox with the objective perturbation (Algorithm \ref{algo:DP-IADMM-Prox}) by \texttt{ObjP}, (ii) our DP-IADMM-Trust with the objective perturbation (Algorithm \ref{algo:DP-IADMM-Trust}) by \texttt{ObjT}, and (iii) the baseline algorithm in \cite{huang2019dp} by \texttt{OutP}.
Note that \texttt{OutP} and \texttt{ObjP} are equivalent in a nonprivate setting.
In this experiment, we use the infinity norm for defining the trust-region in \texttt{ObjT}.

\textbf{FL Model.}
We consider a multiclass logistic regression model (see Appendix \ref{apx:model} for details).

\textbf{Instances.}
We consider two publicly available instances for image classification: MNIST \cite{lecun1998mnist} and FEMNIST \cite{caldas2018leaf}.
Using the MNIST dataset, we evenly distribute the training data to multiple agents to mimic a homogeneous system (i.e., each agent has the same number of data), and we use the FEMNIST dataset that describes a heterogeneous system.
In Table \ref{Table:instance}, we summarize some input parameters of the two instances.
\begin{table}[!h]
  \caption{Input parameters of MNIST and FEMNIST.}
  \centering
  \begin{tabular}{l|cccccc}
  \toprule
  & \# of data  & \# of features &  \# of classes & \# of agents & \multicolumn{2}{c}{\# of data per agent} \\
   &  ($I$)  & ($J$) &  ($K$) & ($P$) & mean & stdev \\ \midrule
  MNIST & 60000 & 784 & 10 & 10 & 6000 & 0 \\
  FEMNIST$^{\texttt{A}}$ & 36708 & 784 & 62 & 195 & 188.25 & 87.99 \\
  \bottomrule
  \end{tabular}
  \label{Table:instance}
  \begin{tablenotes}\scriptsize
    \item[] \texttt{A}. We extract 5\% of the FEMNIST training data.
  \end{tablenotes}
\end{table}

\textbf{Parameters.}
Under the multi-class logistic regression model, we compute $\bar{\Delta}_p^t$ in \eqref{Delta} as
\begin{align}
\bar{\Delta}_p^t = \textstyle \max_{i^* \in [I_p]} \sum_{j=1}^J \sum_{k=1}^K \Big| \frac{1}{I} \big\{ x_{pi^*j} \big(h_k(z^t_p;x_{pi^*}) - y_{pi^* k} \big)  \big\} \Big|. \label{sensitivity_cal}
\end{align}
Note that $\bar{\Delta}_p^t/\bar{\epsilon}$ is proportional to the standard deviation of the Laplace distribution in \eqref{Laplace-pdf}, thus controlling the noise level.
In the experiments, we consider various $\bar{\epsilon} \in \{0.01, 0.05, 0.1, 1, 3, 5 \}$, where stronger data privacy is achieved with smaller $\bar{\epsilon}$.

We emphasize that the baseline algorithm \texttt{OutP} guarantees $(\bar{\epsilon},\bar{\delta})$-DP, which provides stronger privacy as $\bar{\delta} > 0$ decreases for fixed $\bar{\epsilon}$, but still weaker than $\bar{\epsilon}$-DP.
In the experiment, we set $\bar{\delta}=10^{-6}$ for \texttt{OutP}.
In addition, we set the regularization parameter $\beta$ in \eqref{def_fn_f} by $\beta \leftarrow 10^{-6}$ as in \cite{huang2019dp}.

The parameter $\rho^t$ in Assumption \ref{assump:convergence} may affect the learning performance because it can affect the proximity of the local solution $z_p^{t+1}$ from the global solution $w^{t+1}$.
For all algorithms, we set $\rho^t \leftarrow \hat{\rho}^t$ given by
\begin{align}
  & \hat{\rho}^t := \min \{ 1e9, \ c_1 (1.2)^{\lfloor t/T_c \rfloor} + c_2/\bar{\epsilon}\}, \ \forall t \in [T], \label{dynamic_rho}
\end{align}
where (i) $c_1=2$, $c_2=5$, and $T_c=1e4$ for MNIST and (ii) $c_1=0.005$, $c_2=0.05$, and $T_c=2e3$ for FEMNIST, which are chosen based on the justifications described in Appendix \ref{apx-hyperparameter-rho}. Note that the chosen parameter $\hat{\rho}^t$ is nondecreasing and bounded above, thus satisfying Assumption \ref{assump:convergence} (ii).

\textbf{MNIST Results.}
Using MNIST described in Table \ref{Table:instance}, we compare the performances of \texttt{ObjP}, \texttt{ObjT}, and \texttt{OutP}.
For each algorithm and fixed $\bar{\epsilon}$, we generate $10$ instances, each of which has different realizations of the random noises.
The random noises to \texttt{OutP} are generated by the Gaussian mechanism with \textit{decreasing} variance as in~\cite{huang2019dp}, whereas the noises to our algorithms are generated by the Laplacian mechanism as in \eqref{DPADMM-2-Trust}.
To compare the two different mechanisms in terms of the magnitude of noises generated, we compute the following average noise magnitude:
\begin{align}
  \textstyle\frac{1}{PJK} \sum_{p=1}^P \sum_{j=1}^J \sum_{k=1}^K |\hat{\xi}^t_{pjk}|, \ \forall t \in [T],  \nonumber
\end{align}
where $\hat{\xi}^t_{pjk}$ is a realization of random noise $\tilde{\xi}^t_{pjk}$.

\begin{figure}[!h]
  \centering
  \begin{subfigure}[b]{0.32\textwidth}
      \centering
      \includegraphics[width=\textwidth]{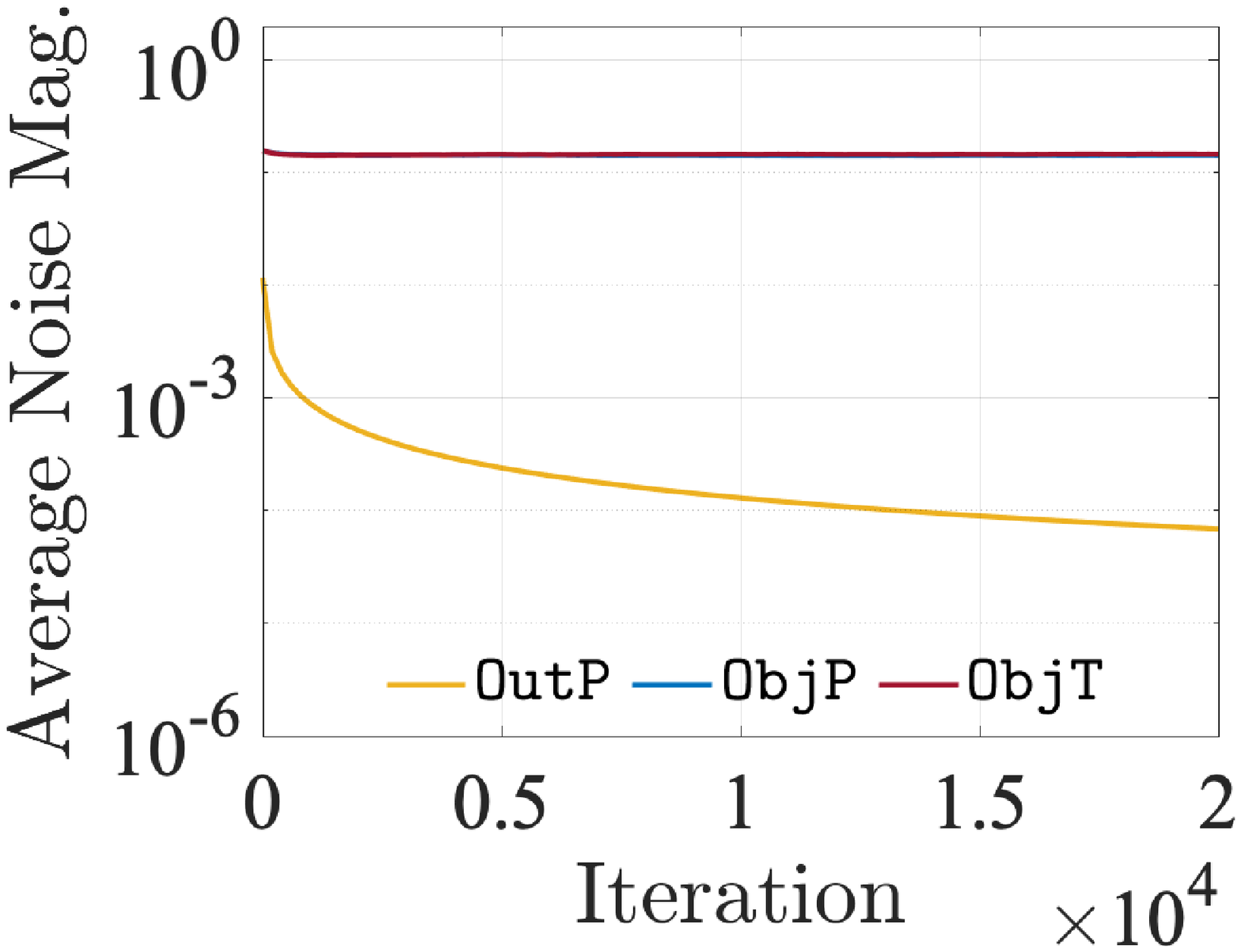}
      \includegraphics[width=\textwidth]{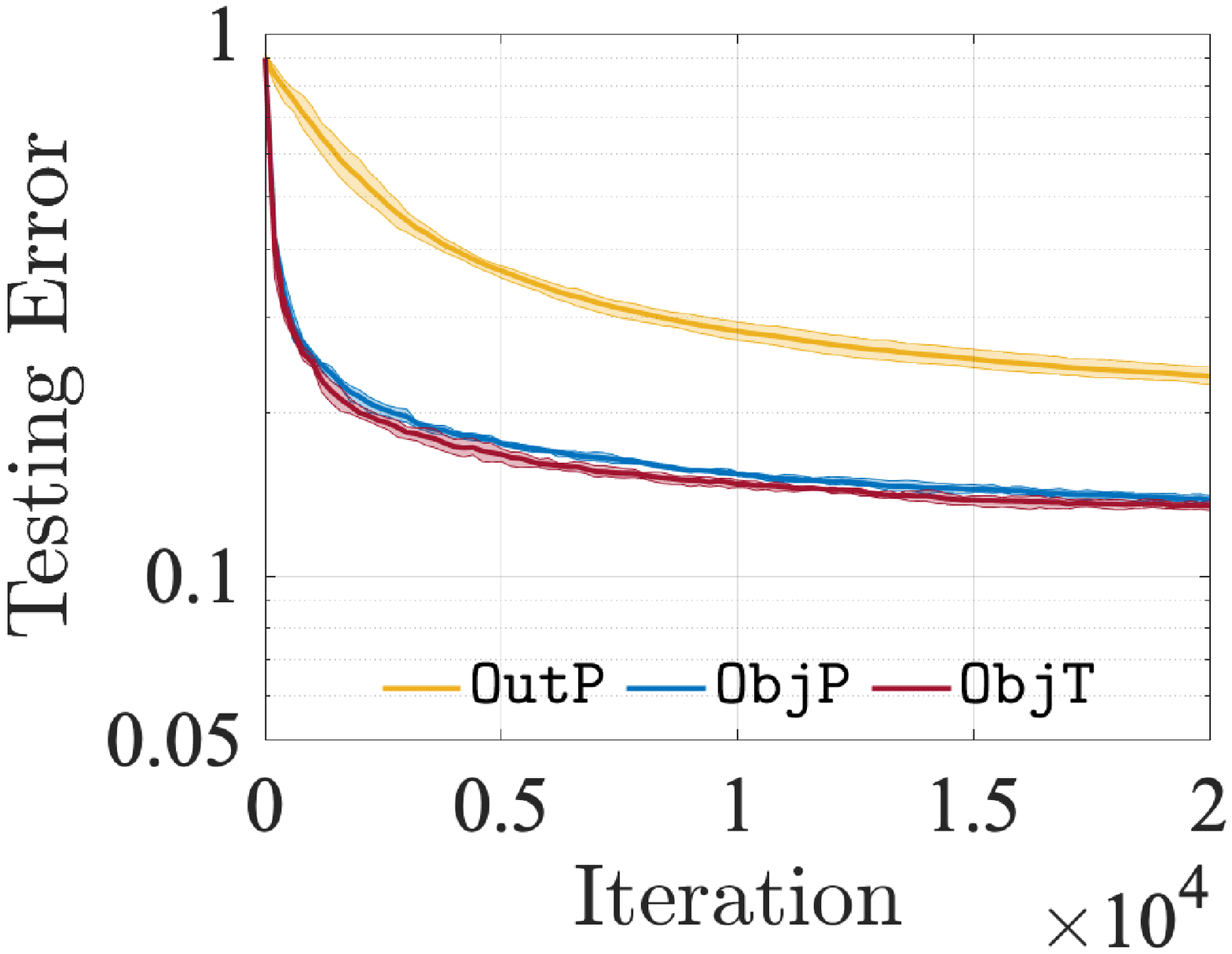}
      \caption{$\bar{\epsilon}=0.05$}
  \end{subfigure}
  \hfill
  \begin{subfigure}[b]{0.32\textwidth}
      \centering
      \includegraphics[width=\textwidth]{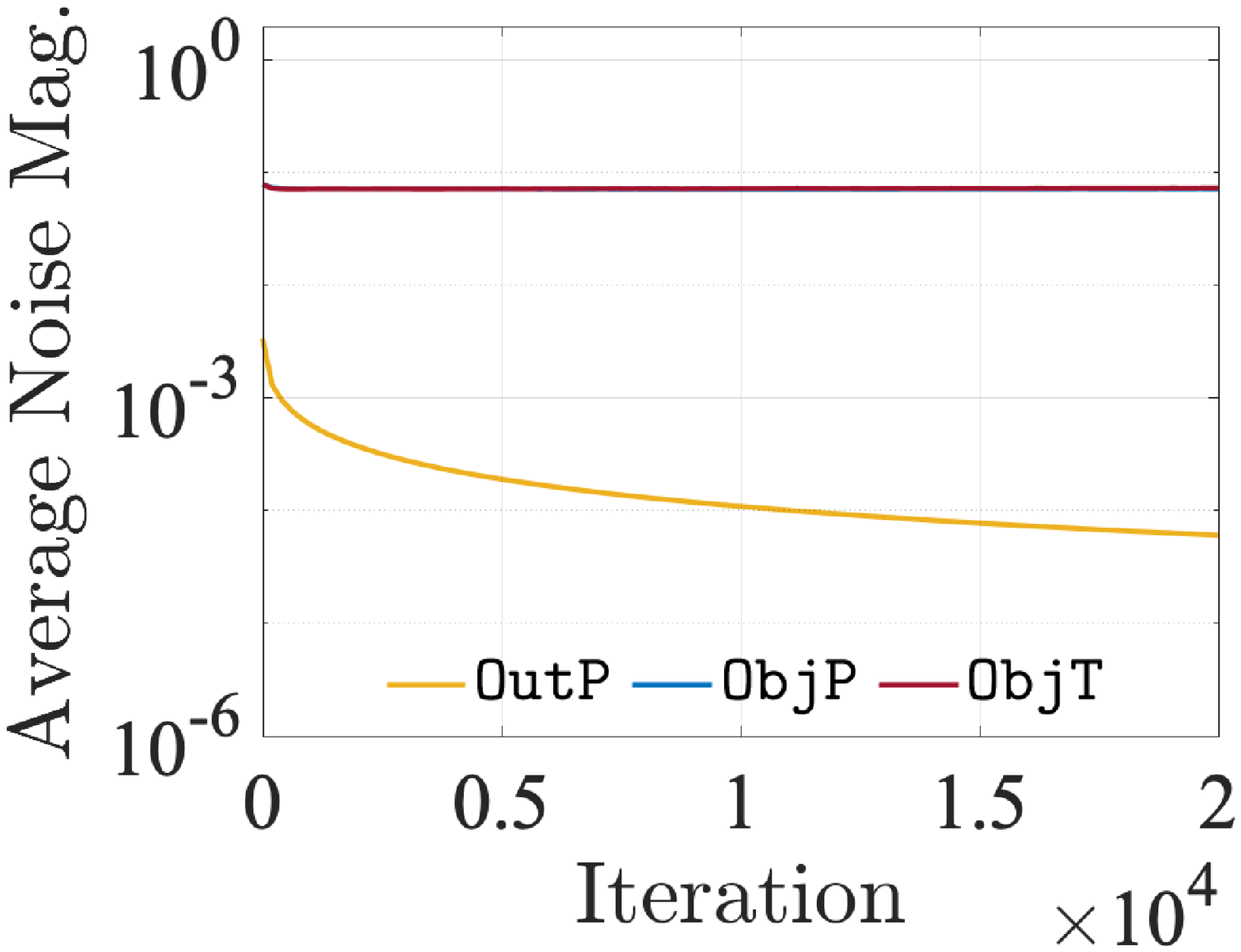}
      \includegraphics[width=\textwidth]{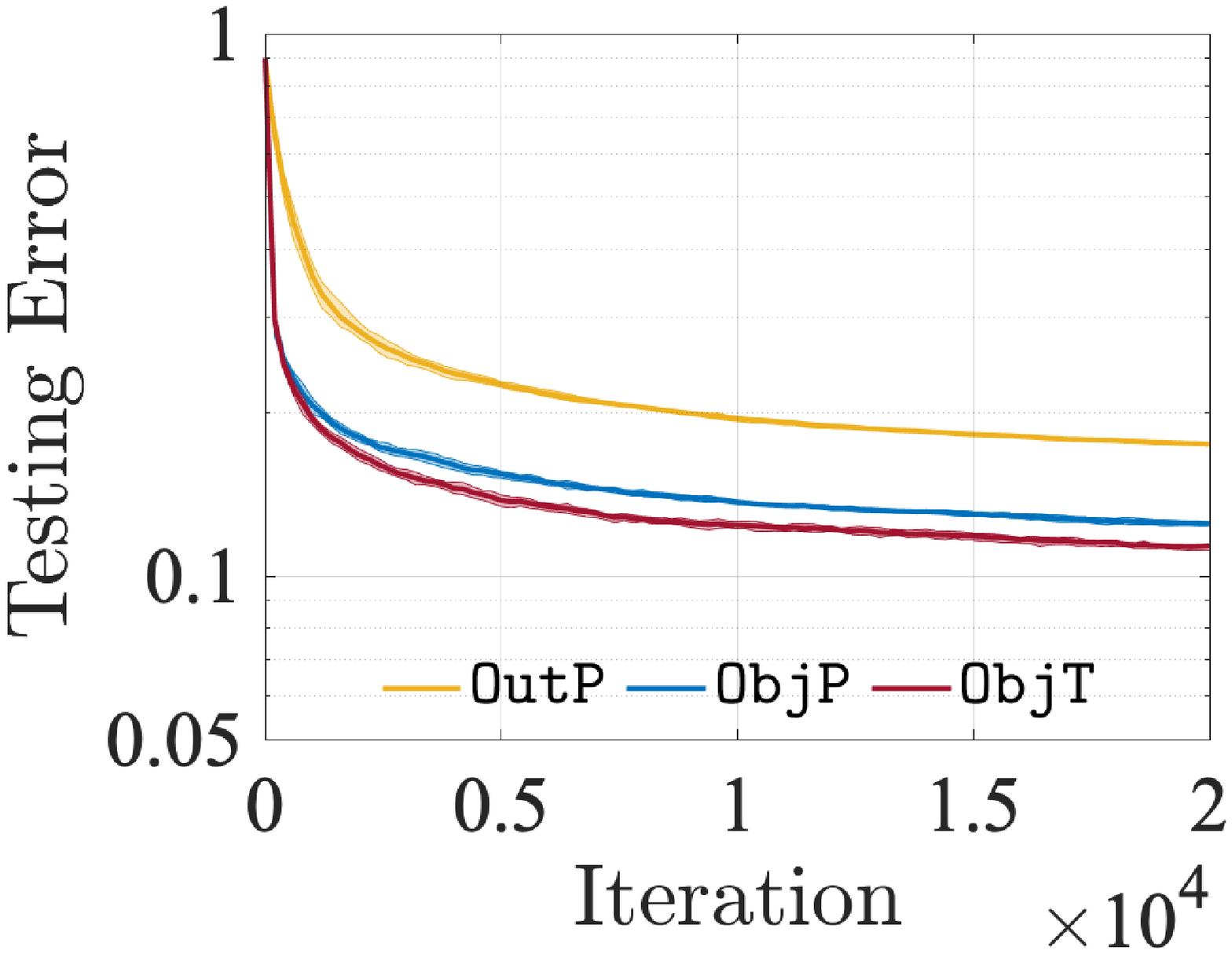}
      \caption{$\bar{\epsilon}=0.1$}
  \end{subfigure}
  \hfill
  \begin{subfigure}[b]{0.32\textwidth}
      \centering
      \includegraphics[width=\textwidth]{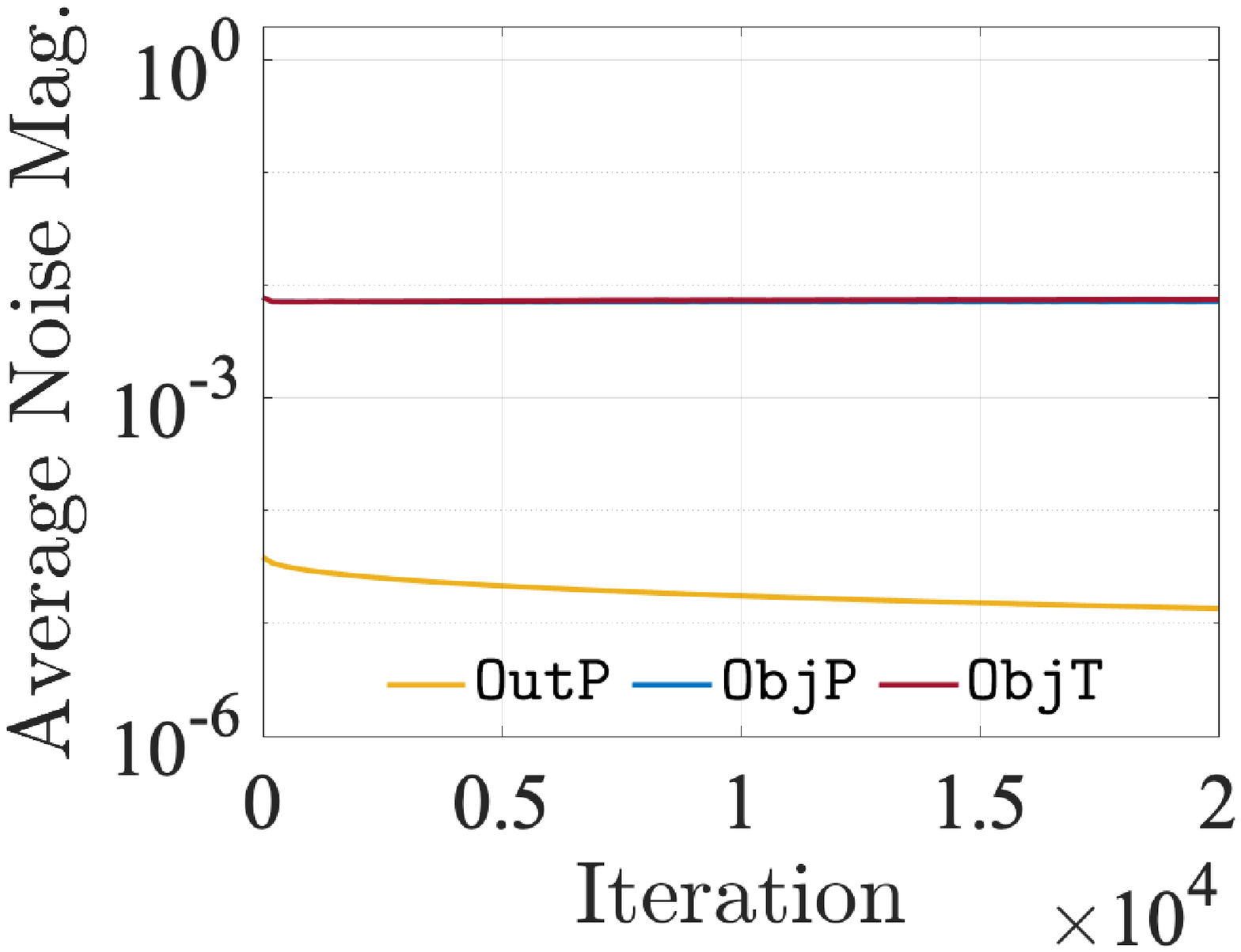}
      \includegraphics[width=\textwidth]{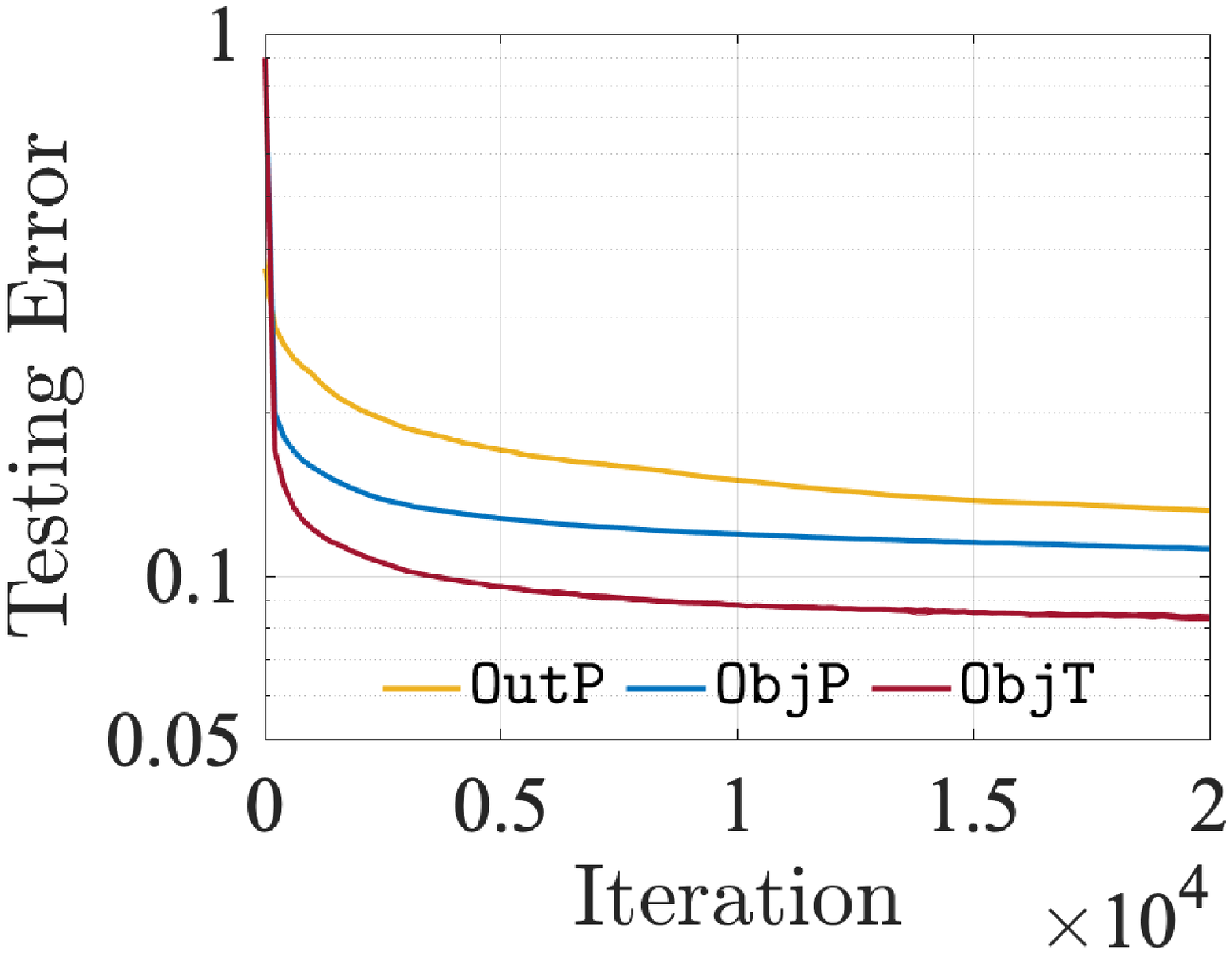}
      \caption{$\bar{\epsilon}=1$}
  \end{subfigure}
  \caption{[MNIST] Average noise magnitudes (top) and testing errors (bottom) for every iteration.}
  \label{fig:MNIST_Iteration}
\end{figure}

In Figure \ref{fig:MNIST_Iteration}, for every algorithm, $\bar{\epsilon} \in \{0.05, 0.1, 1\}$, and iteration $t \in [2e4]$, we report the average noise magnitudes and testing errors on average (solid line) with the $20$- and $80$-percentile confidence bounds (shaded), respectively in the top and bottom rows of the figure.
We exclude the cases when $\bar{\epsilon} \in \{3,5\}$ in Figure \ref{fig:MNIST_Iteration}, since \texttt{ObjT} provides an accurate solution even when $\bar{\epsilon}=1$.
In what follows, we present some observations from the figures and their implications.
The average noise magnitudes of all the algorithms increase as $\bar{\epsilon}$ decreases, achieving stronger data privacy.
% This implies that larger noises are required for stronger data privacy, which makes sense as the standard deviations of the probability distributions in \eqref{DPADMM-2-Trust} and \cite{huang2019dp} increase as $\bar{\epsilon}$ decreases.
For fixed $\bar{\epsilon}$, the average noise magnitudes of our algorithms \texttt{ObjT} and \texttt{ObjP} are greater than those of \texttt{OutP} while the testing errors of our algorithms are less than those of \texttt{OutP}.
These results imply that the performance of our algorithms is less sensitive to the random perturbation than that of \texttt{OutP}, even with a larger magnitude of noises for stronger $\bar{\epsilon}$-DP.
The greater performance of our algorithms is also consistent with the findings in \cite{chaudhuri2011differentially,zhang2016dynamic} that the better performance of the objective perturbation than the output perturbation is guaranteed with higher probability.
The sequence of solutions produced by our algorithms, especially \texttt{ObjT}, converges faster than that produced by \texttt{OutP}. 
% This is also consistent with our theoretical finding of the faster convergence rate result presented in Theorem \ref{thm:convergence_rate}.
% Also, as $\bar{\epsilon}$ decreases, the convergence of \texttt{ObjT} becomes slower because the rate of convergence in Theorem \ref{thm:convergence_rate} is affected by $\bar{\epsilon}$.

\begin{figure}[!h]
  \centering
  \begin{subfigure}[b]{0.4\textwidth}
    \centering
    \includegraphics[width=\textwidth]{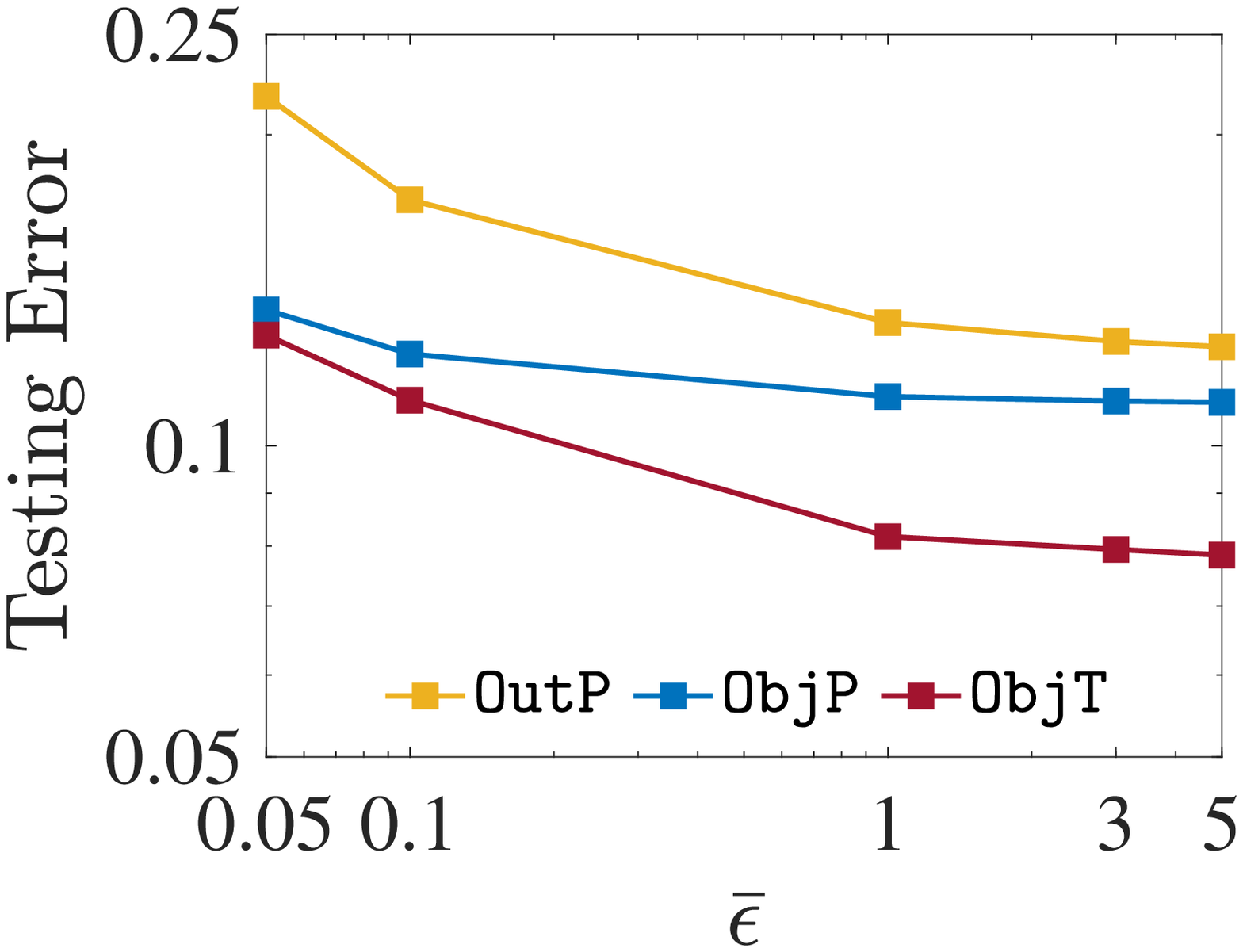}
    \caption{MNIST}
    \label{fig:MNIST_Summary}
  \end{subfigure}
  \hfill
  \begin{subfigure}[b]{0.4\textwidth}
    \centering
    \includegraphics[width=\textwidth]{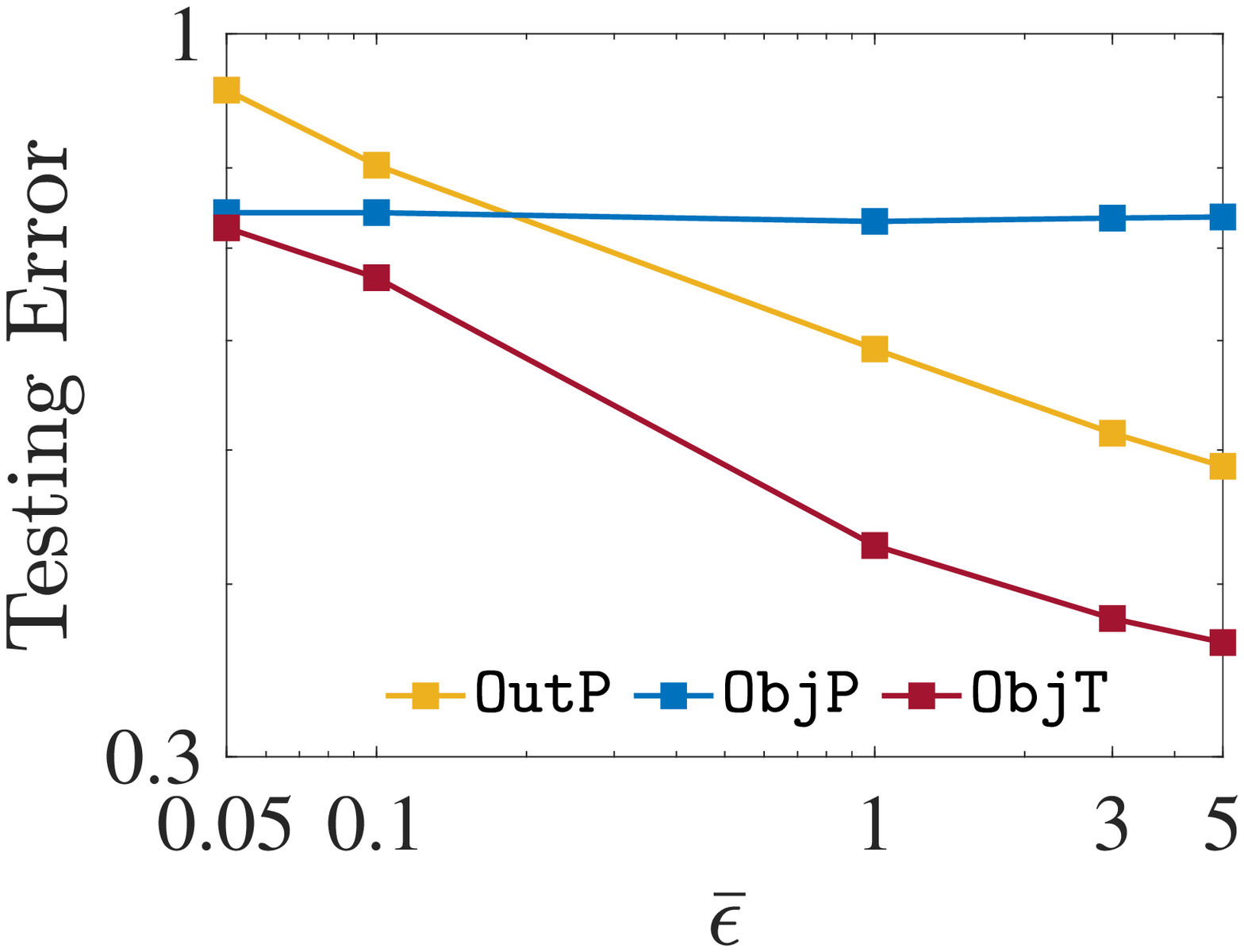}
    \caption{FEMNIST}
    \label{fig:FEMNIST_Summary}
  \end{subfigure}
  \caption{Testing errors of the three algorithms under various $\bar{\epsilon}$. }
\end{figure}

In Figure \ref{fig:MNIST_Summary} we report the testing errors of the three algorithms for every $\bar{\epsilon} \in \{0.05, 0.1, 1, 3, 5 \}$.
When $\bar{\epsilon}=5$, the testing error produced by \texttt{ObjT} is $7.84\%$, which is close to that of a nonprivate algorithm (i.e., $7.42\%$).
As $\bar{\epsilon}$ decreases (i.e., stronger data privacy), the testing errors of all algorithms increase,  implying a fundamental trade-off between solution accuracy and data privacy.
When $\bar{\epsilon}=0.05$, the testing error of \texttt{ObjT} is $12.80\%$ while that of \texttt{OutP} is $21.79\%$, an $8.99\%$ improvement.

\begin{remark}
Additionally, we increase the number of iterations to $T=1e6$ and verify that the solutions provided by the three algorithms are feasible, namely, satisfying the consensus constraints \eqref{ERM_1-1} (see Appendix \ref{apx-residual} for more details). Under this setting, we additionally consider a case when $\bar{\epsilon}=0.01$, and we demonstrate that the testing error of \texttt{OutT} is $15.64\%$ while that of \texttt{OutP} is $37.98\%$, a $22.34\%$ improvement.
In summary, the results demonstrate the outperformance of our algorithms.
\end{remark}

\textbf{FEMNIST Results.}
Using FEMNIST described in Table \ref{Table:instance}, we aim to show that our algorithms outperform \texttt{OutP} under the heterogeneous data setting (i.e., the number of data per agent varies).

\begin{figure}[!h]
  \centering
  \begin{subfigure}[b]{0.32\textwidth}
    \centering
    \includegraphics[width=\textwidth]{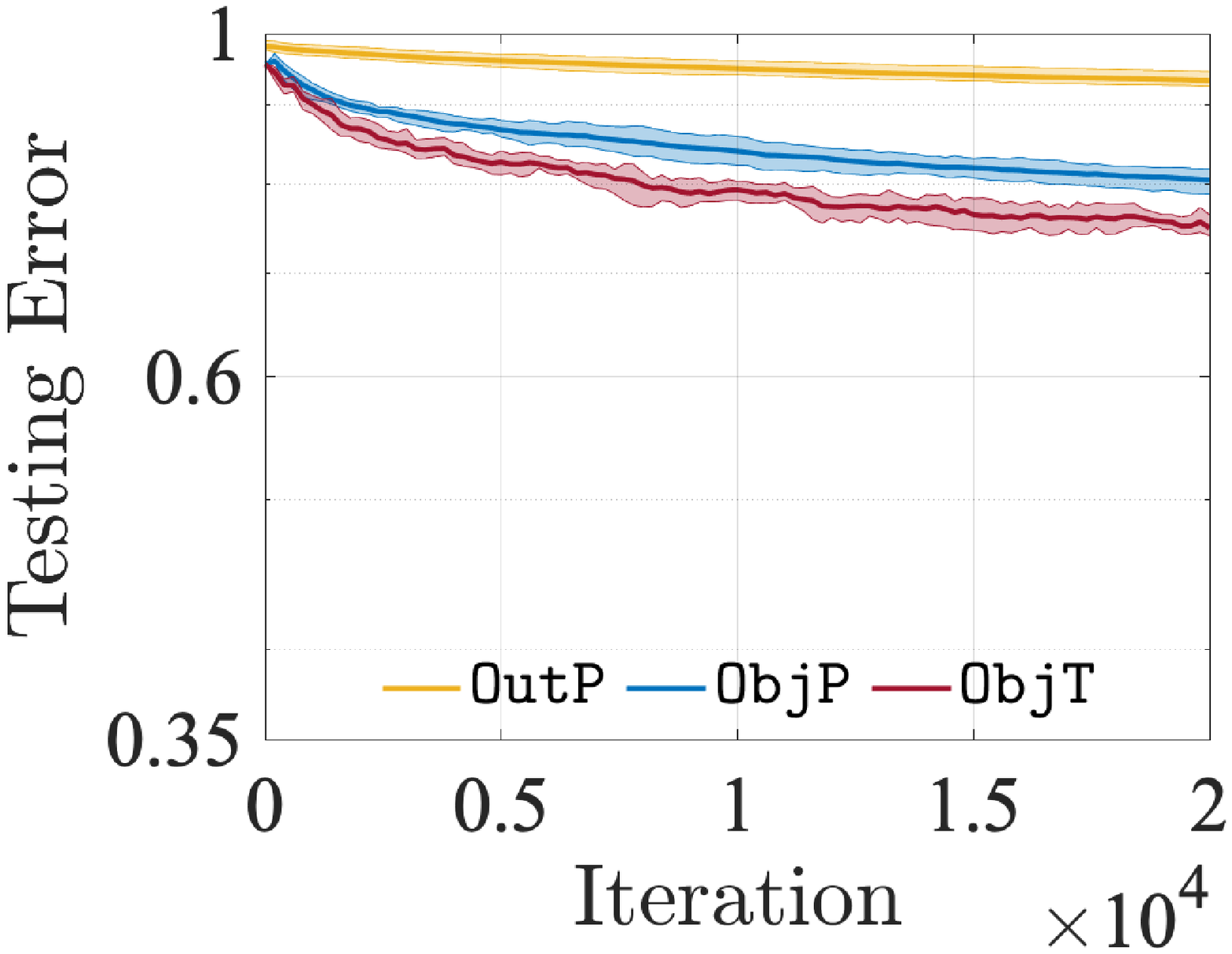}
    \caption{$\bar{\epsilon}=0.05$}
  \end{subfigure}
  \hfill
  \begin{subfigure}[b]{0.32\textwidth}
    \centering
    \includegraphics[width=\textwidth]{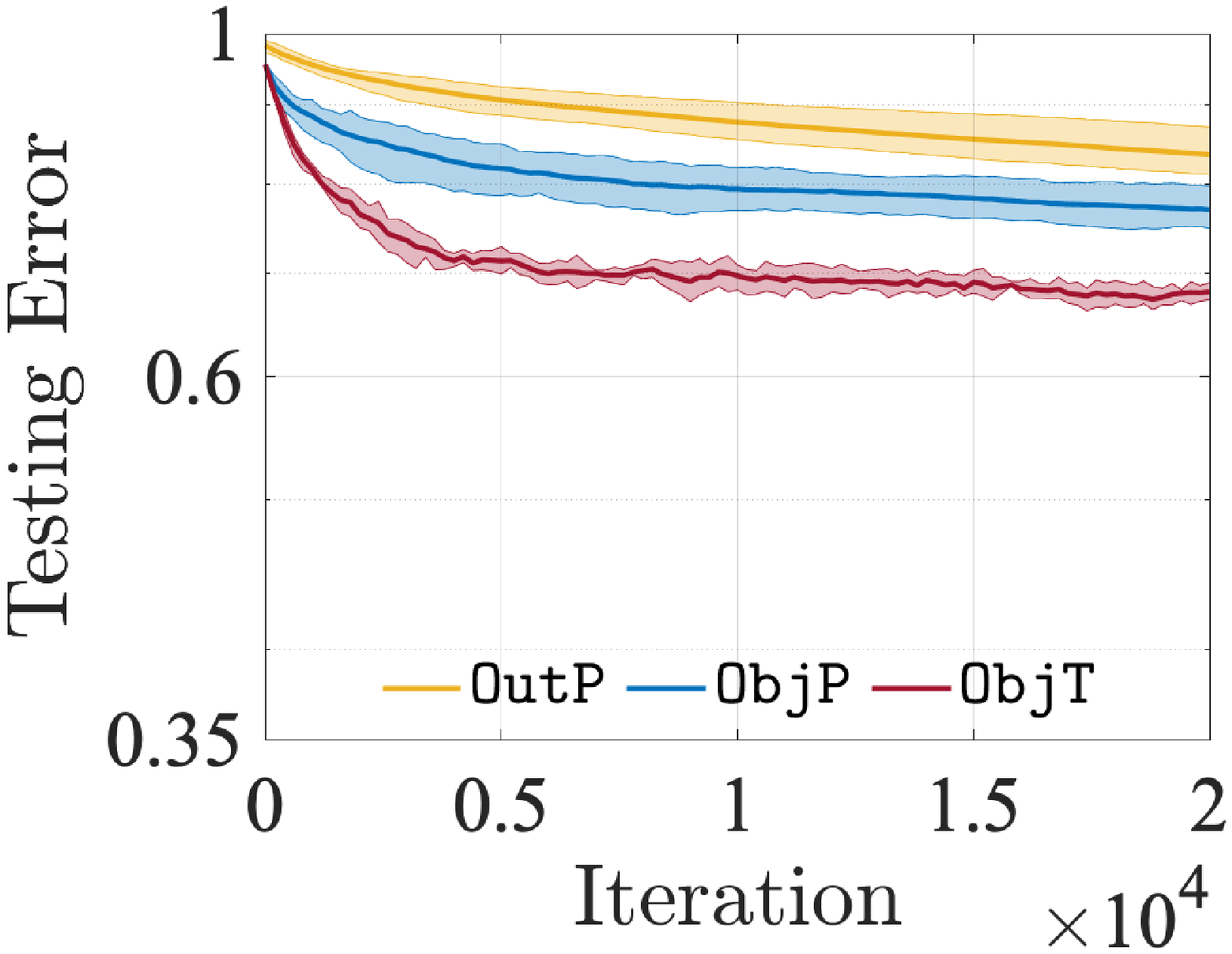}
    \caption{$\bar{\epsilon}=0.1$}
  \end{subfigure}
  \hfill
  \begin{subfigure}[b]{0.32\textwidth}
      \centering
      \includegraphics[width=\textwidth]{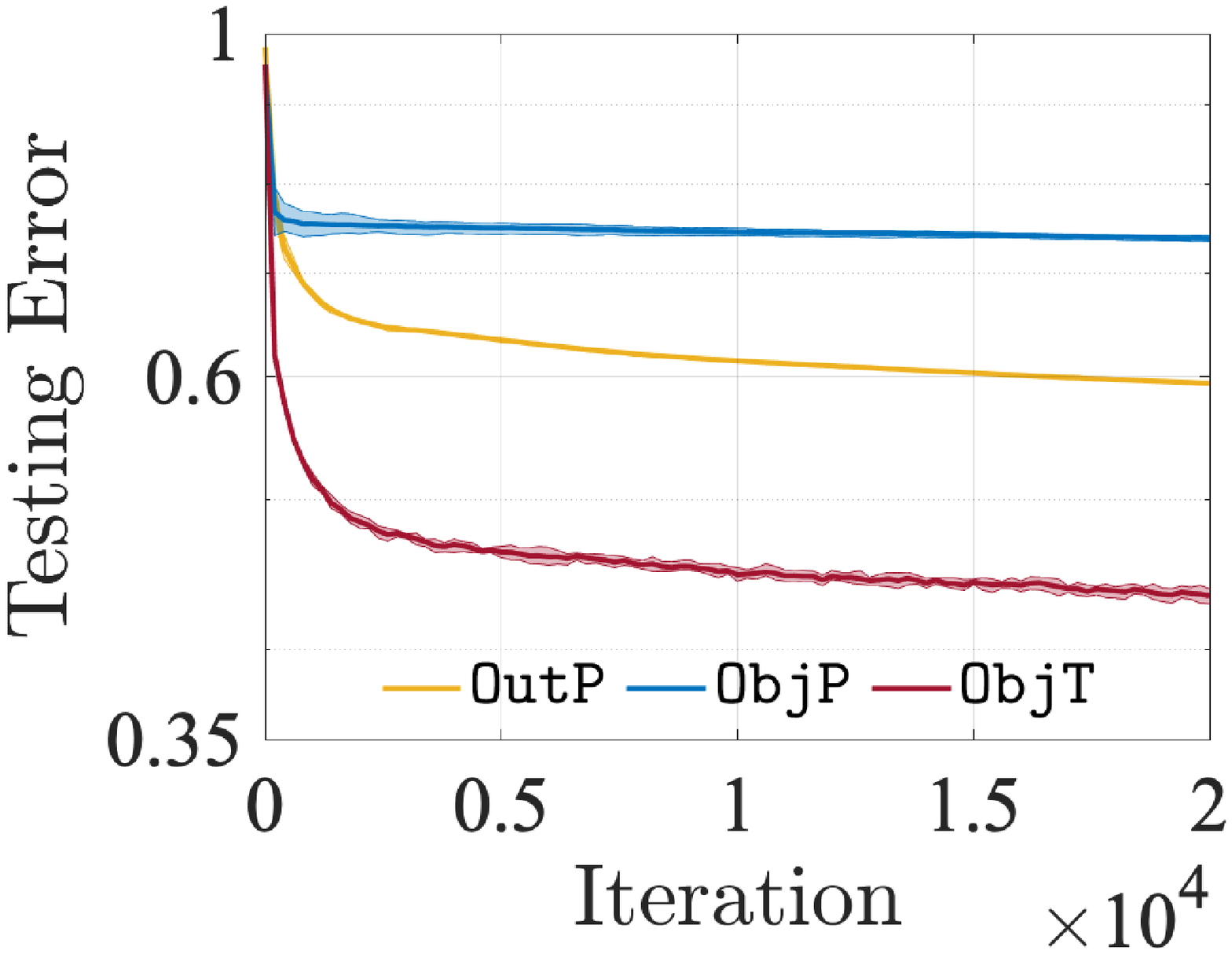}
      \caption{$\bar{\epsilon}=1$}
  \end{subfigure}
     \caption{[FEMNIST] Testing errors for every iteration.}
     \label{fig:FEMNIST_TraningError}
\end{figure}

In Figure \ref{fig:FEMNIST_TraningError}, for every algorithm, $\bar{\epsilon} \in \{0.05, 0.1, 1\}$, and iteration $t \in [2e4]$, we report the testing errors on average (solid line) with the $20$- and $80$-percentile confidence bounds (shaded).
In what follows, we present some observations from the figures and their implications.
\texttt{ObjT} produces the least testing error with the fastest convergence, which is similar to the result from Figure \ref{fig:MNIST_Iteration}.
When $\bar{\epsilon}=1$, the testing error of \texttt{ObjP} is greater than that of \texttt{OutP}.
To see this in more detail, we also note that the effect of $\bar{\epsilon}$ on the testing error of \texttt{ObjP} is not significant (see also the blue line in Figure \ref{fig:FEMNIST_Summary}).
In Appendix \ref{apx-hyperparameter-proximity} we verify that \texttt{ObjP} requires an additional hyperparameter tuning process since the proximity is controlled in the objective function and thus affected by the other parameters, such as $\rho^t$.
However, we highlight that this additional tuning process is not required for \texttt{ObjT} since the proximity is controlled in the constraints and thus not affected by the other parameters. Taking  this viewpoint, \texttt{ObjT} has an additional advantage  over \texttt{ObjP}.

In Figure \ref{fig:FEMNIST_Summary} we report the testing errors of the three algorithms for every $\bar{\epsilon} \in \{0.05, 0.1, 1, 3, 5 \}$.
When $\bar{\epsilon}=5$, the testing error of \texttt{ObjT} is $36.29\%$ which is close to that of a nonprivate algorithm (i.e., $35.25\%$).
As $\bar{\epsilon}$ decreases (i.e., stronger data privacy), the testing errors of all algorithms increase, thus implying a trade-off between solution accuracy and data privacy.
When $\bar{\epsilon}=0.05$, the testing error of \texttt{ObjT} is $72.40\%$ while that of \texttt{OutP} is $91.05\%$, an $18.65\%$ improvement.

\section{Conclusion} \label{sec:conclusion}
We incorporated the objective perturbation into an IADMM algorithm for solving the FL model while ensuring data privacy during a training process.
The proposed DP-IADMM algorithm iteratively solves a sequence of subproblems whose objective functions are randomly perturbed by noises sampled from a calibrated Laplace distribution to ensure $\bar{\epsilon}$-DP.
We showed that the rate of convergence in expectation for the proposed Algorithm \ref{algo:DP-IADMM-Prox} is $\mathcal{O}(1/\sqrt{T})$ with $T$ being the number of iterations.
In the numerical experiments, we demonstrated the outperformance of the proposed algorithm by using MNIST and FEMNIST instances.

We note that the performance of the proposed DP algorithm can be further improved by lowering the magnitude of noises required for ensuring the same level of data privacy (see Figure \ref{fig:MNIST_Iteration} (top) that our algorithm requires larger noises). By improving the performance further, we expect that the proposed DP algorithm can be utilized for learning from larger decentralized datasets with more features and classes.

\textbf{\textit{Acknowledgments.}}
This material was based upon work  supported by the U.S. Department of Energy, Office of Science, Advanced Scientific Computing Research, under Contract DE-AC02-06CH11357.
We gratefully acknowledge the computing resources provided on Swing, a high-performance computing cluster operated by the Laboratory Computing Resource Center at Argonne National Laboratory.

\bibliographystyle{plain}
\bibliography{References}

\vspace{0.5in}
\noindent\fbox{\parbox{\textwidth}{
The submitted manuscript has been created by UChicago Argonne, LLC, Operator of Argonne National Laboratory (``Argonne''). Argonne, a U.S. Department of Energy Office of Science laboratory, is operated under Contract No. DE-AC02-06CH11357. The U.S. Government retains for itself, and others acting on its behalf, a paid-up nonexclusive, irrevocable worldwide license in said article to reproduce, prepare derivative works, distribute copies to the public, and perform publicly and display publicly, by or on behalf of the Government. The Department of Energy will provide public access to these results of federally sponsored research in accordance with the DOE Public Access Plan (http://energy.gov/downloads/doe-public-access-plan).}
}

\newpage
\appendix
\section{Appendix} \label{sec:appendix}
\subsection{Derivation of \eqref{function_G_2}} \label{apx-derivation}
By adding $\langle \tilde{\xi}^t_p, z_p \rangle$ to the function \eqref{function_G_1}, we have
\begin{align}
&  \langle f'_p(z^t_p;\mathcal{D}_p), \ z_p  \rangle + \textstyle \frac{\rho^t}{2}\|w^{t+1}-z_p + \frac{1}{\rho^t} \lambda^t_p \|^2 + \langle \tilde{\xi}^t_p, z_p \rangle. \label{subproblem_1}
\end{align}
Now we add a constant $\frac{1}{2\rho^t} \| \tilde{\xi}^t_p \|^2 - \langle w^{t+1} + \frac{1}{\rho^t} \lambda^t_p, \tilde{\xi}^t_p \rangle$ to \eqref{subproblem_1}, yielding
\begin{align}
& \langle f'_p(z^t_p;\mathcal{D}_p), \ z_p  \rangle + \textstyle \frac{\rho^t}{2} \|w^{t+1}-z_p + \frac{1}{\rho^t} \lambda^t_p \|^2  + \frac{1}{2\rho^t} \| \tilde{\xi}^t_p \|^2 - \langle w^{t+1} - z_p + \frac{1}{\rho^t} \lambda^t_p, \tilde{\xi}^t_p \rangle \nonumber \\
= &  \langle f'_p(z^t_p;\mathcal{D}_p), \ z_p  \rangle + \textstyle \frac{\rho^t}{2} \Big\{ \|w^{t+1}-z_p + \frac{1}{\rho^t} \lambda^t_p \|^2 + \frac{1}{(\rho^t)^2} \| \tilde{\xi}^t_p \|^2 - \frac{2}{\rho^t} \langle w^{t+1} - z_p + \frac{1}{\rho^t} \lambda^t_p, \tilde{\xi}^t_p \rangle \Big\} \nonumber  \\
= &  \langle f'_p(z^t_p;\mathcal{D}_p), \ z_p  \rangle + \textstyle \frac{\rho^t}{2} \|w^{t+1}-z_p + \frac{1}{\rho^t} \lambda^t_p - \frac{1}{\rho^t}  \tilde{\xi}^t_p \|^2,  \nonumber
\end{align}
which is equivalent to \eqref{function_G_2}.

\subsection{Proof of Proposition \ref{prop:pointwise_convergence} } \label{apx-prop:pointwise_convergence}
Fix $t$ and $p$.
We denote by $\hat{z}^{t+1}_p$ (resp., $\hat{z}^{t+1}_{p \ell}$) the unique optimal solution of an optimization problem in \eqref{DPADMM-2-Prox} (resp., \eqref{ADMM-2-Prox-log}), where the uniqueness is due to the strong convexity of the objective functions.
For ease of exposition, we define $G^t_p(z_p) := G^t(z_p;\mathcal{D}_p,\tilde{\xi}^t_p) + (1/2\eta^t)\|z_p-z_p^t\|^2$ and $g_{p\ell}(z_p) := g(z_p;\ell)$.

In \eqref{DPADMM-2-Prox}, the continuity of $G^t_p: \mathcal{W} \mapsto \mathbb{R}$ at $\hat{z}^{t+1}_p$ implies that, for every $\epsilon > 0$, there exists a $\delta > 0$ such that for all $z \in \mathcal{W}$:
\begin{subequations}
\begin{align}
z \in \mathcal{B}_{\delta}(\hat{z}^{t+1}_p) := \{ z \in \mathbb{R}^{J \times K} : \| z - \hat{z}^{t+1}_p \| < \delta \} \ \Rightarrow \  G^t_p(z) - G^t_p(\hat{z}^{t+1}_p)  < \epsilon. \label{continuity}
\end{align}
Consider $\tilde{z} \in \mathcal{B}_{\delta}(\hat{z}^{t+1}_p) \cap \textbf{relint}(\mathcal{W})$, where \textbf{relint} indicates the relative interior.
As $h_m(\tilde{z}) < 0$ for all $m \in [M]$, $g_{p\ell}(\tilde{z})$ goes to zero as $\ell$ increases.
Hence, there exists $\ell' > 0$ such that
\begin{align}
  g_{p\ell}(\tilde{z}) = \textstyle \sum_{m=1}^M \ln (1 + e^{\ell h_m(\tilde{z})}) < \epsilon, \ \forall \ell \geq \ell'. \label{ineq_z_tilde}
\end{align}
For all $\ell \geq \ell'$, we derive the following inequalities:
\begin{align}
  G^t_p(\hat{z}^{t+1}_{p\ell}) + g_{p\ell}(\hat{z}^{t+1}_{p\ell}) \leq G^t_p(\tilde{z}) + g_{p\ell}(\tilde{z}) < G^t_p(\tilde{z}) + \epsilon < G^t_p(\hat{z}^{t+1}_p) + 2\epsilon \label{sandwich_ineq}
\end{align}
\end{subequations}
where
the first inequality holds because $\hat{z}^{t+1}_{p\ell}$ is the optimal solution of \eqref{ADMM-2-Prox-log},
the second inequality holds by \eqref{ineq_z_tilde}, and
the last inequality holds by \eqref{continuity}.
The inequalities \eqref{sandwich_ineq} imply that, for very small $\epsilon \approx 0$, the optimal value $G^t_p(\hat{z}^{t+1}_{p\ell}) + g_{p\ell}(\hat{z}^{t+1}_{p\ell})$ of \eqref{ADMM-2-Prox-log} converges to the optimal value $G^t_p(\hat{z}^{t+1}_p)$ of \eqref{DPADMM-2-Prox} as $\ell$ increases.

It remains to show that $\hat{z}^{t+1}_{p\ell}$ converges to $\hat{z}^{t+1}_p$ as $\ell$ increases.
Suppose that $\hat{z}^{t+1}_{p\ell}$ converges to $\widehat{z} \neq \hat{z}^{t+1}_p$ as $\ell$ increases.
Consider $\zeta := \| \widehat{z} - \hat{z}^{t+1}_p \| / 2$.
Since $\hat{z}^{t+1}_{p\ell}$  converges to $\widehat{z}$, there exists $\ell'' > 0$ such that $\| \widehat{z} - \hat{z}^{t+1}_{p\ell} \| < \zeta$ for all $\ell \geq \ell''$.
By the triangle inequality, we have
\begin{subequations}
\begin{align}
\| \hat{z}^{t+1}_{p\ell} - \hat{z}^{t+1}_p \| \geq \| \widehat{z} - \hat{z}^{t+1}_p \| - \| \widehat{z} - \hat{z}^{t+1}_{p\ell} \| > 2\zeta - \zeta = \zeta, \ \forall \ell \geq \ell''. \label{triangle}
\end{align}
As $G^t_p$ is strongly convex with a constant $\rho^{\text{\tiny min}} > 0$, we have
\begin{align}
G^t_p(\hat{z}^{t+1}_{p\ell}) - G^t_p(\hat{z}^{t+1}_p) \geq \textstyle \frac{\rho^{\text{\tiny min}}}{2} \| \hat{z}^{t+1}_{p\ell} - \hat{z}^{t+1}_p \|^2 >  \frac{\rho^{\text{\tiny min}} \zeta^2}{2}, \ \forall \ell \geq \ell'', \label{strong_triangle}
\end{align}
where the last inequality holds by \eqref{triangle}.
By adding $g_{p \ell} (\hat{z}^{t+1}_{p\ell}) \geq 0$ to both sides of \eqref{strong_triangle}, we derive the following inequality:
\begin{align}
\big\{ G^t_p(\hat{z}^{t+1}_{p\ell}) + g_{p \ell} (\hat{z}^{t+1}_{p\ell})  \big\} - G^t_p(\hat{z}^{t+1}_p)  > \textstyle \frac{\rho^{\text{\tiny min}} \zeta^2}{2} + g_{p \ell} (\hat{z}^{t+1}_{p\ell}), \ \forall \ell \geq \ell'',
\end{align}
which contradicts that the optimal value $G^t_p(\hat{z}^{t+1}_{p\ell}) + g_{p\ell}(\hat{z}^{t+1}_{p\ell})$ of \eqref{ADMM-2-Prox-log} converges to the optimal value $G^t_p(\hat{z}^{t+1}_p)$ of \eqref{DPADMM-2-Prox} as $\ell$ increases.
This completes the proof.
\end{subequations}

\subsection{Proof of Proposition \ref{prop:dpinapproximation}} \label{apx-prop:dpinapproximation}
Fix $t$, $p$, and $\ell$. It suffices to show that the following is true:
\begin{subequations}
\begin{align}
  e^{-\bar{\epsilon}} \ \textbf{pdf} \big( z^{t+1}_p (\ell; \mathcal{D}'_p) = \alpha \big)
  \leq \textbf{pdf} \big( z^{t+1}_p (\ell; \mathcal{D}_p) = \alpha \big)
  \leq e^{\bar{\epsilon}} \ \textbf{pdf} \big( z^{t+1}_p (\ell; \mathcal{D}'_p) = \alpha \big), \ \forall \alpha \in \mathbb{R}^{J \times K}, \label{DP_PDF}
\end{align}
where $\textbf{pdf}$ represents a probability density function.

Consider an $\alpha \in \mathbb{R}^{J \times K}$.
If we have $z^{t+1}_p(\ell;\mathcal{D}_p)=\alpha$, then $\alpha$ is the unique minimizer of \eqref{ADMM-2-Prox-log} because the objective function in \eqref{ADMM-2-Prox-log} is strongly convex.
Setting the gradient of the objective function in \eqref{ADMM-2-Prox-log} to zero yields
\begin{align}
  \tilde{\xi}^{t}_p (\alpha; \mathcal{D}_p)= & - f'_p(z^t_p;\mathcal{D}_p) + \rho^t(w^{t+1}-\alpha) + \lambda^t_p - \nabla g(\alpha; \ell)  - \textstyle \frac{1}{\eta^t} \big( \alpha - z_p^t \big), \label{correspondence}
\end{align}
where $\nabla g(\alpha; \ell) = \sum_{m=1}^M \frac{\ell e ^{\ell h_m(\alpha)}}{1+e^{\ell h_m(\alpha)}} \nabla h_m(\alpha)$.
Therefore, the relation between $\alpha$ and $\tilde{\xi}^t_p$ is bijective, which enables us to utilize the inverse function theorem (Theorem 17.2 in \cite{billingsley}), namely
  \begin{align}
  \textbf{pdf} \big( z^{t+1}_p (\ell; \mathcal{D}_p) = \alpha \big) \cdot \big|\textbf{det}[\nabla \tilde{\xi}^{t}_p (\alpha;\mathcal{D}_p) ] \big| = L \big( \tilde{\xi}^{t}_p (\alpha;\mathcal{D}_p); \bar{\epsilon}, \bar{\Delta}_p^t \big), \label{inverse}
  \end{align}
  where \textbf{det} represents a determinant of a matrix, $L$ is from \eqref{Laplace-pdf}, and
  $\nabla \tilde{\xi}^{t}_p (\alpha;\mathcal{D}_p)$ represents a Jacobian matrix of the mapping from $\alpha$ to $\tilde{\xi}^{t}_p$, namely
  \begin{align}
  \nabla \tilde{\xi}^{t}_p (\alpha;\mathcal{D}_p) =  (-\rho^t - 1/\eta^t) \mathbb{I}_{JK} - \nabla \Big( \sum_{m=1}^M \frac{\ell e ^{\ell h_m(\alpha)}}{1+e^{\ell h_m(\alpha)}} \nabla h_m(\alpha) \Big) , \label{Jacobian}
  \end{align}
  where $\mathbb{I}_{JK}$ is an identity matrix of $JK \times JK$ dimensions.
  As the Jacobian matrix is not affected by the dataset, we have
  \begin{align}
  \nabla \tilde{\xi}^{t}_p (\alpha;\mathcal{D}_p) = \nabla \tilde{\xi}^{t}_p (\alpha;\mathcal{D}'_p).  \label{jacobian}
  \end{align}
  Based on \eqref{inverse} and \eqref{jacobian}, we derive the following inequalities:
  \begin{align}
    & \frac{\textbf{pdf} \big( z^{t+1}_p(\ell; \mathcal{D}_p) = \alpha \big)}{\textbf{pdf} \big( z^{t+1}_p(\ell; \mathcal{D}'_p ) = \alpha \big)}  = \frac{L \big( \tilde{\xi}^{t}_p (\alpha; \mathcal{D}_p); \bar{\epsilon}, \bar{\Delta}_p^t \big) }{L \big( \tilde{\xi}^{t}_p (\alpha; \mathcal{D}'_p); \bar{\epsilon}, \bar{\Delta}_p^t \big) } \cdot \frac{\big|\textbf{det}[\nabla \tilde{\xi}^{t}_p (\alpha;\mathcal{D}'_p) ] \big|}{\big|\textbf{det}[\nabla \tilde{\xi}^{t}_p (\alpha;\mathcal{D}_p) ] \big|} = \frac{L \big( \tilde{\xi}^{t}_p (\alpha; \mathcal{D}_p); \bar{\epsilon}, \bar{\Delta}_p^t \big) }{L \big( \tilde{\xi}^{t}_p (\alpha; \mathcal{D}'_p); \bar{\epsilon}, \bar{\Delta}_p^t \big) } \nonumber \\
    & = \textbf{exp} \Big( (\bar{\epsilon}/\bar{\Delta}_p^t)(\| \tilde{\xi}^{t}_p (\alpha;\mathcal{D}'_p) \|_1 - \| \tilde{\xi}^{t}_p (\alpha;\mathcal{D}_p) \|_1) \Big) \leq \textbf{exp} \Big( (\bar{\epsilon}/\bar{\Delta}_p^t)(\| \tilde{\xi}^{t}_p (\alpha;\mathcal{D}'_p) - \tilde{\xi}^{t}_p (\alpha;\mathcal{D}_p) \|_1) \Big) \nonumber \\
    & = \textbf{exp} \Big( (\bar{\epsilon}/\bar{\Delta}_p^t)(\| f'_p(z^t_p;\mathcal{D}_p) - f'_p(z^t_p;\mathcal{D}'_p)\|_1) \Big) \leq \textbf{exp} (\bar{\epsilon}), \nonumber
  \end{align}
  \end{subequations}
  where \textbf{exp} represents the exponential function,
  the first equality is from \eqref{inverse},
  the second equality holds because of \eqref{jacobian},
  the first inequality holds because of the reverse triangle inequality, the last equality holds because of \eqref{correspondence}, and the last inequality holds because of \eqref{Delta}.
  Similarly, one can derive a lower bound in \eqref{DP_PDF}.
  Integrating $\alpha$ in \eqref{DP_PDF} over $\mathcal{S}$ yields \eqref{DP_1}.
  This completes the proof.

\subsection{Privacy Analsis for Algorithm \ref{algo:DP-IADMM-Trust} }  \label{apx:privacy-analysis-trust}
Using Lemma \ref{lemma:theorem1}, we will show that the \textit{constrained} problem \eqref{DPADMM-2-Trust} provides $\bar{\epsilon}$-DP.
For the rest of this section, we fix $t \in \mathbb{N}$ and $p \in [P]$.
For ease of exposition, we express the feasible region of \eqref{DPADMM-2-Trust} using $M$ inequalities, namely,
\begin{align*}
    \{ \mathcal{W} \cap \widehat{\mathcal{W}}^t_p \} \Leftrightarrow \{ z_p \in \mathbb{R}^{J \times K} : h_m^t (z_p) \leq 0, \ \forall m \in [M] \}, 
\end{align*}
where $h_m^t$ is convex and twice continuously differentiable. 

Similar to the unconstrained problem \eqref{ADMM-2-Prox-log}, we construct the following \textit{unconstrained} problem:
\begin{subequations}
  \label{ADMM-2-log}
\begin{align}
  & z^{t+1}_p(\ell, \mathcal{D}_p) = \textstyle \argmin_{z_p \in \mathbb{R}^{J \times K} } \ G^t(z_p; \mathcal{D}_p, \tilde{\xi}^t_p) + g^t(z_p; \ell), \\
  & g^t(z_p; \ell) :=  \textstyle \sum_{m=1}^M \ln ( 1+ e^{\ell h_m^t(z_p)}), \label{logfn}
\end{align}
\end{subequations}
where $\ell > 0$.

We will show that \eqref{ADMM-2-log} satisfies the pointwise convergence condition and provides $\bar{\epsilon}$-DP as in Propositions \ref{prop:pointwise_convergence_trust} and \ref{prop:dpinapproximation_trust}, respectively.
\begin{proposition}  \label{prop:pointwise_convergence_trust}
  For fixed $t$ and $p$, we have $\lim_{\ell \rightarrow \infty} z_p^{t+1}(\ell,\mathcal{D}_p) = z_p^{t+1}(\mathcal{D}_p)$, where $z_p^{t+1}(\mathcal{D}_p)$ and $z_p^{t+1}(\ell,\mathcal{D}_p)$ are from \eqref{DPADMM-2-Trust} and \eqref{ADMM-2-log}, respectively.
\end{proposition}
\begin{proof}
Fix $t$ and $p$.
Suppose that $\hat{z}^{t+1}_p$ (resp., $\hat{z}^{t+1}_{p\ell}$) is the unique optimal solution of an optimization problem in \eqref{DPADMM-2-Trust} (resp., \eqref{ADMM-2-log}). One can follow the proof in Appendix \ref{apx-prop:pointwise_convergence} by setting $G^t_p(z_p) := G^t(z_p;\mathcal{D}_p,\tilde{\xi}^t_p)$ and $g_{p\ell}(z_p) := g^t(z_p;\ell)$.
\end{proof}

\begin{proposition}   \label{prop:dpinapproximation_trust}
  For fixed $t$, $p$, and $\ell$, \eqref{ADMM-2-log} provides $\bar{\epsilon}$-DP, namely, satisfying
  \begin{align*}    
        e^{-\bar{\epsilon}} \ \mathbb{P} \big( z^{t+1}_p(\ell;  \mathcal{D}'_p) \in \mathcal{S} \big) \leq \mathbb{P} \big( z^{t+1}_p (\ell; \mathcal{D}_p )\in \mathcal{S}  \big)  \leq e^{\bar{\epsilon}} \ \mathbb{P} \big( z^{t+1}_p(\ell;  \mathcal{D}'_p) \in \mathcal{S} \big)
  \end{align*}
  for all $\mathcal{S} \subset \mathbb{R}^{J \times K}$ and all $\mathcal{D}'_p \in \widehat{\mathcal{D}}_p$, where $\widehat{\mathcal{D}}_p$ is from \eqref{DataCollect}.
\end{proposition}
\begin{proof}
One can follow the proof in Appendix \ref{apx-prop:dpinapproximation} by setting $\eta^t = \infty$.  
\end{proof}
Based on Propositions \ref{prop:pointwise_convergence_trust} and \ref{prop:dpinapproximation_trust}, Lemma \ref{lemma:theorem1} can be used for proving the following theorem.
\begin{theorem} \label{thm:privacy_trust} 
  For fixed $t$ and $p$, \eqref{DPADMM-2-Trust} provides $\bar{\epsilon}$-DP, namely, satisfying
  \begin{align*}
    e^{-\bar{\epsilon}} \ \mathbb{P}( z^{t+1}_p(\mathcal{D}'_p) \in \mathcal{S} ) \leq \mathbb{P}( z^{t+1}_p(\mathcal{D}_p) \in \mathcal{S}) \leq e^{\bar{\epsilon}} \ \mathbb{P}( z^{t+1}_p( \mathcal{D}'_p) \in \mathcal{S} ), 
  \end{align*}
  for all $\mathcal{S} \subset \mathbb{R}^{J \times K}$ and all $\mathcal{D}'_p \in \widehat{\mathcal{D}}_p$, where $\widehat{\mathcal{D}}_p$ is from \eqref{DataCollect}.
\end{theorem}

\subsection{Existence of $U_1$, $U_2$, and $U_3$ in \eqref{def_upper}} \label{apx:existence_of_UBs}

Fix $p \in [P]$.\\
(Existence of $U_2$) 
$U_2$ is well-defined because the objective function $\|u-v\|$ is continuous and the feasible region $\mathcal{W}$ is compact. \\
(Existence of $U_1$)
The necessary and sufficient condition of Assumption \ref{assump:convergence} (iii) is that, for all $u \in \mathcal{W}$  and $v \in \partial f_p(u)$, $\| v \|_{\star} \leq L$, where $\|\cdot\|_{\star}$ is the dual norm.
As the dual norm of the Euclidean norm is the Euclidean norm, we have $\| f'_p(u) \| \leq L$.
As the objective function, which is a maximum of finite continuous functions, is continuous and $\mathcal{W}$ is compact, $U_1$ is well-defined. \\
(Existence of $U_3$)
From the norm inequality, we have
\begin{align*}
& \| f'_p(u;\mathcal{D}_p) - f'_p(u;\mathcal{D}'_p)\|_1 \leq \sqrt{JK} \| f'_p(u;\mathcal{D}_p) - f'_p(u;\mathcal{D}'_p)\|_2 \\
& \leq \sqrt{JK} \{ \| f'_p(u;\mathcal{D}_p) \|_2  + \| f'_p(u;\mathcal{D}'_p)\|_2 \} \leq 2L \sqrt{JK}, \ \forall u \in \mathcal{W},
\end{align*}
where the last inequality holds by Assumption \ref{assump:convergence} (iii). Therefore, $U_3$ is well-defined.

\subsection{Proof of Proposition \ref{prop:case_2_basic_inequality}} \label{apx-prop:case_2_basic_inequality} 
Before getting into details, we note that
\begin{align}
(a-b)^{\top} P (c-d) = \frac{1}{2} \{ \| a-d \|^2_P  - \| a-c \|^2_P + \| c-b \|^2_P - \| d - b \|^2_P \}  \label{rule}
\end{align}
for any symmetric matrix $P$. We fix $t$ and $p$ for the rest of this proof.

First, the optimality condition of \eqref{DPADMM-2-Prox} is given by 
\begin{align*}
\langle f'_p(z^t_p) - \rho^t(w^{t+1}-z^{t+1}_p + \frac{1}{\rho^t}(\lambda^t_p - \tilde{\xi}^{t}_p)) + \frac{1}{\eta^t}(z^{t+1}_p - z_p^t), z^{t+1}_p - z_p \rangle \leq 0, \ \forall z_p \in \mathcal{W}.
\end{align*}
By utilzing \eqref{ADMM-3} and \eqref{rule}, for all $z_p \in \mathcal{W}$, we have
\begin{align}
\langle f'_p(z^t_p) - \lambda^{t+1}_p + \tilde{\xi}^{t}_p, z^{t+1}_p - z_p \rangle \leq \frac{1}{2\eta^t} \{ \| z_p - z_p^t\|^2 - \| z_p - z_p^{t+1}\|^2 - \| z_p^{t+1} - z_p^t\|^2 \}.
\end{align}

Second, it follows from the convexity of $f_p$ that, for all $z_p \in \mathcal{W}$, 
\begin{align}
& f_p(z^t_p) - f_p(z_p)  \leq \langle f'_p(z^t_p), z^t_p - z_p \rangle  \nonumber \\
\Leftrightarrow \ & f_p(z^t_p) - f_p(z_p) - \langle  \lambda^{t+1}_p, z^{t+1}_p-z_p \rangle \leq \langle f'_p(z^t_p), z^t_p-z^{t+1}_p \rangle + \langle f'_p(z^t_p) - \lambda^{t+1}_p, z^{t+1}_p - z_p \rangle \nonumber \\
& = \langle  f'_p(z^t_p) + \tilde{\xi}^t_p, z^t_p-z^{t+1}_p \rangle + \langle f'_p(z^t_p) - \lambda^{t+1}_p + \tilde{\xi}^{t}_p, z^{t+1}_p-z_p \rangle + \langle  \tilde{\xi}^{t}_p, z_p - z^t_p \rangle \nonumber \\
& \leq \frac{\eta^t}{2} \| f'(z^t_p) + \tilde{\xi}^{t}_p\|^2  + \frac{1}{2\eta^t} \Big\{ \|z_p-z^t_p\|^2- \|z_p-z^{t+1}_p\|^2   \Big\}  + \langle  \tilde{\xi}^{t}_p, z_p - z^t_p \rangle,  \nonumber 
\end{align}
where the last inequality holds because of Young's inequality (i.e., $ab \leq \frac{a^2}{2\eta} + \frac{\eta b^2}{2}$) and \eqref{rule}.
This completes the proof.

\subsection{Proof of Theorem \ref{thm:convergence_rate_prox}} \label{apx-thm:convergence_rate_prox}

For ease of exposition, we introduce the following notations:
\begin{align}
  & z := [z_1^{\top}, \ldots, z_P^{\top}]^{\top},  \ \ \lambda := [\lambda_1^{\top}, \ldots, \lambda_P^{\top}]^{\top}, \ \ \tilde{\lambda} := [\tilde{\lambda}_1^{\top}, \ldots, \tilde{\lambda}_P^{\top}]^{\top}, \label{Notation}   \\
  & \nabla f(z):=[\nabla f_1(z_1)^{\top}, \ldots, \nabla f_P(z_P)^{\top} ]^{\top}, \ \ \ F(z) := \textstyle \sum_{p=1}^P f_p(z_p), \nonumber\\
  & x := \begin{bmatrix}
    w \\ z \\ \lambda
    \end{bmatrix}, \
    x^t := \begin{bmatrix}
    w^t \\ z^t \\ \lambda^t
    \end{bmatrix}, \
    \tilde{x}^t := \begin{bmatrix}
    w^{t+1} \\ z^{t} \\ \tilde{\lambda}^t
    \end{bmatrix}, \
    A := \begin{bmatrix}
      \mathbb{I}_J \\ \vdots \\ \mathbb{I}_J
    \end{bmatrix}_{PJ \times J}, \
    G := \begin{bmatrix}
        0 & 0 & A^{\top}     \\
        0 & 0 & -\mathbb{I}_{PJ}     \\
        -A & \mathbb{I}_{PJ} & 0     \\
    \end{bmatrix}, \nonumber  \\
    & x^{(T)} := \textstyle \frac{1}{T} \sum_{t=1}^T \tilde{x}^{t}, \ \
w^{(T)} := \textstyle \frac{1}{T} \sum_{t=1}^T w^{t+1}, \ \
z^{(T)} := \textstyle \frac{1}{T} \sum_{t=1}^T z^{t}, \ \
\lambda^{(T)} := \textstyle \frac{1}{T} \sum_{t=1}^T \tilde{\lambda}^{t}. \nonumber
\end{align}

For fixed $t$, we add \eqref{optimality_condition} and \eqref{case_2_basic_inequality} to obtain the inequalities $\text{LHS}^t(w,z) \leq \text{RHS}^t(z)$ for all $w$ and $z_p \in \mathcal{W}$, where
\begin{subequations}
\begin{align}
  & \text{LHS}^t(w,z) := \textstyle\sum_{p=1}^P \big\{ f_p(z^t_p) - f_p(z_p) - \langle \lambda^{t+1}_p, z^{t+1}_p - z_p \rangle + \langle  \tilde{\lambda}^t_p, w^{t+1}-w \rangle \big\}, \label{case_1_LHS_0}  \\
  & \text{RHS}^t(z) := \textstyle\sum_{p=1}^P  \big\{ \frac{\eta^t \| f'(z^t_p) + \tilde{\xi}^{t}_p \|^2}{2}   + \frac{1}{2\eta^t} \big( \|z_p-z^t_p\|^2  - \|z_p-z^{t+1}_p\|^2 \big) + \langle \tilde{\xi}^{t}_p, z_p - z^t_p \rangle \big\}. \label{case_1_RHS_0}
\end{align}
\end{subequations}
In the following Lemma, we first simplify the left-hand side \eqref{case_1_LHS_0}.
\begin{lemma} \label{lemma:simplify_LHS}
Based on the notations in \eqref{Notation}, for any $\lambda$, we have
\begin{align}
  & \text{LHS}^t(w,z)  = F(z^t) - F(z) + \langle \tilde{x}^{t} - x, G x  \rangle - \langle \lambda, z^{t+1} - z^t \rangle  \nonumber \\
  & \hspace{12mm} + \big(\rho^t/2\big) \big( \|z - z^{t+1}\|^2  - \|z-z^t\|^2 \big) + \big( 1/(2\rho^t) \big) \big( \| \lambda - \lambda^{t+1} \|^2  - \| \lambda - \lambda^t \|^2 \big). \label{case_1_LHS}
\end{align}
\end{lemma}
\begin{proof}
  Based on the notations in \eqref{Notation}, we have
  \begin{subequations}
  \begin{align}
    \lambda^{t+1} = \lambda^t + \rho^t (A w^{t+1} - z^{t+1}), \ \  \tilde{\lambda}^{t} = \lambda^t + \rho^t (A w^{t+1} - z^t), \ \ \textstyle\sum_{p=1}^P \tilde{\lambda}^t_p = A^{\top} \tilde{\lambda}^t.
  \end{align}
  We rewrite \eqref{case_1_LHS_0} as
  \begin{align}
  & F(z^t) - F(z) - \langle \lambda^{t+1}, z^{t+1} - z \rangle + \langle A^{\top} \tilde{\lambda}^t, w^{t+1} - w \rangle  \nonumber \\
= \ &  F(z^t) - F(z) + \Biggr\langle
  \begin{bmatrix}
  w^{t+1} - w \\
  z^{t+1} - z \\
  \tilde{\lambda}^t - \lambda
  \end{bmatrix}
  , \
  \begin{bmatrix}
  A^{\top} \tilde{\lambda}^t             \\
  -\tilde{\lambda}^t \\
  - Aw^{t+1} + z^{t+1}
  \end{bmatrix}
  -
  \begin{bmatrix}
  0 \\
  \rho^t (z^t - z^{t+1})        \\
  (\lambda^t- \lambda^{t+1} ) / \rho^t
  \end{bmatrix}
  \Biggr \rangle. \label{diminish_LHS}
  \end{align}

The third term in \eqref{diminish_LHS} can be written as
\begin{align}
    &\Biggr \langle
    \begin{bmatrix}
        w^{t+1} - w \\
        z^{t+1} - z  \\
        \tilde{\lambda}^t - \lambda
    \end{bmatrix}
    ,
    \begin{bmatrix}
    A^{\top} \tilde{\lambda}^t \\
    - \tilde{\lambda}^t \\
    -Aw^{t+1} + z^{t+1}
    \end{bmatrix}
    \Biggr \rangle
    =
    \Biggr \langle
    \begin{bmatrix}
        w^{t+1} - w \\
        z^{t} - z  \\
        \tilde{\lambda}^t - \lambda
    \end{bmatrix}
    ,
    \begin{bmatrix}
    A^{\top} \tilde{\lambda}^t \\
    - \tilde{\lambda}^t \\
    -Aw^{t+1} + z^{t}
    \end{bmatrix}
    \Biggr \rangle
    \nonumber \\
    & + \langle z^t - z^{t+1}, \tilde{\lambda}^t \rangle + \langle \tilde{\lambda}^t - \lambda, z^{t+1} - z^t \rangle  = \langle \tilde{x}^{t} - x, G\tilde{x}^{t} \rangle  - \langle \lambda, z^{t+1} - z^t \rangle  \nonumber \\
    = & \langle \tilde{x}^{t} - x, G(\tilde{x}^{t} - x) \rangle + \langle \tilde{x}^{t} - x , G x \rangle  - \langle \lambda, z^{t+1} - z^t \rangle \nonumber \\
    =& \langle \tilde{x}^{t} - x, G x  \rangle - \langle \lambda, z^{t+1} - z^t \rangle, \label{diminish_Third_term}
\end{align}
where the last equality holds because $G$ is a skew-symmetric matrix and thus $\langle \tilde{x}^{t} - x, G(\tilde{x}^{t} - x ) \rangle=0$.

The last term in \eqref{diminish_LHS} can be written as
    \begin{align}
        & \Biggr\langle
        \begin{bmatrix}
        w - w^{t+1} \\
        z - z^{t+1}   \\
        \lambda - \tilde{\lambda}^t
        \end{bmatrix}
        , \
        \begin{bmatrix}
        0 \\
        \rho^t (z^t-z^{t+1})  \\
        (\lambda^t- \lambda^{t+1} ) / \rho^t
        \end{bmatrix}
        \Biggr \rangle   \nonumber      \\
       = &  \big\{ \| z - z^{t+1}\|_{\rho^t \mathbb{I}}^2 - \|z-z^t\|_{\rho^t \mathbb{I}}^2 + \|z^{t+1} - z^t\|_{\rho^t \mathbb{I}}^2 \big\} / 2
         \nonumber \\
        & + \big\{ \| \lambda - \lambda^{t+1} \|^2_{(1/\rho^t)\mathbb{I}} - \| \lambda - \lambda^t \|^2_{(1/\rho^t)\mathbb{I}} + \|\tilde{\lambda}^{t} - \lambda^t \|^2_{(1/\rho^t)\mathbb{I}} - \| \tilde{\lambda}^t - \lambda^{t+1} \|^2_{(1/\rho^t)\mathbb{I}} \big\} / 2 \nonumber \\
        \geq & \big(\rho^t/2\big) \big( \|z - z^{t+1}\|^2  - \|z-z^t\|^2 \big) + \big( 1/(2\rho^t) \big) \big( \| \lambda - \lambda^{t+1} \|^2  - \| \lambda - \lambda^t \|^2 \big), \label{diminish_Fourth_term}
    \end{align}
    \end{subequations}
    where the equality holds because $(a-b)^{\top}P(c-d)= \big\{\|a-d\|^2_P - \|a-c\|^2_P + \|b-c\|^2_P - \|b-d\|^2_P \big\}/2$ for any symmetric matrix $P$, and the ineqaulity holds because
$\|\tilde{\lambda}^{t} - \lambda^t \|^2_{(1/\rho^t)\mathbb{I}} \geq 0$
and
$\|z^{t+1} - z^t\|_{\rho^t \mathbb{I}}^2 = \| \tilde{\lambda}^t - \lambda^{t+1} \|^2_{(1/\rho^t)\mathbb{I}}$.
\end{proof}
Based on Lemma \ref{lemma:simplify_LHS} and notations in \eqref{Notation}, we derive a lower bound on $\frac{1}{T} \sum_{t=1}^T \text{LHS}^t(w,z)$ in the following lemma.
\begin{lemma} \label{lemma:lower_bound_LHS}
We define $\text{LHS}(w,z) :=\frac{1}{T} \sum_{t=1}^T \text{LHS}^t(w,z)$ and $\text{RHS}(z) := \frac{1}{T} \sum_{t=1}^T \text{RHS}^t(z)$.
For all $w$ and $z_p \in \mathcal{W}$, we have
\begin{align}
  & \text{LHS}(w,z) \geq F(z^{(T)}) - F(z) + \big\langle x^{(T)} - x, Gx \big\rangle \nonumber  \\
  & \hspace{20mm} - \frac{1}{T}  \Big( \langle \lambda, z^{T+1} - z^1\rangle + \frac{U_2 \rho^{\text{max}}}{2} + \frac{1}{2\rho^1} \| \lambda - \lambda^1 \|^2 \Big). \label{lower_bound_LHS}
\end{align}
\end{lemma}
\begin{proof}
Based on Lemma \ref{lemma:simplify_LHS} and notations in \eqref{Notation}, we have
\begin{align}
& \text{LHS}(w,z) = \frac{1}{T} \Big[ \sum_{t=1}^T F(z^{t}) - T F(z) + \Big\langle \sum_{t=1}^T \tilde{x}^t - Tx , Gx \Big\rangle - \langle \lambda, z^{T+1} - z^1\rangle  \nonumber \\
& \hspace{20mm} + \sum_{t=1}^T  \Big\{ \frac{\rho^t}{2}  \big( \|z - z^{t+1}\|^2  - \|z-z^t\|^2 \big) + \frac{1}{2\rho^t} \big( \| \lambda - \lambda^{t+1} \|^2  - \| \lambda - \lambda^t \|^2 \big) \Big\} \Big]. \nonumber
\end{align}
To simplify further, we derive the following lower bounds.
\begin{subequations}
\begin{align}
  & \sum_{t=1}^T \frac{\rho^t}{2}  \big( \|z - z^{t+1}\|^2  - \|z-z^t\|^2 \big) = - \frac{\rho^1}{2} \|z-z^1\|^2 + \sum_{t=2}^T \Big( \frac{\rho^{t-1} - \rho^t}{2} \Big)\|z - z^t\|^2  \nonumber  \\
  & \ \  + \frac{\rho^T}{2}\|z - z^{T+1}\|^2 \geq - \frac{\rho^1}{2} U_2 + \sum_{t=2}^T \Big( \frac{\rho^{t-1} - \rho^t}{2} \Big) U_2 = \frac{ - U_2 \rho^{T}}{2} \geq \frac{ - U_2 \rho^{\text{max}}}{2}, \label{LB_1} \\
  & \sum_{t=1}^T \frac{1}{2\rho^t} \big( \| \lambda - \lambda^{t+1} \|^2  - \| \lambda - \lambda^t \|^2 \big) = - \frac{1}{2\rho^1} \|\lambda-\lambda^1\|^2 + \sum_{t=2}^T \Big( \frac{1}{2 \rho^{t-1}} - \frac{1}{2 \rho^{t}} \Big)\|\lambda - \lambda^t\|^2  \nonumber  \\
  & \ \  + \frac{1}{2\rho^T}\|\lambda - \lambda^{T+1}\|^2 \geq - \frac{1}{2\rho^1} \|\lambda-\lambda^1\|^2, \label{LB_2}
\end{align}
\end{subequations}
where the first inequalities in \eqref{LB_1} and \eqref{LB_2} hold because $\rho^t > 0$ is non-decreasing by Assumption \ref{assump:convergence}, and $U_2$ is from \eqref{def_upper}.

Based on \eqref{LB_1}, \eqref{LB_2}, and $F(z^{(T)}) \leq \frac{1}{T} \sum_{t=1}^T F(z^t)$, which is valid due to the convexity of $F$, one can derive \eqref{lower_bound_LHS}.
\end{proof}

Second, we have
\begin{align*}
\big\langle x^{(T)} - x, Gx \big\rangle & =
\langle w^{(T)} - w, A^{\top} \lambda \rangle - \langle z^{(T)}-z, \lambda \rangle - \langle \lambda^{(T)} - \lambda, Aw-z\rangle \\
& = \langle A w^{(T)} - z^{(T)} - Aw + z, \lambda \rangle - \langle \lambda^{(T)} - \lambda, Aw-z\rangle.
\end{align*}
Let $(w^*,z^*)$ be an optimal solution. As $Aw^* - z^*=0$, we have
\begin{align}
\big\langle x^{(T)} - x^*, Gx^* \big\rangle  =  \langle \lambda, Aw^{(T)} - z^{(T)} \rangle. \label{skewness}
\end{align}
Based on \eqref{skewness}, Lemma \ref{lemma:simplify_LHS}, and Lemma \ref{lemma:lower_bound_LHS}, we derive the following inequalities:
\begin{align}
& F(z^{(T)}) - F(z^*) + \langle \lambda, Aw^{(T)} - z^{(T)} \rangle
- \frac{1}{T}  \Big( \langle \lambda, z^{T+1} - z^1\rangle + \frac{U_2 \rho^{\text{max}}}{2} + \frac{1}{2\rho^1} \| \lambda - \lambda^1 \|^2 \Big)  \nonumber \\
& \leq \frac{1}{T} \sum_{t=1}^T   \big\{ \frac{\eta^t \sum_{p=1}^P \|  f'(z^t_p) + \tilde{\xi}^{t}_p \|^2}{2}   + \frac{1}{2\eta^t} \big( \|z^* - z^t \|^2  - \|z^* - z^{t+1} \|^2 \big) + \langle \tilde{\xi}^{t}, z^* - z^t \rangle \big\}  . \nonumber
\end{align}
Since the above inequality holds for any $\lambda$, we can take the maximum of both
sides over all $\lambda$ in a ball centered at zero with the radius $\gamma$ and obtain
\begin{align}
  & F(z^{(T)}) - F(z^*) + \gamma \| Aw^{(T)} - z^{(T)} \|  \leq \frac{1}{T} \Big( \sum_{t=1}^T \Big\{  \frac{\eta^t \sum_{p=1}^P \| f'(z^t_p) + \tilde{\xi}^{t}_p \|^2  }{2}  
   \nonumber \\
  &  + \frac{1}{2\eta^t} \big( \|z^*-z^t\|^2  - \|z^*-z^{t+1}\|^2 \big) + \langle \tilde{\xi}^{t}, z^* - z^t \rangle \Big\} + \gamma U_2 + \frac{U_2 \rho^{\text{max}}}{2} + \frac{(\gamma + \| \lambda^1 \|)^2}{2\rho^1}  \Big). \nonumber
\end{align}
By taking expectation, we have
\begin{align}
& \mathbb{E} \Big[ F(z^{(T)}) - F(z^*) + \gamma \| Aw^{(T)} - z^{(T)} \| \Big] \leq \frac{1}{T} \Big( \sum_{t=1}^T \Big\{  \frac{\eta^t \sum_{p=1}^P  \mathbb{E}\big[\| \nabla f(z^t_p) + \tilde{\xi}^{t}_p \|^2 \big] }{2}  
 \nonumber \\
&  + \frac{1}{2\eta^t} \big( \|z^*-z^t\|^2  - \|z^*-z^{t+1}\|^2 \big) + \mathbb{E}[ \langle \tilde{\xi}^{t}, z^* - z^t \rangle ] \Big\} + \gamma U_2 + \frac{U_2 \rho^{\text{max}}}{2} + \frac{(\gamma + \| \lambda^1 \|)^2}{2\rho^1}  \Big). \nonumber
\end{align}
Note that we have $\mathbb{E}[ \langle \tilde{\xi}^{t}, z^* - z^t \rangle ]=0$ and 
\begin{align*}
& \mathbb{E}\big[\| f'(z^t_p) + \tilde{\xi}^{t}_p \|^2 \big] = \| f'(z^t_p) \|^2 + \mathbb{E} \big[ \|\tilde{\xi}^{t}_p \|^2 \big] \leq U_1^2 + \sum_{j=1}^J \sum_{k=1}^K \mathbb{E}[ (\tilde{\xi}^t_{pjk})^2 ] \\
& = U_1^2 + \sum_{j=1}^J\sum_{k=1}^K 2 (\bar{\Delta}^t_p)^2/\bar{\epsilon}^2 \leq U_1^2 + 2 JK U_3^2/\bar{\epsilon}^2 ,  \ \forall p \in [P],
\end{align*}
where 
the first equality holds because $\mathbb{E}[2\langle \tilde{\xi}^t_p, f'(z^t_p) \rangle ] = 0$, 
the first inequality holds by the definition of $U_1$ from \eqref{def_upper}, and
the last inequality holds because $\bar{\Delta}^t_p \leq U_3$ for all $t$ and $p$, where $U_3$ is from \eqref{def_upper}.
Therefore, we have
\begin{align}
  & \mathbb{E} \Big[ F(z^{(T)}) - F(z^*) + \gamma \| Aw^{(T)} - z^{(T)} \| \Big] \leq \frac{1}{T} \Big( \sum_{t=1}^T \Big\{  \frac{\eta^t P(U_1^2 + 2JKU_3^2/\bar{\epsilon}^2) }{2}  
   \nonumber \\
  &  + \frac{1}{2\eta^t} \big( \|z^*-z^t\|^2  - \|z^*-z^{t+1}\|^2 \big)  \Big\} + \gamma U_2 + \frac{U_2 \rho^{\text{max}}}{2} + \frac{(\gamma + \| \lambda^1 \|)^2}{2\rho^1}  \Big). \nonumber
\end{align}
By setting $\eta^t = 1/\sqrt{t}$ and $R:= P(U_1^2 + 2JKU_3^2/\bar{\epsilon}^2)$, the first term of RHS can be written as
\begin{align}
R \sum_{t=1}^T \frac{1}{2\sqrt{t}} \leq R \sum_{t=1}^T (\sqrt{t} - \sqrt{t-1}) = R \sqrt{T},
\end{align}
and the second term of RHS can be written as
\begin{align}
  & \sum_{t=1}^T  \frac{1}{2\eta^t}\big(\|z^*-z^t\|^2-\|z^*-z^{t+1}\|^2 \big) \leq \frac{1}{2\eta^1} \|z^*-z^1\|^2 + \sum_{t=2}^T \Big(\frac{1}{2\eta^t}-\frac{1}{2\eta^{t-1}} \Big) \|z^*-z^t\|^2  \nonumber \\
  & - \frac{1}{2\eta^T}\|z^*-z^{T+1}\|^2 \leq \frac{1}{2\eta^1} U_2 + \sum_{t=2}^T \Big(\frac{1}{2\eta^t}-\frac{1}{2\eta^{t-1}} \Big) U_2 = \frac{U_2}{2\eta^T} = \frac{U_2 \sqrt{T}}{2}, \nonumber
\end{align}
where $U_2$ is from \eqref{def_upper}.

Therefore, we have
\begin{align}
  & \mathbb{E} \Big[ F(z^{(T)}) - F(z^*) + \gamma \| Aw^{(T)} - z^{(T)} \| \Big] \leq \frac{1}{T} \Big(  (PU_1^2 + 2PJKU_3^2/\bar{\epsilon}^2 + U_2/2)\sqrt{T} 
   \nonumber \\
  &   + \gamma U_2 + \frac{U_2 \rho^{\text{max}}}{2} + \frac{(\gamma + \| \lambda^1 \|)^2}{2\rho^1}  \Big). \nonumber
\end{align}

This completes the proof.

\subsection{Multi-Class Logistic Regression Model} \label{apx:model}
The multi-class logistic regression model considered in this paper is \eqref{ERM_0} with
\begin{align}
  & \ell(w; x_{pi},y_{pi}) := - \textstyle\sum_{k=1}^K  y_{pik} \ln \big(h_k(w;x_{pi})\big), \ \forall p \in [P], \forall i \in [I_p],  \nonumber \\
  & h_k(w;x_{pi}) :=  \frac{\exp(\textstyle\sum_{j=1}^J  x_{pij} w_{jk}) }{\sum_{k'=1}^K \exp(\textstyle\sum_{j=1}^J  x_{pij} w_{jk'}) }, \ \forall p \in [P], \forall i \in [I_p], \forall k \in [K], \nonumber  \\
  & r(w) := \textstyle\sum_{j=1}^J \sum_{k=1}^K w_{jk}^2, \nonumber  \nonumber  \\
  & f_p(w) = \textstyle - \frac{1}{I} \sum_{i=1}^{I_p} \sum_{k=1}^K \big\{ y_{pik} \ln (h_k(w;x_{pi})) \big\} + \frac{\beta}{P} \sum_{j=1}^J \sum_{k=1}^K  w_{jk}^2, \ \forall p \in [P], \nonumber  \\
  & \nabla_{w_{jk}} f_p(w) = \textstyle\frac{1}{I}  \sum_{i=1}^{I_p} x_{pij} (h_k(w;x_{pi})-y_{pik}) + \frac{2\beta}{P} w_{jk} , \ \forall p \in [P], \forall j \in [J], \forall k \in [K]. \label{gradient}
\end{align}

\subsection{Choice of the Penalty Parameter $\rho^t$} \label{apx-hyperparameter-rho}
We test various $\rho^t$ for our algorithms, and set it as $\hat{\rho}^t$ in \eqref{dynamic_rho} with (i) $c_1=2$, $c_2=5$, and $T_c=1e4$ for MNIST, and (ii) $c_1=0.005$, $c_2=0.05$, and $T_c=2e3$ for FEMNIST.

As these parameter settings may not lead \texttt{OutP} to its best performance, we test various $\rho^t$ for \texttt{OutP} using a set of static parameters, $\rho^t \in \{0.1, 1, 10\}$ for all $t \in [T]$, where $\rho^t=0.1$ is chosen in \cite{huang2019dp}, and dynamic parameters $\rho^t \in \{\hat{\rho}^t, \hat{\rho}^t/100 \}$, where $\hat{\rho}^t$ is from \eqref{dynamic_rho}.
In Figure \ref{fig:hyper_rho}, we report the testing errors of \texttt{OutP} using MNIST and FEMNIST under various $\rho^t$ and $\bar{\epsilon}$.
The results imply that the performance of \texttt{OutP} is not greatly affected by the choice of $\rho^t$, but $\bar{\epsilon}$.
Hence, for all algorithms, we use $\hat{\rho}^t$ in \eqref{dynamic_rho}.

\begin{figure}[!h]
  \centering
  \begin{subfigure}[b]{0.3\textwidth}
      \centering
      \includegraphics[width=\textwidth]{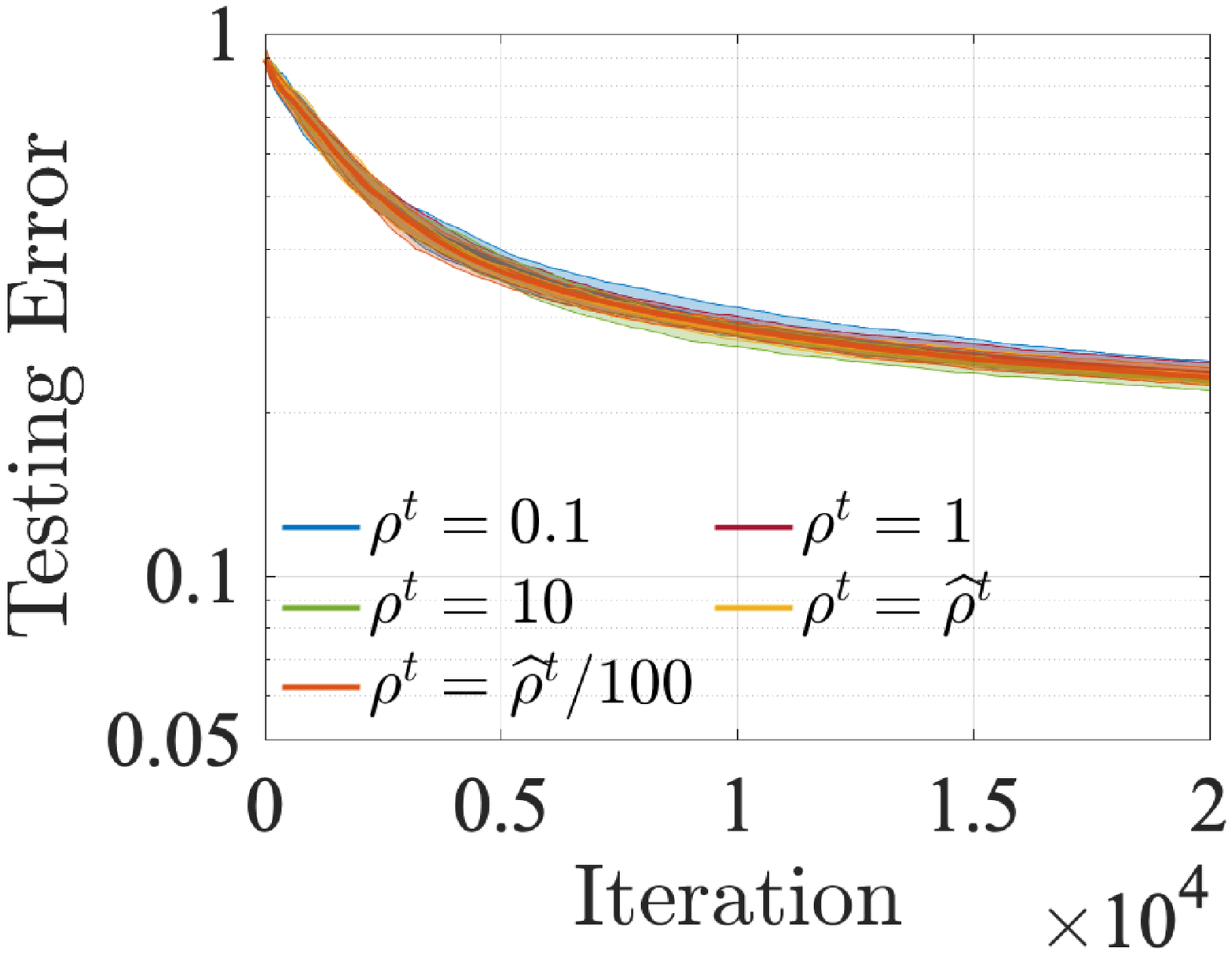}
      \includegraphics[width=\textwidth]{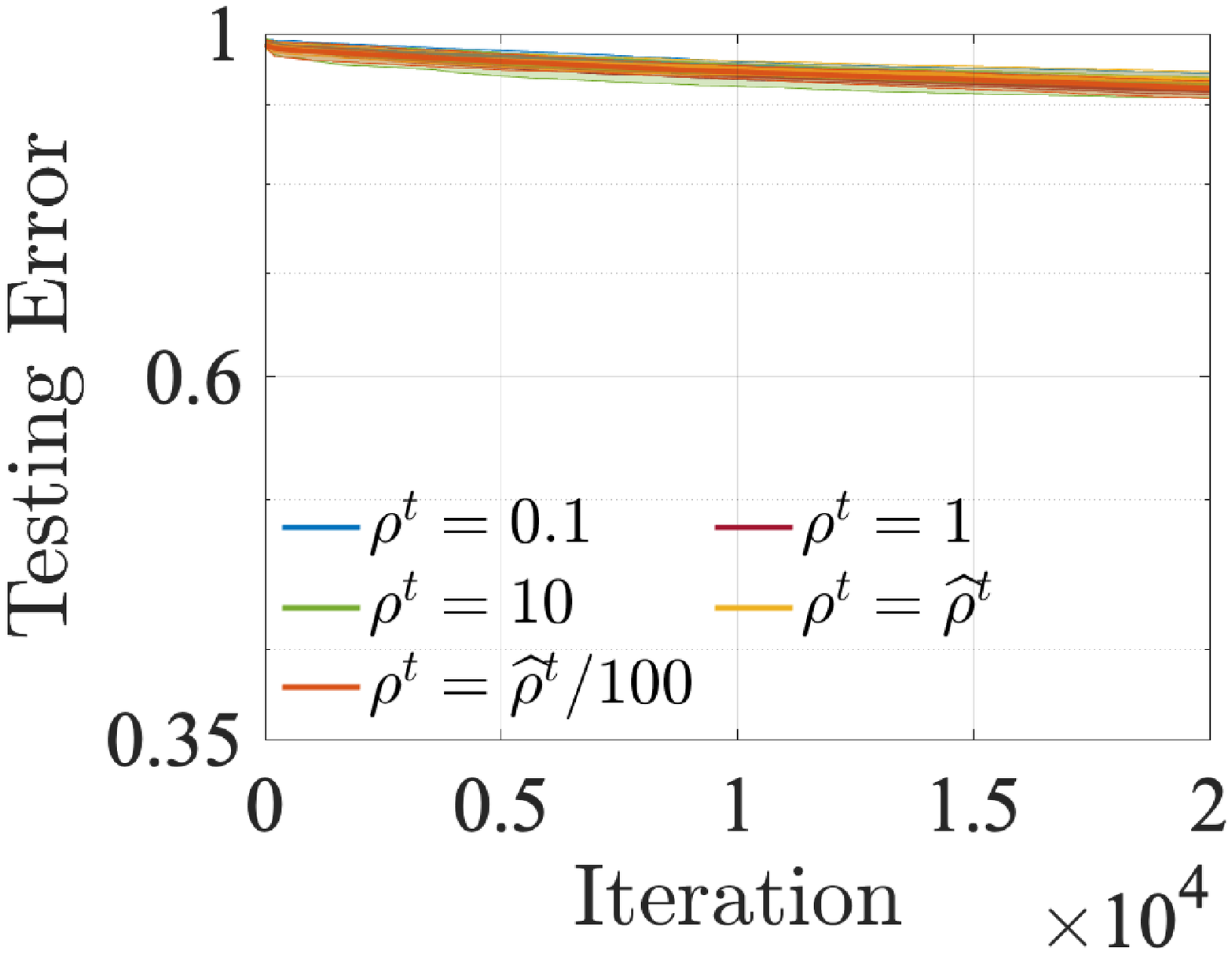}
      \caption{$\bar{\epsilon}=0.05$}
  \end{subfigure}
  \begin{subfigure}[b]{0.3\textwidth}
    \centering
    \includegraphics[width=\textwidth]{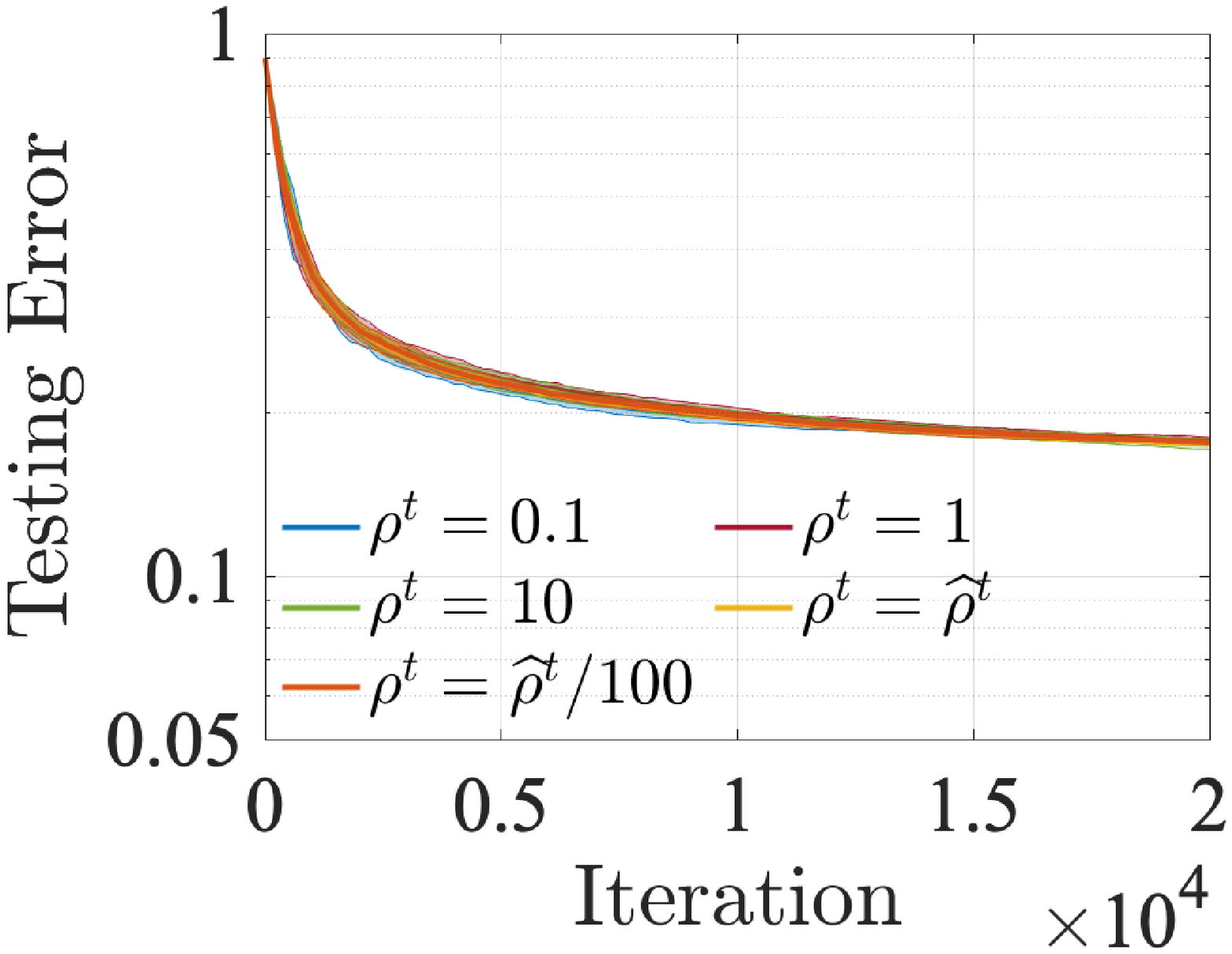}
    \includegraphics[width=\textwidth]{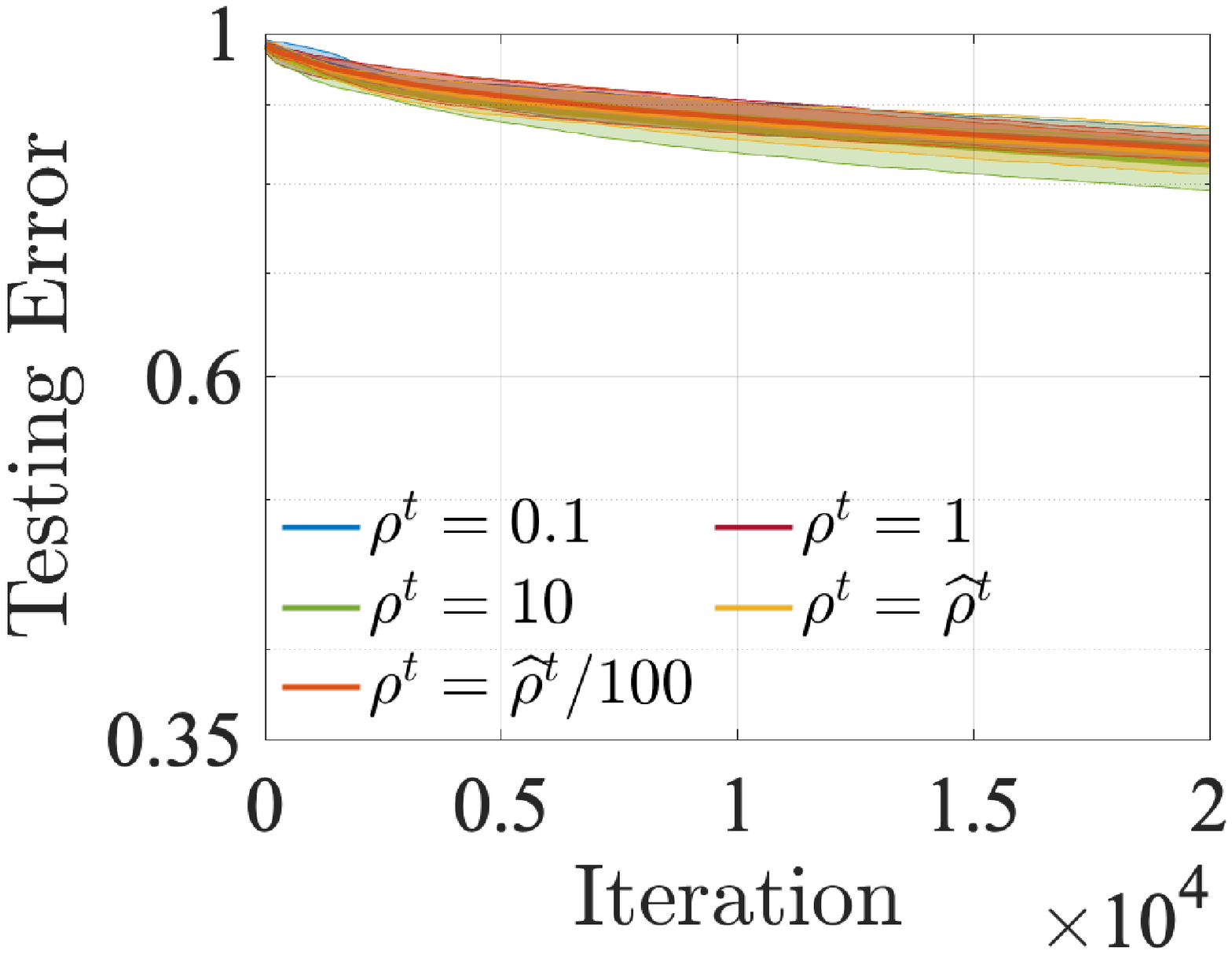}
    \caption{$\bar{\epsilon}=0.1$}
\end{subfigure}
  \begin{subfigure}[b]{0.3\textwidth}
      \centering
      \includegraphics[width=\textwidth]{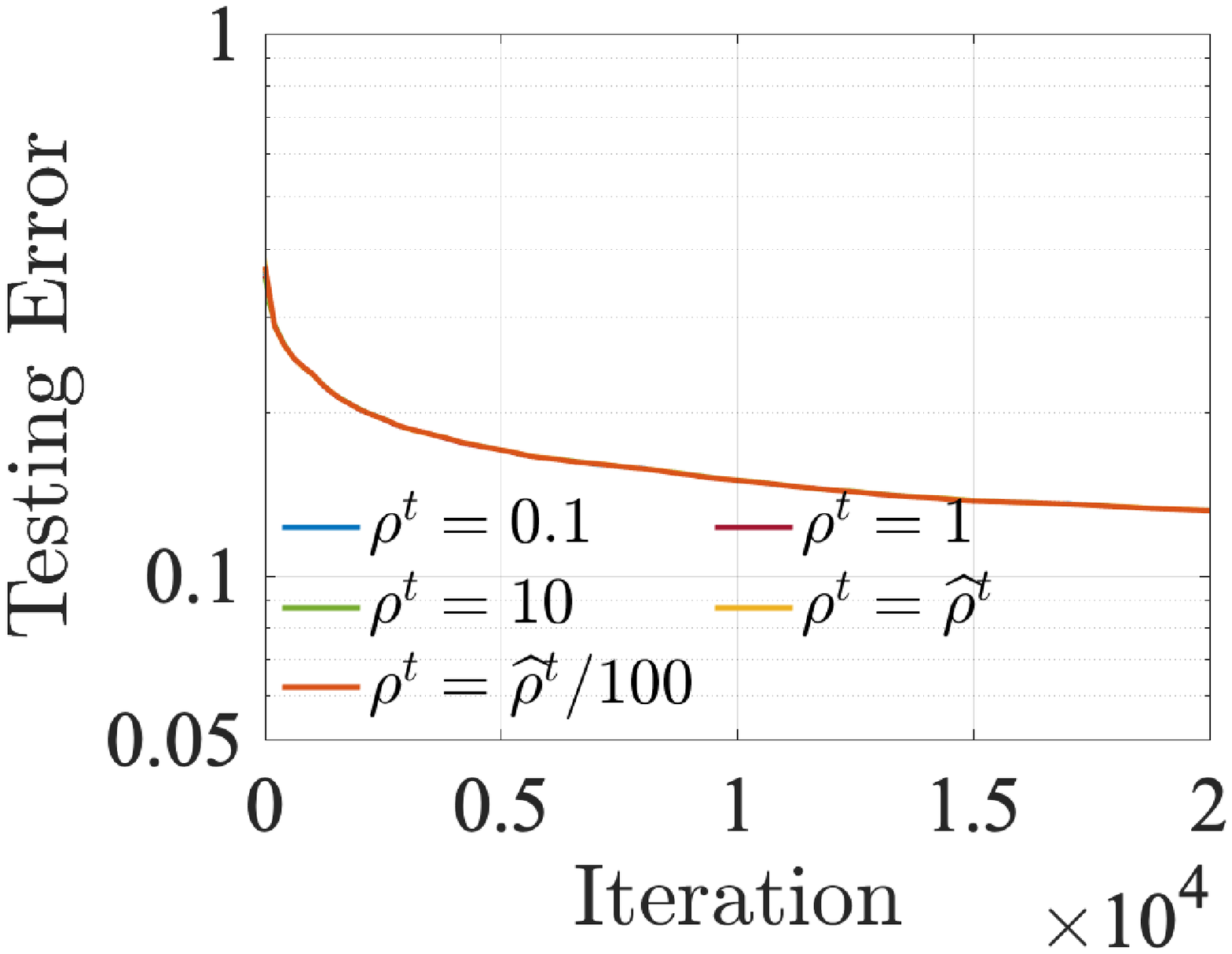}
      \includegraphics[width=\textwidth]{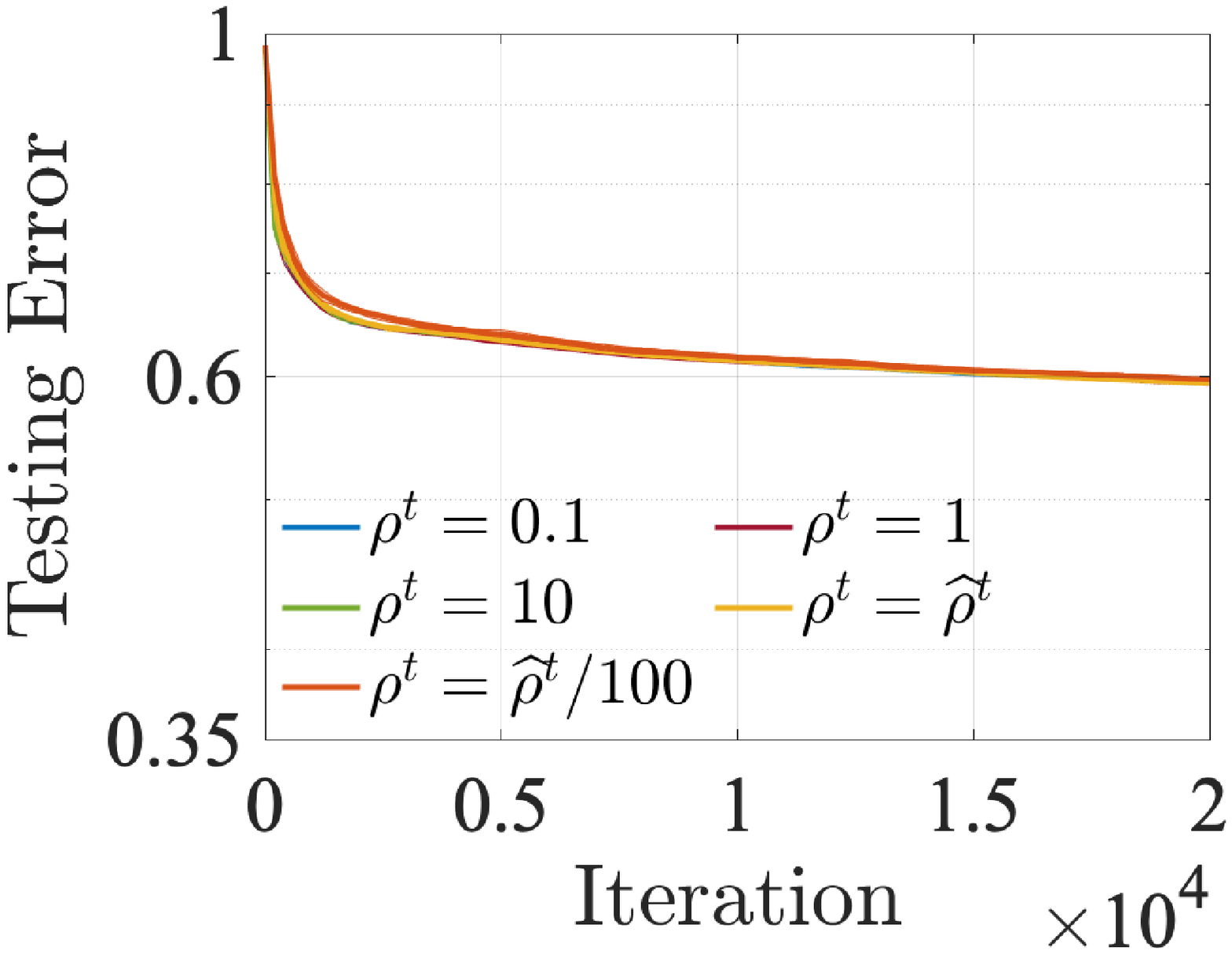}
      \caption{$\bar{\epsilon}=1$}
  \end{subfigure}
     \caption{Testing errors of \texttt{OutP} using MNIST (top) and FEMNIST (bottom).}
     \label{fig:hyper_rho}
\end{figure}

\subsection{Consensus Violation in MNIST.} \label{apx-residual}
To show that the solutions produced by \texttt{OutP}, \texttt{ObjP}, and \texttt{ObjT} are feasible, we report consensus violation (CV), namely, violation of \eqref{ERM_1-1}:
\begin{align}
  \textstyle\sum_{p=1}^P \sum_{j=1}^J \sum_{k=1}^K |w^t_{jk} - z^t_{pjk}|, \ \forall t \in [T], \nonumber
\end{align}
where $w^t_{jk}$ and $z^t_{pjk}$ are solutions at iteration $t$.
If CV is not zero at the termination, the solutions produced by the algorithms are infeasible with respect to \eqref{ERM_1-1}.

As shown in Figure \ref{fig:cv} (top), CV of all algorithms goes down to zero as $t$ increases, which implies that the solutions produced by all algorithms are feasible.
This can be explained by the nondecreasing $\hat{\rho}^t$ in \eqref{dynamic_rho} as it forces to find $z^{t+1}_p$ near $w^{t+1}$ as $t$ increases.

We also observe that CV of \texttt{OutP} quickly drop down compared with that of our algorithms, and this may be considered as a factor that prevents a greater learning performance of \texttt{OutP} by not improving the objective function value while focusing on reducing CV.
To show that this is not the case in our experiments, we construct \texttt{OutP+} which is \texttt{OutP} with $\rho^t \leftarrow 0.01 \times \hat{\rho}^t$, where $\hat{\rho}^t$ is from \eqref{dynamic_rho}.
As shown in Figure \ref{fig:cv} (top), CV of \texttt{OutP+} is larger than that of our algorithms, but our algorithms still outperform \texttt{OutP+} as shown in Figure \ref{fig:cv} (bottom).

\begin{figure}[!h]
  \centering
  \begin{subfigure}[b]{0.3\textwidth}
      \centering
      \includegraphics[width=\textwidth]{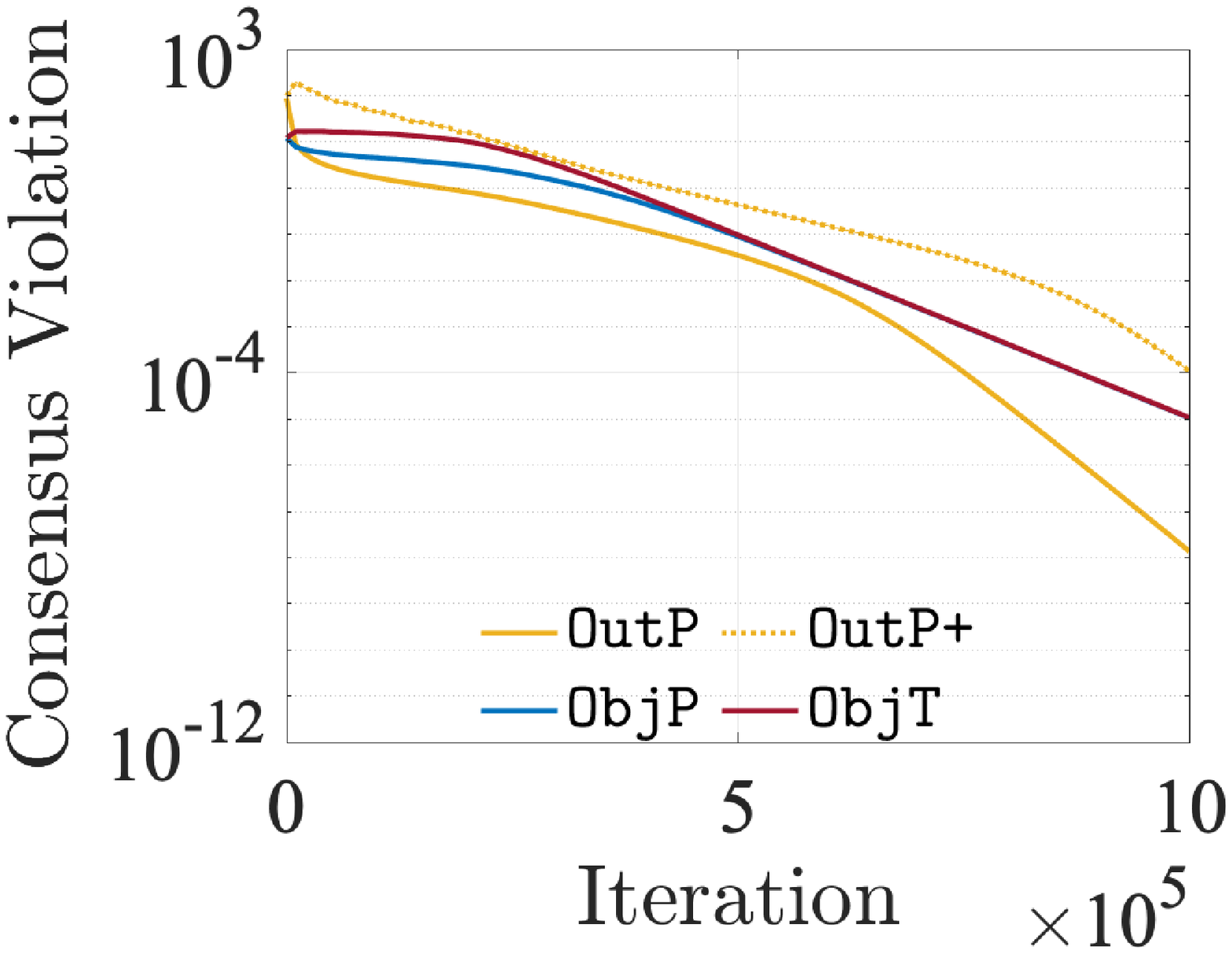}
      \includegraphics[width=\textwidth]{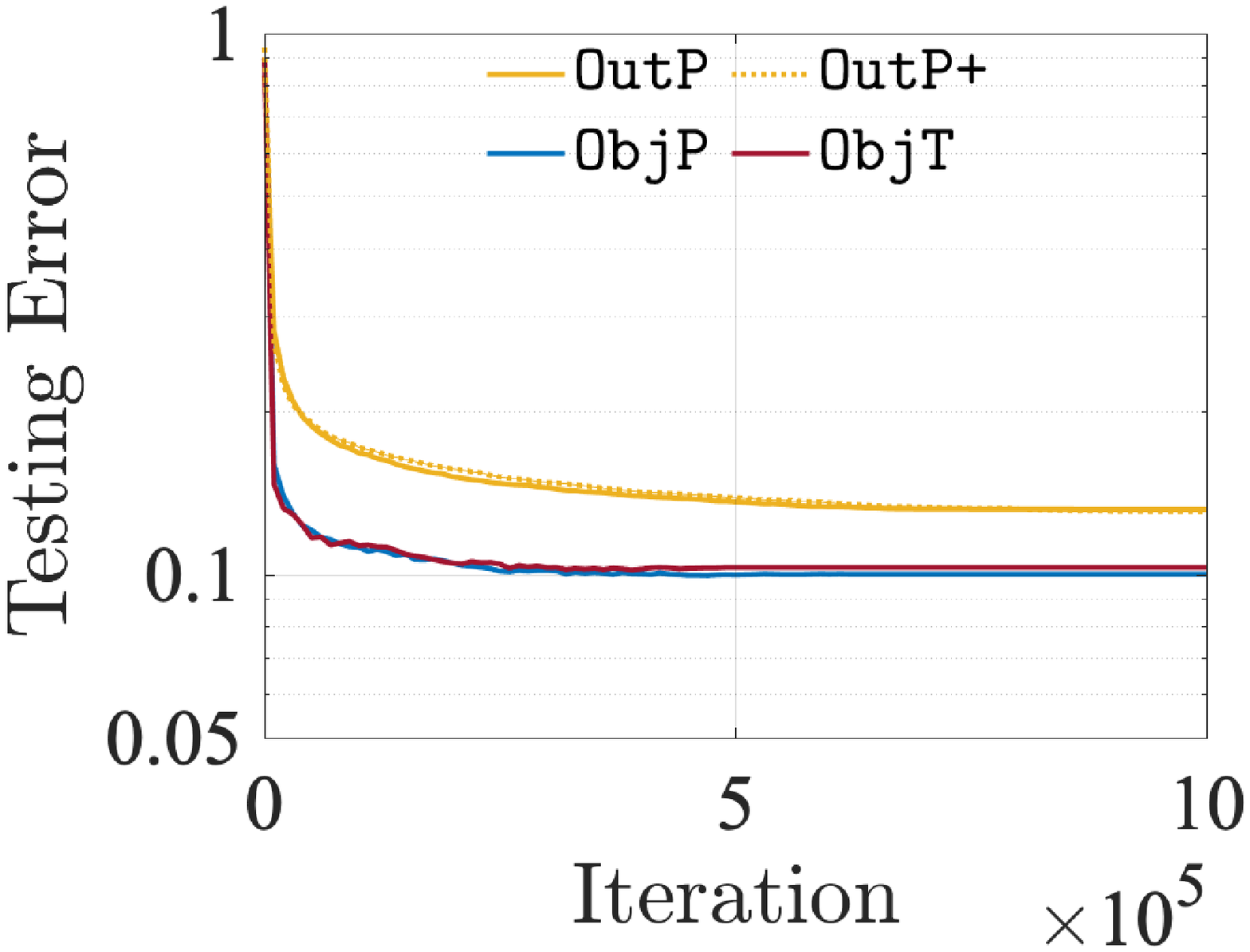}
      \caption{$\bar{\epsilon}=0.05$}
  \end{subfigure}
  \begin{subfigure}[b]{0.3\textwidth}
      \centering
      \includegraphics[width=\textwidth]{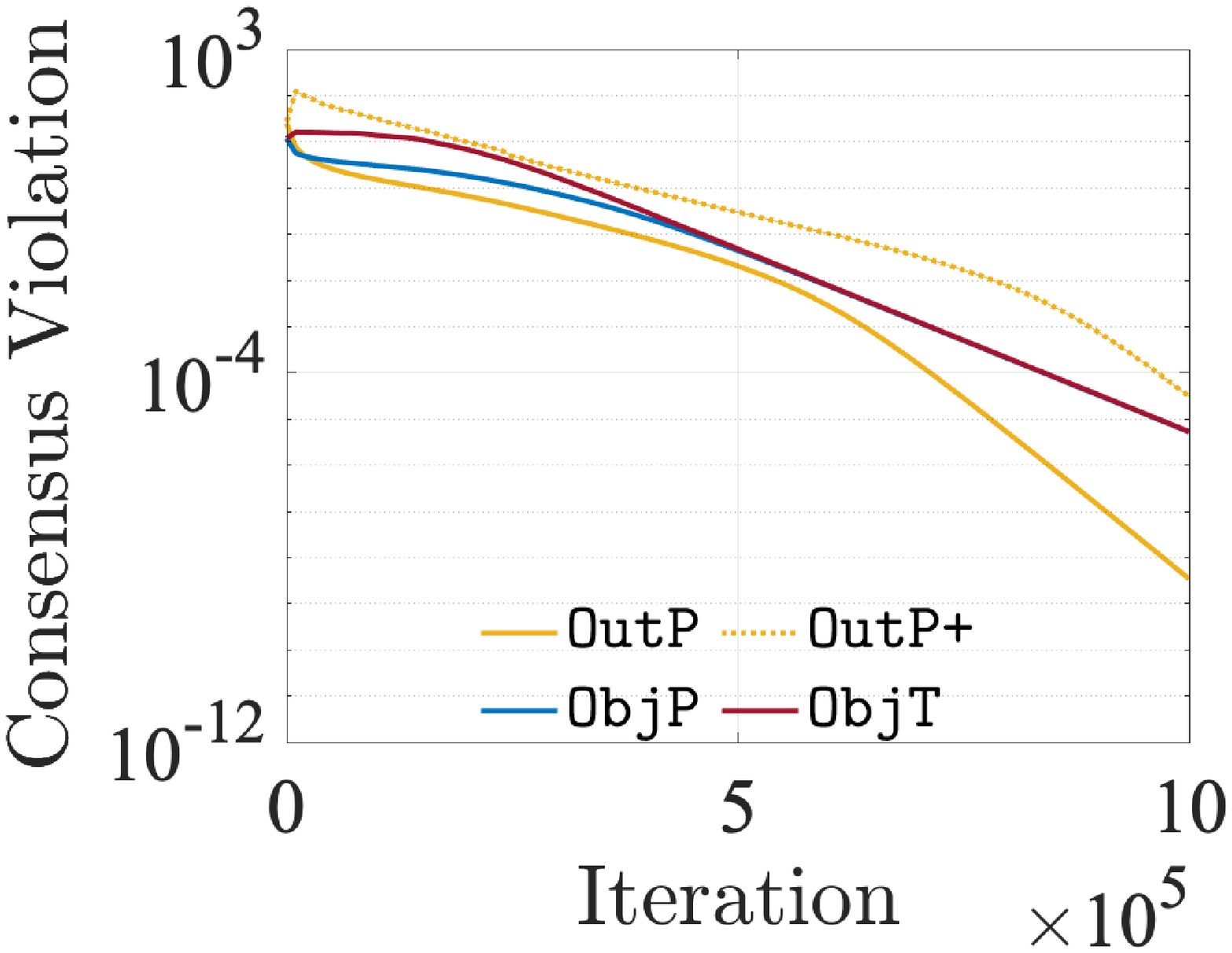}
      \includegraphics[width=\textwidth]{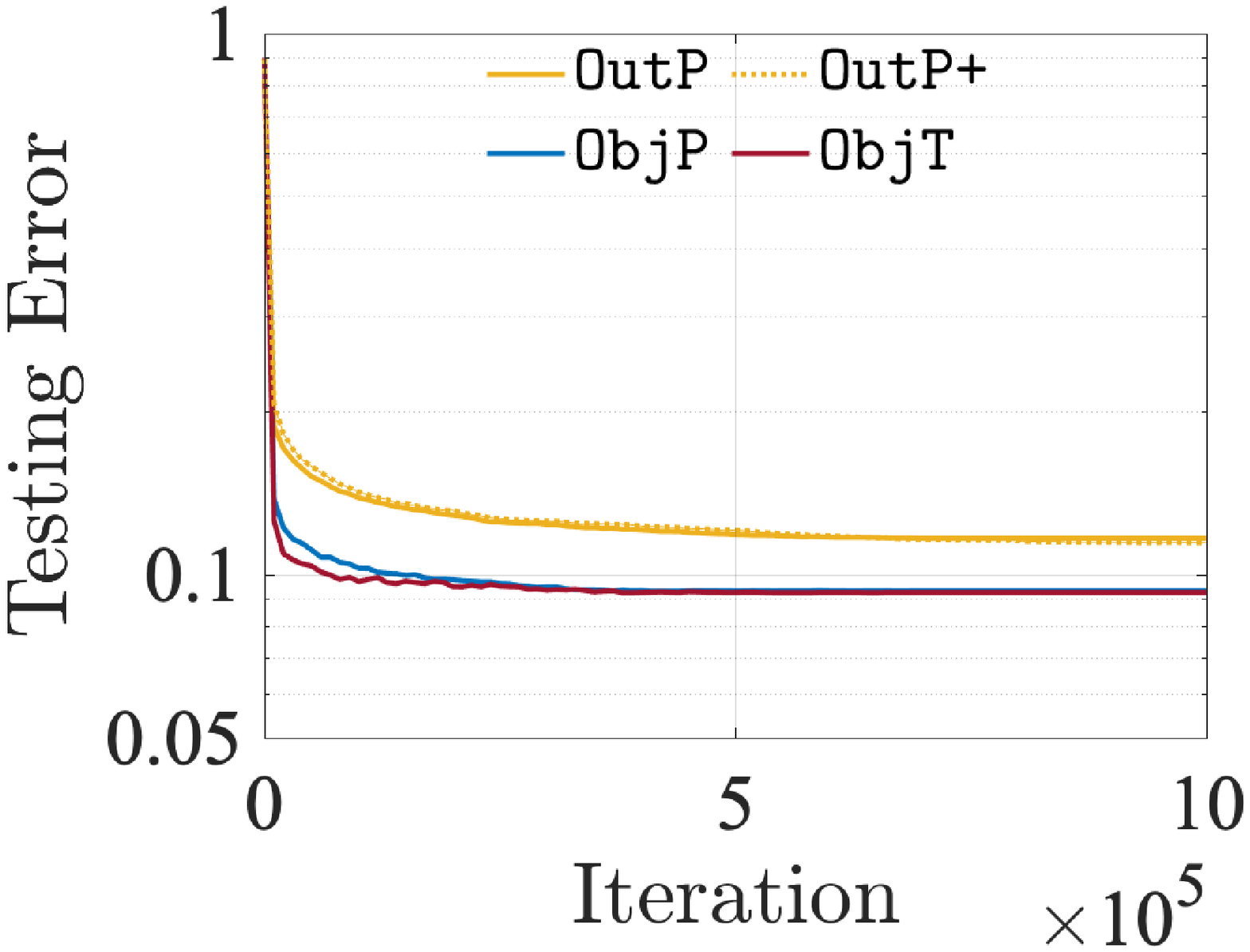}
      \caption{$\bar{\epsilon}=0.1$}
  \end{subfigure}
  \begin{subfigure}[b]{0.3\textwidth}
      \centering
      \includegraphics[width=\textwidth]{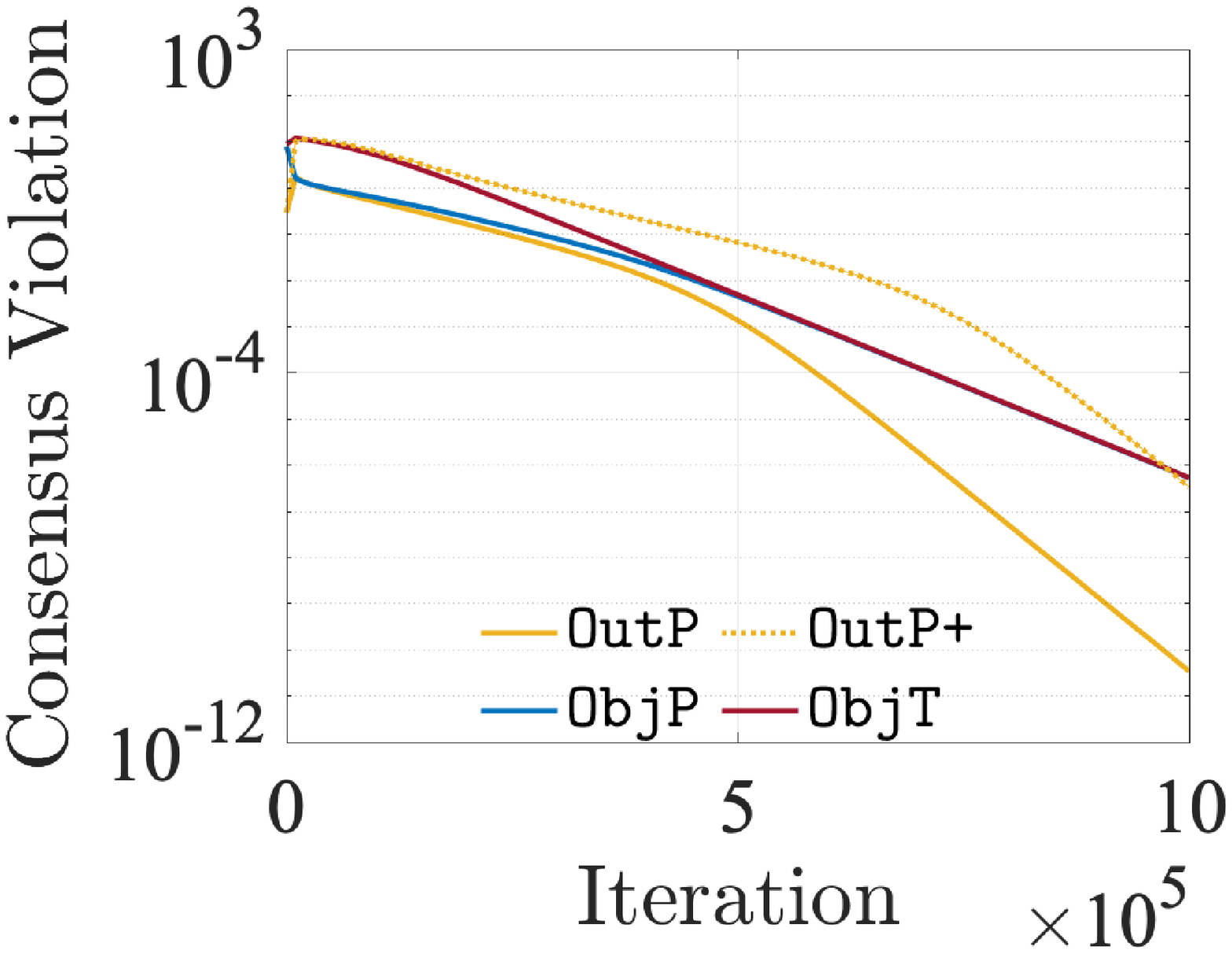}
      \includegraphics[width=\textwidth]{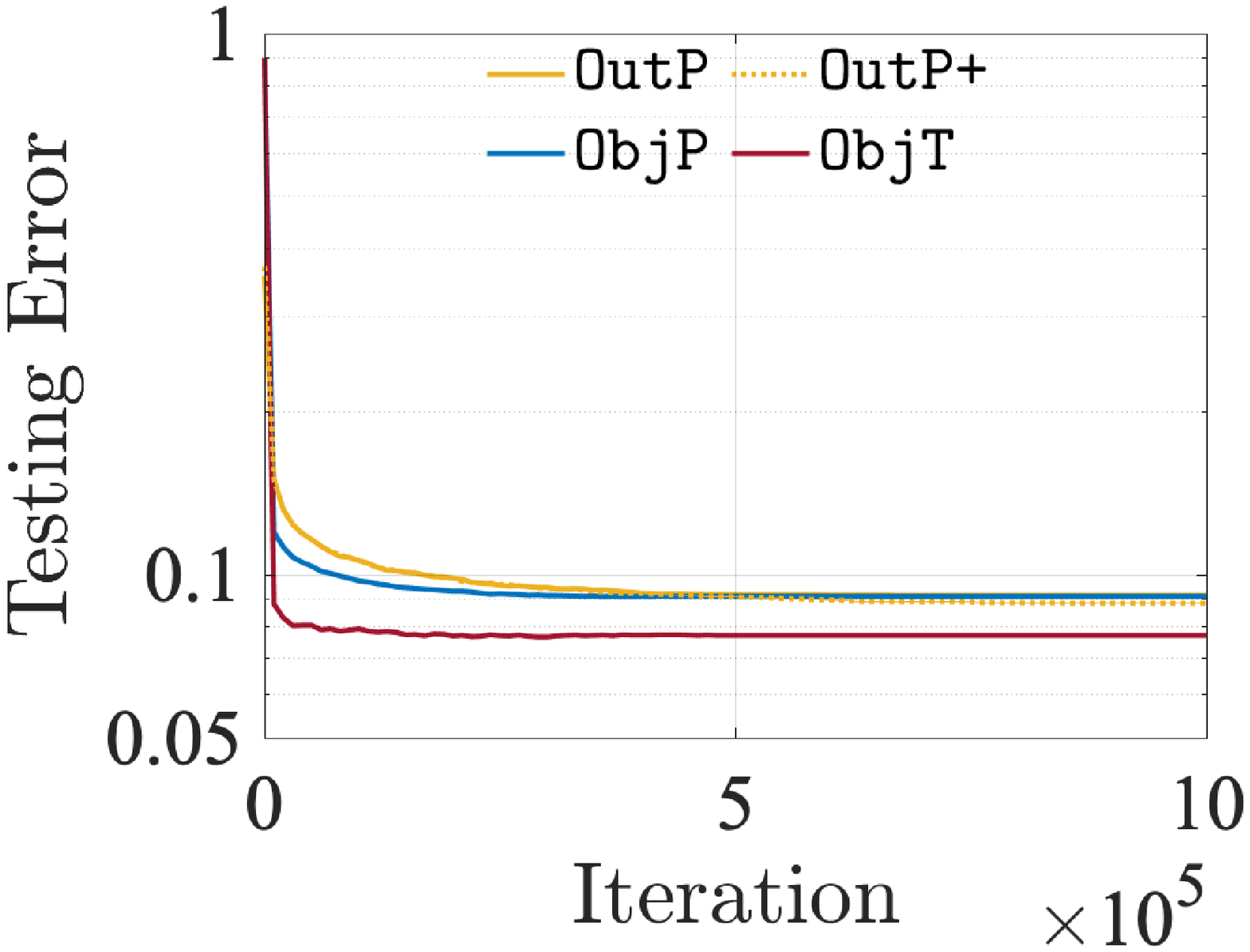}
      \caption{$\bar{\epsilon}=1$}
  \end{subfigure}
     \caption{[MNIST] Consensus violation and testing error when $\bar{\epsilon} \in \{0.05, 0.1, 1\}$.}
     \label{fig:cv}
\end{figure}

\subsection{Additional Hyperparameter Tuning for \texttt{ObjP}} \label{apx-hyperparameter-proximity}

We note that the proximity parameters $\delta^t=1/t^2$ in \eqref{DPADMM-2-Trust} and $\eta^t=1/\sqrt{t}$ in \eqref{DPADMM-2-Prox} can be scaled by multiplying a constant $a \in (0, \infty)$ without affecting the convergence results, as well as the proximity parameter $\widehat{\eta}^t$ for \texttt{OutP} (see Theorem 4 in \cite{huang2019dp}), which is a function of $1/\sqrt{t}$ and numerous parameters, such as $\bar{\epsilon}$, the Gaussian noise parameters, the numbers of data and classes, and so on.

In Figure \ref{fig:FEMNIST_Various_a}, we report the testing errors of the three algorithms when their proximity parameters, namely $\widehat{\eta}^t, \delta^t, \eta^t$ are multiplied by $a \in \{1,100,1000\}$.
First, we note that \texttt{OutP} with $a=1$ as in \cite{huang2019dp} produces the best performance, which implies that the paper \cite{huang2019dp} has well calibrated the proximity parameter $\widehat{\eta}^t$.
Second, the testing errors of \texttt{ObjP} with $a \in \{100, 1000\}$ are less than those of \texttt{OutP} with $a=1$. Even for some cases, \texttt{ObjP} outperforms \texttt{ObjT}.
This implies that the \texttt{ObjP} requires an additional hyperparameter tuning process.
Lastly, we note that the testing errors of \texttt{ObjT} are not greatly varied according to the value of $a$ compared with those of \texttt{ObjP}.
This is because the proximity of a new point $z^{t+1}_p$ from $z^t_p$ in \texttt{ObjP} can be affected by other parameters, such as $\rho^t$ in the objective function of \eqref{DPADMM-2-Prox} while the proximity in \texttt{ObjT} is controlled in the constraints.
This partially shows the benefit of using \texttt{ObjT}.

\begin{figure}[!h]
  \centering
  \begin{subfigure}[b]{0.3\textwidth}
    \centering
    \includegraphics[width=\textwidth]{./Figures1/FEMNIST_TestError_a1_EPS_0.05}
    \includegraphics[width=\textwidth]{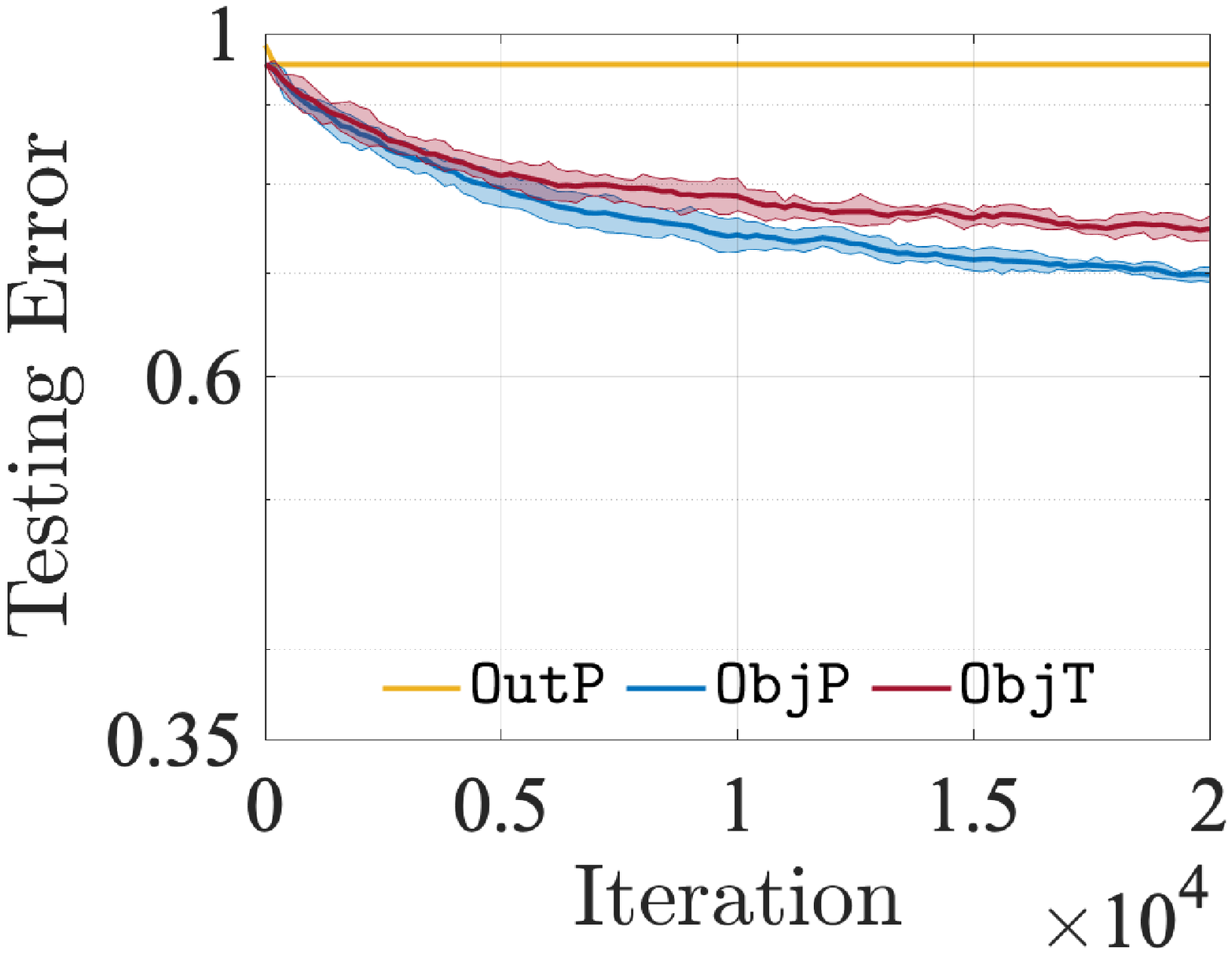}
    \includegraphics[width=\textwidth]{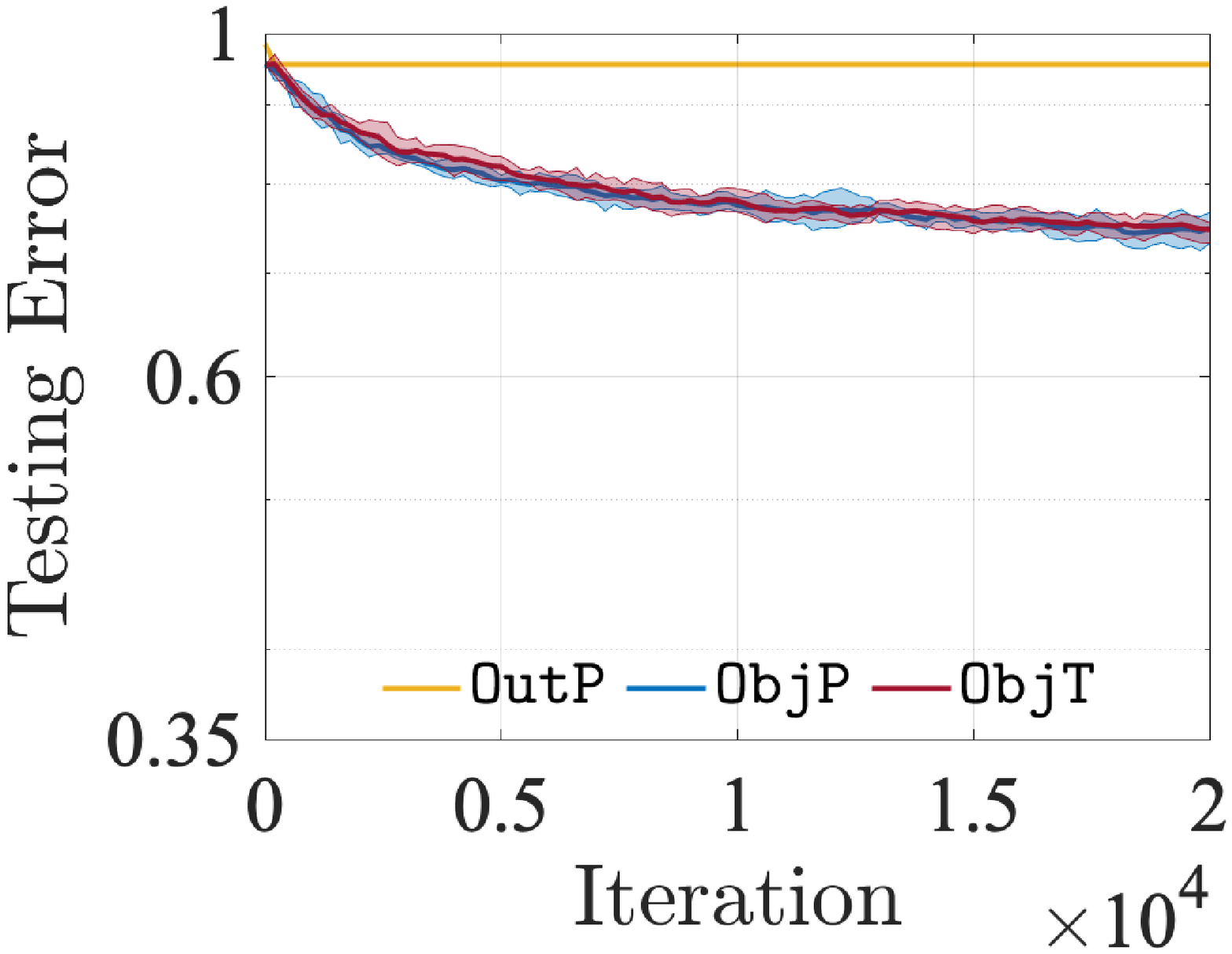}
    \caption{$\bar{\epsilon}=0.05$}
\end{subfigure}
  \begin{subfigure}[b]{0.3\textwidth}
    \centering
    \includegraphics[width=\textwidth]{./Figures1/FEMNIST_TestError_a1_EPS_0.1}
    \includegraphics[width=\textwidth]{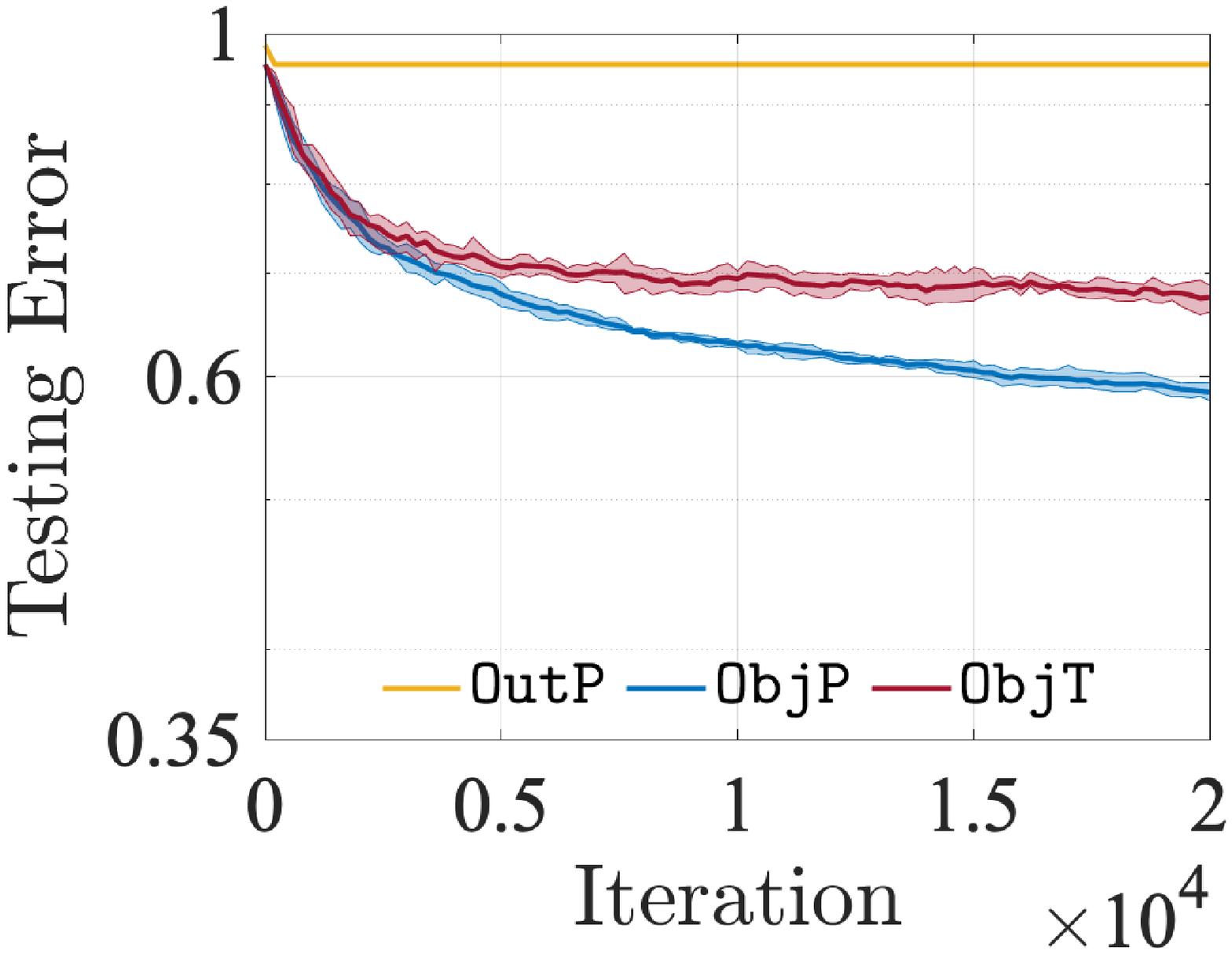}
    \includegraphics[width=\textwidth]{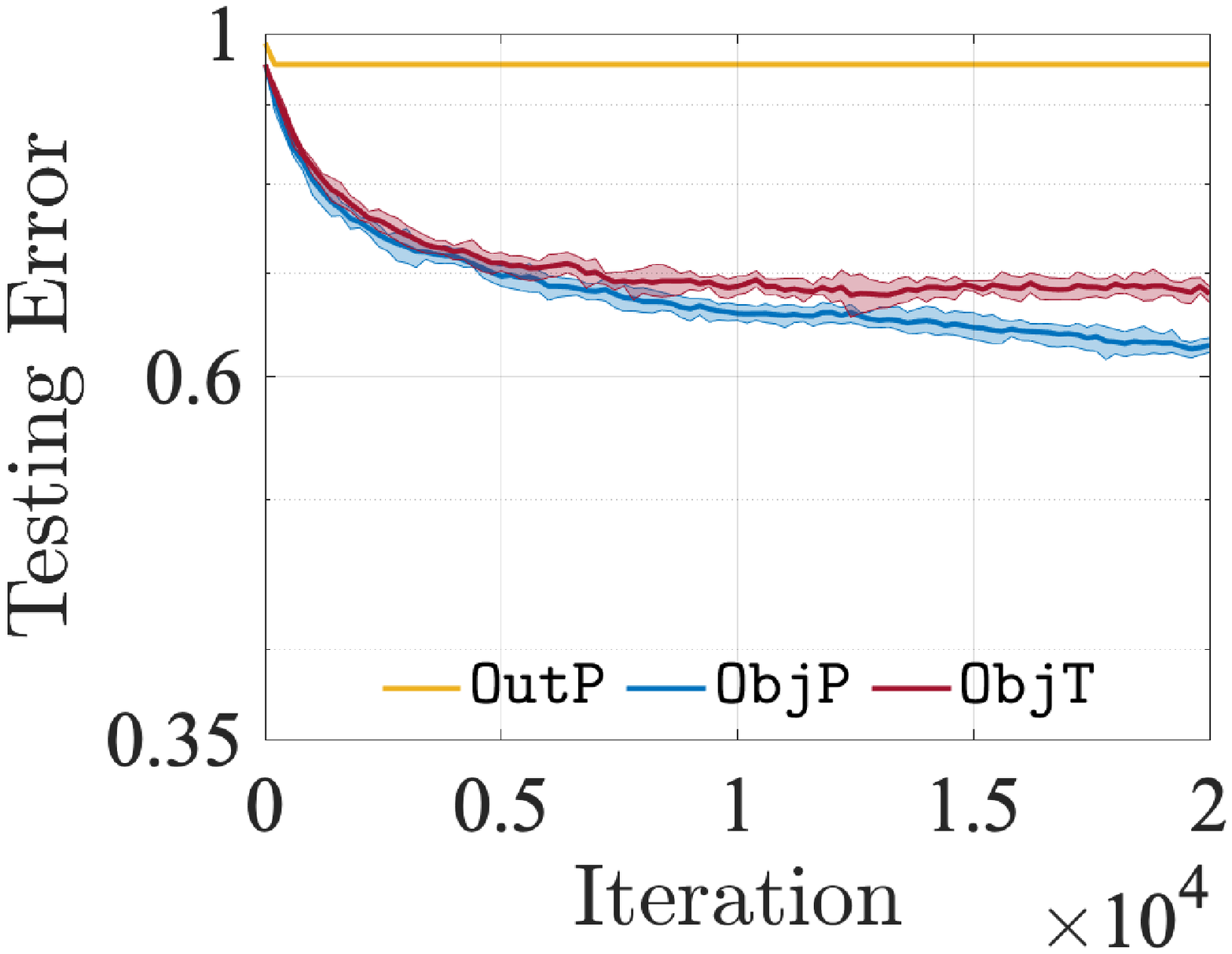}
    \caption{$\bar{\epsilon}=0.1$}
  \end{subfigure}
  \begin{subfigure}[b]{0.3\textwidth}
      \centering
      \includegraphics[width=\textwidth]{./Figures1/FEMNIST_TestError_a1_EPS_1.0}
      \includegraphics[width=\textwidth]{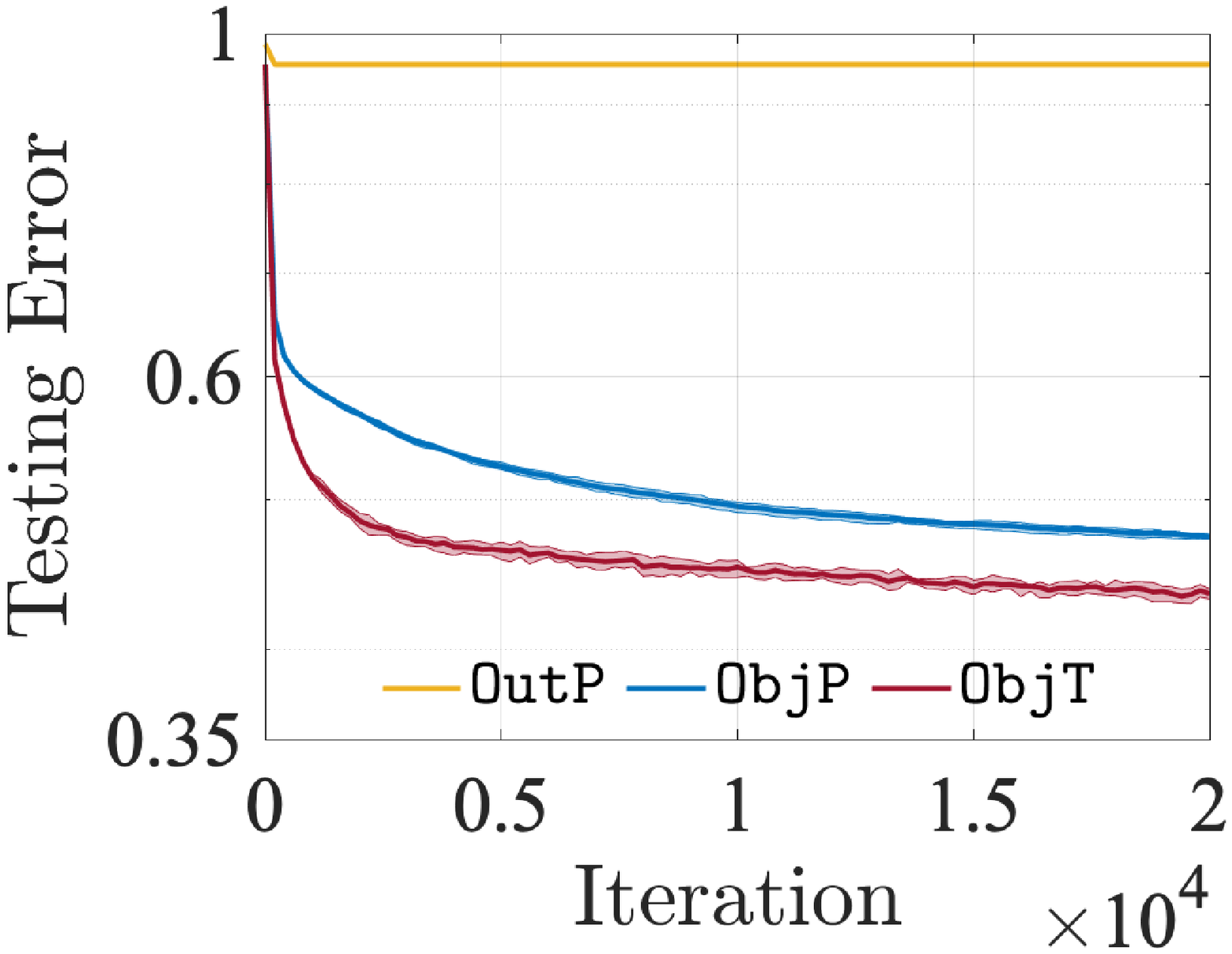}
      \includegraphics[width=\textwidth]{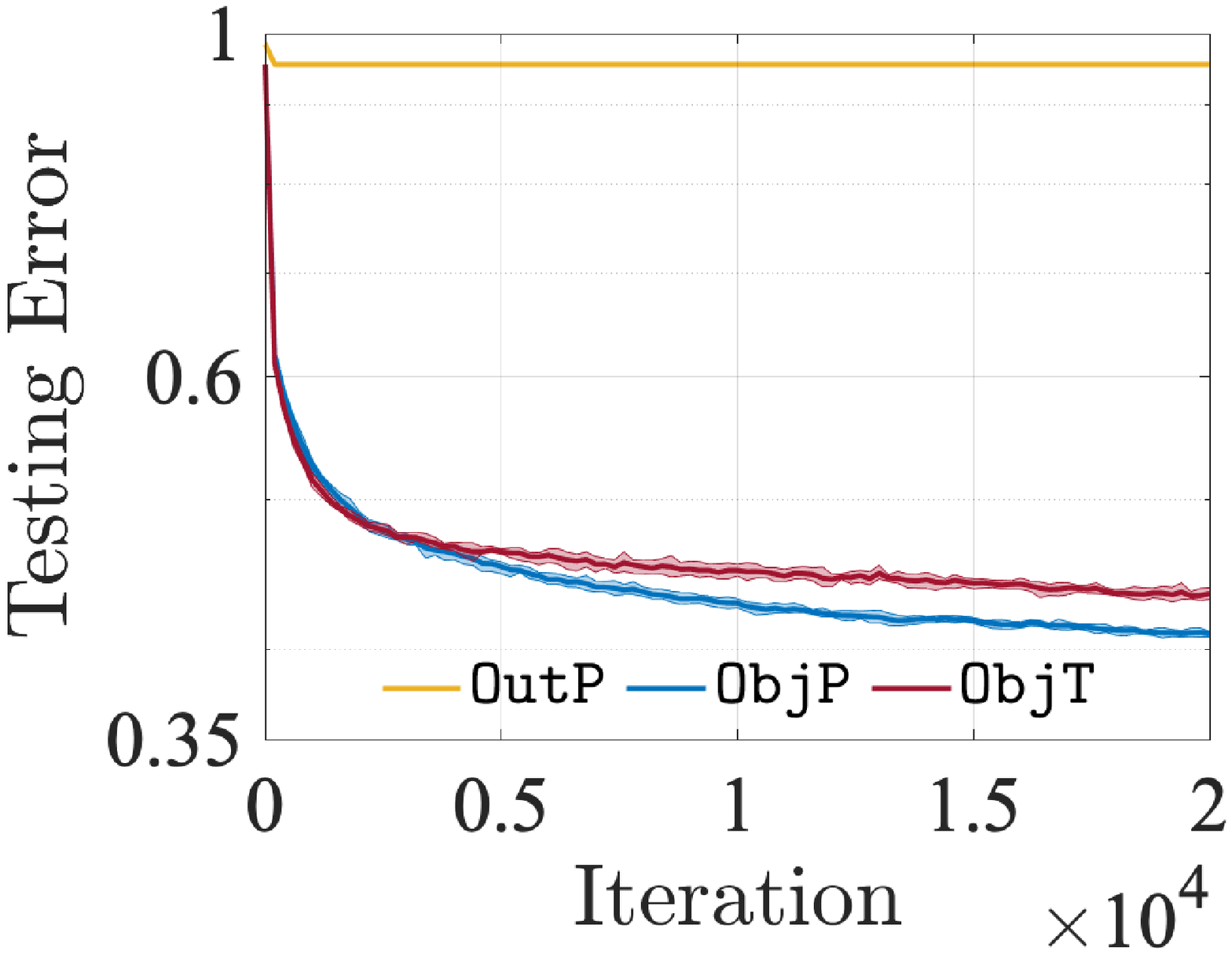}
      \caption{$\bar{\epsilon}=1$}
  \end{subfigure}
  % \begin{subfigure}[b]{0.3\textwidth}
  %     \centering
  %     \includegraphics[width=\textwidth]{./Figures1/FEMNIST_TestError_a1_EPS_3.0}
  %     \includegraphics[width=\textwidth]{./Figures1/FEMNIST_TestError_a100_EPS_3.0}
  %     \includegraphics[width=\textwidth]{./Figures1/FEMNIST_TestError_a1000_EPS_3.0}
  %     \caption{$\bar{\epsilon}=3$}
  % \end{subfigure}
  % \begin{subfigure}[b]{0.3\textwidth}
  %     \centering
  %     \includegraphics[width=\textwidth]{./Figures1/FEMNIST_TestError_a1_EPS_5.0}
  %     \includegraphics[width=\textwidth]{./Figures1/FEMNIST_TestError_a100_EPS_5.0}
  %     \includegraphics[width=\textwidth]{./Figures1/FEMNIST_TestError_a1000_EPS_5.0}
  %     \caption{$\bar{\epsilon}=5$}
  % \end{subfigure}
     \caption{[FEMNIST] Testing errors when $a\in\{1, 100, 1000\}$ (top, middle, bottom).}
     \label{fig:FEMNIST_Various_a}
\end{figure}

\newpage

\end{document}